\documentclass{article}
 \pdfoutput=1 

\usepackage[preprint, nonatbib]{neurips_2021}

\usepackage[table]{xcolor}

\usepackage[utf8]{inputenc}
\usepackage{booktabs}
\usepackage{multirow}
\usepackage{float}
\usepackage{amssymb}
\usepackage{amsmath}
\usepackage{amsthm}
\usepackage{bbold}
\usepackage{siunitx}
\usepackage{tikz}
\usepackage{chngcntr}
\usepackage{verbatim}
\usepackage{mathtools}
\usepackage{nicefrac}
\usepackage{esvect}
\usepackage{xspace}
\usepackage{xparse}
\usepackage{algorithm}
\usepackage{algorithmic}
\usepackage{thmtools,thm-restate}
\usepackage{graphicx}
\usepackage{sidecap}
\usepackage{hyperref}
\usepackage{wrapfig}
\hypersetup{breaklinks, colorlinks, citecolor=green, linkcolor=blue}      
\usepackage{enumitem}
\usepackage[flushleft]{threeparttable}
\usepackage{tikzit}
\usepackage{subcaption}

\tikzstyle{server}=[fill={rgb,255: red,64; green,64; blue,64}, draw=black, shape=circle]
\tikzstyle{client}=[fill={rgb,255: red,55; green,126; blue,184}, draw=black, shape=circle]
\tikzstyle{optima}=[fill={rgb,255: red,55; green,126; blue,184}, draw=black, shape=rectangle]
\tikzstyle{globalopt}=[fill={rgb,255: red,255; green,128; blue,0}, draw=black, shape=rectangle]
\tikzstyle{globalroundopt}=[fill={rgb,255: red,255; green,128; blue,0}, draw=black, shape=circle]
\tikzstyle{pseudoserver}=[fill={rgb,255: red,191; green,191; blue,191}, draw={rgb,255: red,128; green,128; blue,128}, shape=circle]
\tikzstyle{sgdnode}=[fill=white, draw=black, shape=circle]
\tikzstyle{client2}=[fill={rgb,255: red,77; green,175; blue,74}, draw=black, shape=circle]
\tikzstyle{clientopt}=[fill={rgb,255: red,77; green,175; blue,74}, draw=black, shape=rectangle]
\tikzstyle{averageopt}=[fill={rgb,255: red,191; green,191; blue,191}, draw={rgb,255: red,128; green,128; blue,128}, shape=rectangle]
\tikzstyle{averageroundopt}=[fill={rgb,255: red,191; green,191; blue,191}, draw={rgb,255: red,128; green,128; blue,128}, shape=circle]
\tikzstyle{gradient}=[dashed]

\tikzstyle{server update}=[->, draw={rgb,255: red,77; green,175; blue,74}, very thick]
\tikzstyle{client update}=[draw={rgb,255: red,55; green,126; blue,184}, ->, very thick]
\tikzstyle{update}=[draw=gray, dashed, -, very thick, very thick]
\tikzstyle{sgd}=[draw={rgb,255: red,65; green,61; blue,58}, ->, very thick]
\tikzstyle{correction}=[draw={rgb,255: red,228; green,26; blue,28}, ->, very thick]
\tikzstyle{midpoint}=[-, draw={rgb,255: red,191; green,191; blue,191}, very thick]

\setlist{leftmargin=5.5mm}
\DeclareCaptionFormat{myformat}{\fontsize{8.5}{10}\selectfont#1#2#3}
\usepackage{bm}
\usepackage{tabularx} 

\DeclarePairedDelimiterX{\inp}[2]{\langle}{\rangle}{#1, #2}
\DeclarePairedDelimiterX{\abs}[1]{\lvert}{\rvert}{#1}
\DeclarePairedDelimiterX{\norm}[1]{\lVert}{\rVert}{#1}
\DeclarePairedDelimiterX{\cbr}[1]{\{}{\}}{#1} 
\DeclarePairedDelimiterX{\rbr}[1]{(}{)}{#1} 
\DeclarePairedDelimiterX{\sbr}[1]{[}{]}{#1} 

\providecommand{\tsum}{\textstyle\sum} 
  \providecommand{\R}{\mathbb{R}} 

  \DeclareMathOperator{\expect}{\mathbb{E}}
  \DeclareMathOperator{\E}{\expect}

  \DeclareMathOperator{\sgn}{sign}
  \makeatletter
  \def\sign{\@ifnextchar*{\@sgnargscaled}{\@ifnextchar[{\sgnargscaleas}{\@ifnextchar{\bgroup}{\@sgnarg}{\sgn} }}}
  \def\@sgnarg#1{\sgn\rbr{#1}}
  \def\@sgnargscaled#1{\sgn\rbr*{#1}}
  \def\@sgnargscaleas[#1]#2{\sgn\rbr[#1]{#2}}
  \makeatother

  \DeclareMathOperator*{\argmin}{arg\,min}

  \providecommand{\cc}{\bm{c}}
  \providecommand{\dd}{\bm{d}}
  \providecommand{\ee}{\bm{e}}

  \renewcommand{\gg}{\bm{g}}

  \let\lll\ll
  \renewcommand{\ll}{\bm{l}}
  \providecommand{\mm}{\bm{m}}

  \renewcommand{\ss}{\bm{s}}

  \renewcommand{\vv}{\bm{v}}
  
  \providecommand{\xx}{\bm{x}}
  \providecommand{\yy}{\bm{y}}
  \providecommand{\zz}{\bm{z}}

  \providecommand{\cB}{\mathcal{B}}
  
  \providecommand{\cD}{\mathcal{D}}
  \providecommand{\cE}{\mathcal{E}}

  \providecommand{\cO}{\mathcal{O}}

  \providecommand{\cS}{\mathcal{S}}
  
  \providecommand{\cU}{\mathcal{U}}
  \providecommand{\cV}{\mathcal{V}}

\newtheorem{theorem}{Theorem}

\newtheorem{corollary}[theorem]{Corollary}

\newtheorem{lemma}{Lemma}

\newcommand{\ignore}[1]{}

\newcommand{\e}{\epsilon}

\definecolor{color1}{RGB}{228,26,28}
\definecolor{color2}{RGB}{55,126,184}
\definecolor{color3}{RGB}{77,175,74}
\definecolor{color4}{RGB}{152,78,163}
\definecolor{color5}{RGB}{255,127,0}

\definecolor{mygreen}{RGB}{77,175,74}
\definecolor{myblue}{RGB}{55,126,184}
\definecolor{skyblue}{RGB}{117,187,253}
\definecolor{myred}{RGB}{228,26,28}

\usepackage{pifont}
\makeatletter
\newcommand{\myitem}[1]{%
\item[\textbf{(#1)}]\protected@edef\@currentlabel{#1}%
}
\makeatother

\usetikzlibrary{fit,calc}
\newcommand*{\tikzmk}[1]{\tikz[remember picture,overlay,] \node (#1) {};}

\newcommand{\highlight}[1]{\tikz[remember picture,overlay]{\node[yshift=3pt,fill=#1,opacity=.25,fit={($(A)+(5pt,7pt)$)($(B)+(-3pt,-7pt)$)}] {};}}
\colorlet{client}{red!20}
\newcommand{\speedup}[1]{{\color{gray}(\ifdim #1 pt > 0.3pt #1\else $< #1$\fi{}$\times$)}}

\providecommand{\mycomment}[3]{\todo[caption={},size=footnotesize,color=#1!20]{\textbf{#2: }#3}}%
\providecommand{\inlinecomment}[3]{%
  {\color{#1}#2: #3}}%
\newcommand\commenter[2]%
{%
  \expandafter\newcommand\csname i#1\endcsname[1]{\inlinecomment{#2}{#1}{##1}}
  \expandafter\newcommand\csname #1\endcsname[1]{\mycomment{#2}{#1}{##1}}
}

\usepackage{xspace}
\xspaceaddexceptions{]\}}

\newcommand{\fedavg}{{\sc FedAvg}\xspace}
\newcommand{\mime}{{\sc Mime}\xspace}
\newcommand{\mimelite}{{\sc MimeLite}\xspace}
\newcommand{\scaffold}{{\sc Scaffold}\xspace}
\newcommand{\base}{\sc Base}
\newcommand{\serveronly}{{\sc Server-Only}\xspace}

\newcommand{\mimebox}[1]{\colorbox{myblue!30}{#1}}
\newcommand{\litebox}[1]{\colorbox{mygreen!30}{#1}}
\newcommand{\badbox}[1]{\colorbox{myred!30}{#1}}

\usepackage[toc, page, header]{appendix}
\title{Mime: Mimicking Centralized Stochastic Algorithms in Federated Learning}

\author{
  Sai Praneeth Karimireddy\\
  EPFL\\
  \texttt{sai.karimireddy@epfl.ch}\\
  \And
  Martin Jaggi\\
  EPFL\\
  \texttt{martin.jaggi@epfl.ch}\\
  \And
  Satyen Kale\\
  Google Research\\
  \texttt{satyenkale@google.com}\\
  \And
  Mehryar Mohri\\
  Google Research\\
  \texttt{mohri@google.com}\\
  \And
  Sashank J. Reddi\\
  Google Research\\
  \texttt{sashank@google.com}\\
  \And
  Sebastian U. Stich\\
  EPFL\\
  \texttt{sebastian.stich@epfl.ch}\\
  \And
  Ananda Theertha Suresh\\
  Google Research\\
  \texttt{theertha@google.com}
}

\begin{document}
\maketitle

\begin{abstract}

Federated learning (FL) is a challenging setting for optimization due to the heterogeneity of the data across different clients which can cause a \emph{client drift} phenomenon. In fact, designing an algorithm for FL that is uniformly better than simple centralized training has been a major open problem thus far. In this work, we propose a general algorithmic framework, \mime, which i) mitigates client drift and ii) adapts an arbitrary centralized optimization algorithm such as momentum and Adam to the cross-device federated learning setting. \mime uses a combination of \emph{control-variates} and \emph{server-level optimizer state} (e.g.\ momentum) at every client-update step to ensure that each local update mimics that of the centralized method run on i.i.d.\ data. We prove a reduction result showing that \mime can translate the convergence of a generic algorithm in the centralized setting into convergence in the federated setting. Moreover, we show that, when combined with momentum-based variance reduction, \mime is provably \emph{faster than any centralized method}--the first such result. We also perform a thorough experimental exploration of \mime's performance on real world datasets.

\end{abstract}



\section{Introduction}

Federated learning (FL) is an increasingly important 
large-scale learning framework where the training data remains distributed over a
large number of clients, which may be mobile phones or network
sensors~\cite{konevcny2016federated,konecny2016federated2,
mcmahan2017communication,mohri2019agnostic,kairouz2019advances}. A server 
then orchestrates the clients to train a single model, here referred to 
as a \emph{server model}, without ever transmitting client data over the network,
thereby providing some basic levels of data privacy and security.

Two important settings are distinguished in FL
\cite[Table 1]{kairouz2019advances}: the \emph{cross-device} and the
\emph{cross-silo} settings.  The cross-silo setting corresponds to a
relatively small number of reliable clients, typically organizations,
such as medical or financial institutions.  In contrast, in the
\emph{cross-device} federated learning setting, the number of clients
may be extremely large and include, for example, all 3.5 billion
active android phones \cite{holst2019smartphone}. Thus, in that
setting, we may never make even a single pass over the entire clients'
data during training. The cross-device setting is further
characterized by resource-poor clients communicating over a highly
unreliable network. Together, the essential features of this setting
give rise to unique challenges not present in the cross-silo
setting. In this work, we are interested in the more challenging cross-device setting, for
which we will formalize and study stochastic optimization algorithms. Importantly, recent advances in FL optimization, such as SCAFFOLD \cite{karimireddy2019scaffold} or FedDyn \cite{acar2021federated}, are \emph{not anymore applicable} since they are designed for the cross-silo setting. 

\paragraph{The problem.}
The de facto standard algorithm for the cross-device setting is \fedavg
\cite{mcmahan2017communication}, which performs multiple SGD updates
on the available clients before communicating to the server. While
this approach can reduce the \emph{frequency} of communication required,
performing multiple steps on the same client can lead to
`over-fitting' to its atypical local data, a phenomenon known as
\emph{client drift} \cite{karimireddy2019scaffold}. This in turn leads to slower convergence and can, somewhat counter-intuitively, require \emph{larger total communication} \cite{woodworth2020minibatch}. Despite significant attention received from the optimization community, the communication complexity of heterogeneous cross-device has not improved upon that of simple centralized methods, which take no local steps (aka \serveronly methods). Furthermore, algorithmic innovations such as momentum
\cite{sutskever2013importance,cutkosky2019momentum}, adaptivity
\cite{kingma2014adam,zaheer2018adaptive,zhang2019adam}, and clipping
\cite{you2017large,you2019large,zhang2020gradient} are critical to
the success of deep learning applications. The lack of a theoretical understanding of the impact of multiple client steps has also hindered adapting these techniques in a principled manner into the client updates, in order to replace the vanilla SGD update of \fedavg.

To overcome such deficiencies, we propose a new framework, \mime,
that mitigates client drift and can adapt an arbitrary centralized
optimization algorithm, e.g.\ SGD with momentum or Adam, to the
federated setting. In each local client update, \mime uses global
optimizer state, e.g.\ momentum or adaptive learning rates, and an SVRG-style correction to mimic the
updates of the centralized algorithm run on i.i.d.\ data. This optimizer state is computed only at the server level and kept fixed
throughout the local steps, thereby avoiding overfitting to the
atypical local data of any single client. 
\vspace{-2mm}
\paragraph{Contributions.} We summarize our main results below.\vspace{-1mm} 
\begin{itemize}[nosep]
  \item {\bf \mime framework.} We formalize the cross-device federated learning problem, 
  and propose a new framework \mime that can adapt arbitrary 
  centralized algorithms to this setting.
  
  \item 
  {\bf Convergence result.} We prove a result showing that \mime successfully reduces client drift. We also prove that the convergence of any generic algorithm in the centralized setting translates convergence of its \mime version in the federated setting.
  
  \item {\bf Speed-up over centralized methods.} By carefully tracking the bias introduced due to multiple local steps, we prove that \mime with momentum-based variance reduction (MVR) can beat a lower bound for centralized methods, thus breaking a fundamental barrier. This is the first such result in FL, and also the first general result showing asymptotic speed-up due to local steps. 
  
  \item {\bf Empirical validation.} We propose a simpler variant, \mimelite, with
  an empirical performance similar to \mime. We report the results
  of thorough experimental analysis demonstrating that both 
  \mime and \mimelite indeed converge faster than \fedavg.
\end{itemize}


\paragraph{Related work.}

\emph{Analysis of \fedavg:} Much of the recent work in federated
learning has focused on analyzing \fedavg. For identical clients,
\fedavg coincides with parallel SGD, for which
\cite{zinkevich2010parallelized} derived an analysis with asymptotic
convergence. Sharper and more refined analyses of the same method,
sometimes called local SGD, were provided by \cite{stich2018local},
and more recently by \cite{stich2019error},
\cite{patel2019communication}, \cite{khaled2020tighter}, and
\cite{woodworth2020local}, for identical functions.  Their analysis
was extended to heterogeneous clients in
\cite{wang2019adaptive,yu2019parallel, karimireddy2019scaffold,
khaled2020tighter,koloskova2020unified}. \cite{charles2020outsized}
derived a tight characterization of FedAvg with quadratic functions
and demonstrated the sensitivity of the algorithm to both client and
server step sizes. Matching upper and lower bounds were recently given
by \cite{karimireddy2019scaffold} and \cite{woodworth2020minibatch}
for general functions, proving that \fedavg can be slower than even
SGD for heterogeneous data, due to the \emph{client-drift}.

\emph{Comparison to \scaffold:} For the cross-silo setting where the
number of clients is relatively low, \cite{karimireddy2019scaffold}
proposed the \scaffold algorithm, which uses control-variates (similar
to SVRG) to correct for client drift. However, their algorithm
crucially relies on \emph{stateful clients} which repeatedly
participate in the training process. 
FedDyn \cite{acar2021federated} reduces the communication requirements, but also requires persistent stateful clients.
In contrast, we focus on the
cross-device setting where clients may be visited only once during
training and where they are \emph{stateless} (and thus \scaffold and FedDyn are inapplicable). This is akin to the
difference between the finite-sum (corresponding to cross-silo) and stochastic (cross-device) settings in traditional centralized optimization~\cite{lei2017less}.

\emph{Comparison to FedAvg and variants:} \cite{hsu2019measuring} and
\cite{wang2019slowmo} observed that using \emph{server momentum}
significantly improves over vanilla \fedavg. This idea was generalized
by \cite{reddi2020adaptive}, who replaced the server update with an
arbitrary optimizer, e.g.\ Adam. However, these methods only modify
the server update while using SGD for the client updates. \mime, on
the other hand, ensures that every \emph{local client update}
resembles the optimizer e.g.\ \mime would apply momentum in every
client update and not just at the server level. Beyond this,
\cite{li2018fedprox} proposed to add a regularizer to ensure client
updates remain close. However, {this may slow down convergence} 
(cf.~Fig.~\ref{fig:add-prox-scaffold-experiments} and 
\cite{karimireddy2019scaffold,wang2020tackling}). Other orthogonal
directions which can be combined with \mime include tackling
computation heterogeneity, where some clients perform many more
updates than others \cite{wang2020tackling}, improving fairness by
modifying the objective \cite{mohri2019agnostic,li2019fair},
incorporating differential privacy
\cite{geyer2017differentially,agarwal2018cpsgd,thakkar2020understanding},
Byzantine adversaries
\cite{pillutla2019robust,wang2020attack,karimireddy2020learning}, secure
aggregation \cite{bonawitz2017practical,he2020secure}, etc. We defer 
additional discussion to the extensive survey by \cite{kairouz2019advances}.

\section{Problem setup}
This section formalizes the problem of cross-device federated learning~\cite{kairouz2019advances}. Cross-device FL is characterized by a large number of client devices like mobile phones which may potentially connect to the server at most once. Due to their transient nature, it is not possible to store any state on the clients, precluding an algorithm like \scaffold. Furthermore, each client has only a few samples, and there is wide heterogeneity in the samples across clients. Finally, communication is a major bottleneck and a key metric for optimization in this setting is the number of communication rounds.




  
Thus, our objective will be to minimize the following
quantity within the fewest number of client-server communication
rounds:\vspace{-3mm}
\begin{equation}
\label{eqn:problem}
  f(\xx) = \expect_{i \sim \cD}\sbr[\Big]{f_i(\xx) := \frac{1}{n_i}\sum_{\nu=1}^{n_i} f_i(\xx; \zeta_{i,\nu})}\,.
\end{equation}
Here, $f_i$ denotes the loss function of client $i$ and
$\{\zeta_{i,1}, \ldots, \zeta_{i,n_i}\}$ its local data. Since the
number of clients is extremely large, while the size of each local data is
rather modest, we represent the former as an expectation and the
latter as a finite sum. In each round, the algorithm samples a subset
of clients (of size $S$) and performs some updates to the server
model. Due to the transient and heterogeneous nature of the clients, it is easy to see that the problem becomes intractable with arbitrarily dissimilar clients.
Thus, it is \textbf{necessary} to
assume bounded dissimilarity across clients.\vspace*{-2mm}
\begin{enumerate}[leftmargin=26pt]
  \myitem{A1} \label{asm:heterogeneity} \textbf{$G^2$-BGV} or bounded inter-client gradient variance: there exists $G \geq 0$ such that
    \begin{equation*}
        \expect_{i \sim \cD} \sbr{\norm{\nabla f_i(\xx)- \nabla f(\xx)}^2} \leq G^2 \,,\ \forall \xx\,.
    \end{equation*}
    \end{enumerate} 
Next, we also characterize the variance in the Hessians.
\begin{enumerate}[leftmargin=26pt]    \vspace*{-2mm}
  \myitem{A2}\textbf{$\delta$-BHV} or bounded Hessian variance: Almost surely, the loss function of any client $i$ satisfies
    \label{asm:hessian-similarity} 
    \begin{equation*}\label{eqn:hessian-similarity}
    \norm{\nabla^2 f_i(\xx; \zeta) - \nabla^2 f(\xx)} \leq  \delta\,,\ \forall \xx\,.
    \end{equation*}
  \end{enumerate} \vspace*{-2mm}
 This is in contrast to the usual smoothness assumption that can be stated as: \vspace*{-2mm}
 \begin{enumerate}[leftmargin=31pt]    
  \myitem{A2*}\textbf{$L$-smooth}: $\norm{\nabla^2 f_i(\xx; \zeta)} \leq  L\,,\ \forall \xx\,,$ a.s. for any $i$.
    \label{asm:main-smoothness}
  \end{enumerate} \vspace*{-2mm}
  Note that if $f_i(\xx; \zeta)$ is $L$-smooth then \eqref{asm:hessian-similarity} is satisfied with $\delta \leq 2L$, and hence \eqref{asm:hessian-similarity} is \emph{weaker} than \eqref{asm:main-smoothness}. In realistic examples we expect the clients to be similar and hence that $\bm{\delta \lll L}$.
In addition, we assume that $f(\xx)$ is bounded from below by $f^\star$ and is $L$-smooth, as is standard.


\section{Mime framework}
In this section we describe how to adapt an arbitrary centralized optimizer (referred to as the ``base'' algorithm) which may have internal state (e.g. momentum in SGD)
to the federated learning problem \eqref{eqn:problem} while ensuring there is no client-drift.
Algorithm~\ref{alg:mime} describes our framework. We develop two variants, \mime and \mimelite, which consist of three components i) a base algorithm we are seeking to mimic, ii) how we compute the global (server) optimizer state, and iii) the local client updates.

\setlength{\intextsep}{0pt}
\setlength{\columnsep}{8pt}
\begin{wrapfigure}{Lt}{0.6\textwidth}
\begin{minipage}{0.6\textwidth}
\begin{algorithm}[H] 
  \caption{\mimebox{\textbf{Mime}} and \litebox{\textbf{MimeLite}}}\label{alg:mime}
\begin{algorithmic}
    \STATE \textbf{input:} initial $\xx$ and $\ss$, learning rate $\eta$ and base algorithm $\cB=(\cU, \cV)$
    \FOR {each round $t = 1, \cdots, T$}
        \STATE sample subset $\cS$ of clients
        \STATE \textbf{communicate} $(\xx, \ss)$ to all clients $i \in \cS$
        \STATE\hspace{-2mm} \mimebox{\textbf{communicate} $\cc \leftarrow \frac{1}{\abs{\cS}} \sum_{j \in \cS}\nabla f_j(\xx)$ (only Mime)}
        \ONCLIENT{\tikzmk{A} $i \in \cS $}
            \STATE initialize local model $\yy_i \leftarrow \xx$
            \FOR {$k = 1, \cdots, K$}
                \STATE sample mini-batch $\zeta$ from local data
                \STATE \mimebox{$\gg_i \leftarrow \nabla f_i(\yy_i; \zeta) - \nabla f_i(\xx; \zeta) + \cc$ \ (\textbf{Mime})}
                \STATE \litebox{$\gg_i \leftarrow \nabla f_i(\yy_i; \zeta)$ \ (\textbf{MimeLite})}
                \STATE update $\yy_i \leftarrow \yy_i - \eta \cU(\gg_i, \ss)$
            \ENDFOR
            \STATE compute full local-batch gradient $\nabla f_i(\xx)$
            \STATE \textbf{communicate} $(\yy_i, \nabla f_i(\xx))$
        \ENDON
        \tikzmk{B}
        \STATE $\ss \leftarrow \cV\rbr*{\tfrac{1}{\abs{\cS}} \tsum_{i \in \cS} \nabla f_i(\xx),\ \ss}$ \ (update optimizer state)
        \STATE $\xx \leftarrow \frac{1}{\abs{\cS}} \sum_{i \in \cS} \yy_i$ \ (update server parameters)
  \ENDFOR
\end{algorithmic}
\end{algorithm}
\end{minipage}
\vspace{-17mm}
\end{wrapfigure}

\paragraph{Base algorithm.}
We assume the centralized base algorithm we are imitating can be decomposed into two steps: an \emph{update step} $\cU$ which updates the parameters $\xx$, and a \emph{optimizer state update step} $\cV(\cdot)$ which keeps track of global optimizer state $\ss$. Each step of the base algorithm $\cB=(\cU, \cV)$ uses a gradient $\gg$ to update the parameter $\xx$ and the optimizer state $\ss$ as follows:
\begin{equation}\tag{\base{\sc Alg}}\label{eqn:base-alg}
  \begin{split}
    \xx &\leftarrow \xx - \eta\,\cU(\gg, \ss)\,,\\
    \ss &\leftarrow \cV(\gg, \ss)\,.
  \end{split}
\end{equation}
As an example, consider SGD with momentum (SGDm). The state in SGDm is the momentum $\mm_t$. SGDm uses the following update steps:
\begin{equation*}
  \begin{split}
    \xx_{t} &= \xx_{t-1} - \eta\, ((1-\beta)\nabla f_i(\xx_{t-1}) +\beta\mm_{t-1})\,,\\
    \mm_{t} &= (1-\beta)\nabla f_i(\xx_{t-1}) +\beta\mm_{t-1}\,.
  \end{split}
\end{equation*}
Thus, SGDm can be represented in the above generic form with $\cU(\gg, \ss) = (1-\beta)\gg +\beta\ss$ and $\cV(\gg, \ss) = (1-\beta)\gg +\beta\ss$. Table~\ref{tab:add-alg-details} in Appendix shows how other algorithms like Adam, Adagrad, etc. can be represented in this manner. We keep the update $\cU$ to be linear in the gradient $\gg$, whereas $\cV$ can be more complicated.
 This implies that while the parameter update step $\cU$ is relatively resilient to receiving a biased gradient $\gg$ while $\cV$ can be much more sensitive. 

\paragraph{Compute optimizer state globally, apply locally.}
When updating the optimizer state of the base algorithm, we use only the gradient computed at the server parameters. Further, they remain fixed throughout the local updates of the clients. This ensures that these optimizer state remain unbiased 
and representative of the global function $f(\cdot)$. At the end of the round, the server performs
\begin{align}
  & \ss \leftarrow \cV\rbr*{\tfrac{1}{\abs{\cS}} \tsum_{i \in \cS} \nabla f_i(\xx),\ \ss}\,,  \nonumber\\ 
  & \nabla f_i(\xx) = \tfrac{1}{n_i}\tsum_{\nu=1}^{n_i}\nabla f_i(\xx; \zeta_{i,\nu})\,. \tag{\sc OptState}
\end{align}
Note that we use full-batch gradients computed at the server parameters $\xx$, not client parameters $\yy_i$.
\paragraph{Local client updates.}
Each client $i \in \cS$ performs $K$ updates using $\cU$ of the base algorithm and a minibatch gradient. There are two variants possible corresponding to \mimebox{\mime} and \litebox{\mimelite} differentiated using colored boxes. Starting from $\yy_i \leftarrow \xx$, repeat the following $K$ times
\begin{align}
    \yy_i &\leftarrow \yy_i - \eta \cU(\gg_i, \ss) \tag{\sc CltStep} 
\end{align}
where \litebox{$\gg_i \leftarrow \nabla f_i(\yy_i; \zeta)$} for \mimelite, and  \mimebox{$\gg_i \leftarrow \nabla f_i(\yy_i; \zeta) - \nabla f_i(\xx; \zeta) + \tfrac{1}{\abs{\cS}} \tsum_{j \in \cS} \nabla f_j(\xx)$} for \mime. \mimelite simply uses the local minibatch gradient whereas \mime uses an SVRG style correction~\cite{johnson2013accelerating}. This is done to reduce the noise from sampling a local mini-batch. While this correction yields faster rates in theory (and in practice for convex problems), in deep learning applications we found that \mimelite closely matches the performance of \mime. 

Finally, there are two modifications made in practical FL: we weight all averages across the clients by the number of datapoints $n_i$ \cite{mcmahan2017communication}, and we perform $K$ epochs instead of $K$ steps \cite{wang2020tackling}. 


\section{Theoretical analysis of Mime}\label{sec:convergence}

Table~\ref{tab:convergence} summarizes the rates of \mime (highlighted in blue) and \mimelite (highlighted in green) and compares them to \serveronly methods when using SGD, Adam and momentum methods as the base algorithms. We will first examine the convergence of \mime and \mimelite with a generic base optimizer and show that its properties are preserved in the federated setting. We then examine a specific momentum based base optimizer, and prove that Mime and MimeLite can be asymptotically faster than the best server-only method. This is the first result to prove the usefulness of local steps and demonstrate asymptotic speed-ups.

\vspace{-2mm}
\subsection{Convergence with a generic base optimizer}\vspace{-1mm}
We will prove a generic reduction result demonstrating that if the underlying base algorithm converges, and is robust to slight perturbations, then \mime and \mimelite also preserve the convergence of the algorithm when applied to the federated setting with additinoal local steps. 
\begin{theorem}\label{thm:reduction}
  Suppose that we have $G^2$ inter-client gradient variance \eqref{asm:heterogeneity}, $L$-smooth $\{f_i\}$ \eqref{asm:main-smoothness}, and $\sigma^2$ intra-client gradient variance \eqref{asm:noise}. Further, suppose that the updater $\cU$ of our base-optimizer $\cB=(\cU, \cV)$ satisfies i) linearity: $\cU(\gg_1 + \gg_2) = \cU(\gg_1) + \cU(\gg_2)$, and ii) Lipschitzness: $\norm{\cU(\gg)} \leq B\norm{\gg}$ for some $B \geq 0$. Then, running \mimebox{\mime} or \litebox{\mimelite} with $K$ local updates and step-size $\eta$ is equivalent to running a \textbf{centralized} algorithm with step-size $\tilde\eta := K\eta \leq \frac{1}{2LB}$, and updates
 \begin{align*}
    \xx_t &\leftarrow \xx_{t-1} - \tilde\eta \,\cU(\gg_t + \text{\badbox{$\ee_t$}}, \ss_{t-1})\,, \text{ and } \\
    \ss_{t} &\leftarrow \cV(\gg_{t}, \ss_{t-1})\,, \text{ where we have}
 \end{align*}
 $\E_t[\gg_t] = \nabla f(\xx_{t-1})$, $\E_t\norm{\gg_t - \nabla f(\xx_{t-1})}^2 \leq G^2/S$, and  
\[
\tfrac{1}{B^2L^2\tilde\eta^2}\E_t\norm{\text{\badbox{$\ee_t$}}}^2 \leq 
    \begin{cases}
        \E_t\norm{\gg_t}^2 & \text{\mimebox{\mime}\,,}\\
        \E_t\norm{\gg_t}^2 + G^2 + \frac{\sigma^2}{K} &\text{\litebox{\mimelite}.}
    \end{cases}
\]
\end{theorem}
Here, we have proven that \mime and \mimelite truly mimic the centralized base algorithm with very small perturbations---the magnitude of $\ee_t$ is $\cO(\tilde\eta^2)$. The key to the result is the linearity of the parameter update step $\cU(\,\cdot\,)$. By separating the base optimizer into a very simple parameter step $\cU$ and a more complicated optimizer state update step $\cV$, we can ensure that commonly used algorithms such as momentum, Adam, Adagrad, and others all satisfy this property. Armed with this general reduction, we can easily obtain specific convergence results.

\begin{table*}[!t]
  \caption{Number of communication rounds required to reach $\norm{\nabla f(\xx)}^2 \leq \e $ (log factors are ignored) with $S$ clients sampled each round.
   All analyses except \scaffold assume $G^2$ bounded gradient dissimilarity~\eqref{asm:heterogeneity}. All analyses assume $L$-smooth losses, except MimeLiteMVR and MimeMVR, which only assume $\delta$ bounded Hessian dissimilarity~\eqref{asm:hessian-similarity}. Convergence of SCAFFOLD depends on the total number of clients $N$ which is potentially infinite. \fedavg and \litebox{\mimelite} are slightly slower than the server-only methods due to additional \badbox{drift} terms in most cases. \mimebox{\mime} is the fastest and either matches or improves upon the optimal statistical rates (first term in the rates). In fact, MimeMVR and MimeLiteMVR beat lower bounds for any server-only method when $\delta \lll L$.
    }
  \centering
  \begin{threeparttable}
  \begin{tabular*}{\textwidth}{@{}l@{\extracolsep{\fill}}>{\large}l>{\large}l@{\extracolsep{\fill}}>{\small}r@{}}
  \toprule
    Algorithm  & \normalsize{Non-convex}  & \normalsize{$\mu$-PL inequality} \\
    \midrule
    {\sc Scaffold}\tnote{a} \cite{karimireddy2019scaffold} & $\rbr*{\frac{N}{S}}^{\frac{2}{3}}\frac{L}{\e}$  & $\frac{N}{S} + \frac{L}{\mu}$ \\ 
    \cmidrule{2-4}
    {\sc SGD}\vspace{-1.5mm}\\
        \qquad \serveronly \cite{ghadimi2013stochastic} &  $\frac{L G^2}{S\e^2} + \frac{L}{\e}$  & $\frac{G^2}{\mu S\e} + \frac{L}{\mu}$ \\
        \qquad\tikzmk{A} MimeLiteSGD$\equiv$\badbox{FedSGD}~\tnote{c} & $\frac{L G^2}{S\e^2}+ \badbox{$\frac{L^2 G}{\e^{3/2}}$} + \frac{L}{\e}$  & $\frac{G^2}{\mu S\e} + \badbox{$\frac{L G}{\mu \sqrt{\e}}$} + \frac{L}{\mu}$  &\tikzmk{B}\highlight{mygreen}\vspace{-0.2mm}\\
        \qquad\tikzmk{A} MimeSGD & $\frac{LG^2}{S\e^2} + \frac{L}{\e}$ & $\frac{G^2}{\mu S\e} + \frac{L}{\mu}$ & \tikzmk{B}\highlight{myblue} \\
    \cmidrule{2-4}
    {\sc Adam} \vspace{-1.5mm}\\
      \qquad \serveronly~\cite{zaheer2018adaptive}\tnote{b} &  $\frac{L}{\e - G^2/S}$  & --  \vspace{1mm} \\
      \qquad\tikzmk{A} MimeLiteAdam\tnote{b}\tnote{c} & $\frac{L\sqrt{S}}{\e - G^2/S}$ & --  & \tikzmk{B}\highlight{mygreen}\vspace{1mm}\\
      \qquad\tikzmk{A} MimeAdam\tnote{b} & $\frac{L}{\e - G^2/S}$  & --  &\tikzmk{B}\highlight{myblue}\\
    \cmidrule{2-4}
    Momentum Variance Reduction (MVR) \vspace{-1mm}\\
      \qquad \serveronly \cite{cutkosky2019momentum}   & $\frac{L G}{\sqrt{S}\e^{3/2}} + \frac{L}{\e}$  & --  \vspace{1mm}\\
        \qquad\tikzmk{A} MimeLiteMVR\tnote{d} & $\frac{\delta (G + \sigma)}{\e^{3/2}} + \frac{G^2 + \sigma^2}{\e} + \frac{\delta}{\e}$  & -- & \tikzmk{B}\highlight{mygreen}\vspace{1.3mm}\\
        \qquad\tikzmk{A} \textbf{MimeMVR}\tnote{d} & $\frac{\delta G}{\sqrt{S}\e^{3/2}} + \frac{G^2}{S \e} + \frac{\delta}{\e}$  & -- & \tikzmk{B}\highlight{myblue}\\
    \cmidrule{1-4}
    \serveronly lower bound \cite{arjevani2019lower} & $\Omega\rbr[\big]{\frac{L G}{\sqrt{S}\e^{3/2}} + \frac{G^2}{S\e} + \frac{L}{\e}}$ & $\Omega\rbr[\big]{\frac{G^2}{S\e}}$  \\
    \bottomrule
  \end{tabular*}
  \begin{tablenotes}
    \item[a] Num. clients ($N$) can be same order as num. total rounds or even $\infty$, making the bounds vacuous.
    \item[b] Adam requires large batch-size $S \geq G^2/\e$ to converge \cite{reddi2018convergence,zaheer2018adaptive}. Convergence of FedAdam with client sampling is unknown (\cite{reddi2020adaptive} only analyze with full client participation).
    \item[c] Requires $K \geq \sigma^2/ G^2$ number of local updates. Typically, intra-client variance is small ($\sigma^2 \lesssim G^2$).
    \item[d] Requires $K \geq L/\delta$ number of local updates. Faster than the lower bound (and hence any \serveronly algorithm) when $\delta \lll L$ i.e. our methods can take advantage of Hessian similarity, whereas \serveronly methods cannot. In worst case, $\delta \approx L$ and all methods are comparable.
  \end{tablenotes}
  \end{threeparttable}
  \label{tab:convergence}\vspace{-5mm}
  \end{table*}

\begin{corollary}[(Mime/MimeLite) with SGD]\label{cor:sgd}
Given that the conditions in Theorem~\ref{thm:reduction} are satisfied, let us run $T$ rounds with $K$ local steps using SGD as the base optimizer and output $\xx^{\text{out}}$. This output satisfies $\expect\norm{\nabla f(\xx^{\text{out}})}^2 \leq \e$ for $F := f(\xx_0) - f^\star$, $\tilde G^2 := G^2 + \sigma^2/K$ and
\begin{itemize}[topsep=0pt,itemsep=-1ex,partopsep=0ex,parsep=0ex]
            \item \textbf{$\mu$-PL inequality:} $\eta = \tilde\cO\rbr[\big]{\frac{1}{\mu K T} }$, and \vspace{-2mm}
            \[
              T = 
              \begin{cases}
                \tilde\cO\rbr[\Big]{\frac{L G^2}{\mu S \e} + \frac{LF}{\mu}\log\rbr[\big]{\frac{1}{\e}}} &\text{\mimebox{\mime}}\,,\\
                \tilde\cO\rbr[\Big]{\frac{L \tilde G^2}{\mu S \e}+ \frac{L \tilde G}{\mu\sqrt{\e}} + \frac{LF}{\mu}\log\rbr[\big]{\frac{1}{\e}}} &\text{\litebox{\mimelite}}\,.
              \end{cases}
            \]
        \item \textbf{Non-convex:} for $\eta = \cO\rbr[\big]{\sqrt{\frac{F S}{L \tilde G^2 T K^2}}}$, and\vspace{-2mm}
        \[
         T =
         \begin{cases}
            \cO\rbr[\Big]{ \frac{L G^2 F}{S\e^2} + \frac{L F}{\e}}&\text{\mimebox{\mime}}\,,\\
            \cO\rbr[\Big]{ \frac{L \tilde G^2 F}{S\e^2}+ \frac{L^2 \tilde G F}{\e^{3/2}} + \frac{L F}{\e}}&\text{\litebox{\mimelite}}\,.
         \end{cases}
       \]
       \vspace{-3mm}
       \end{itemize}
\end{corollary}
If we take a sufficient number of local steps $K \geq G^2/\sigma^2$, then we have $\tilde G = \cO(G)$ in the above rates. On comparing with the rates in Table~\ref{tab:convergence} for \serveronly SGD, we see that \mime exactly matches its rates. \mimelite matches the asymptotic term but has a few higher order terms. Note that when using SGD as the base optimizer, \mimelite becomes exactly the same as \fedavg and hence has the same rate of convergence.
\begin{corollary}[(Mime/MimeLite) with Adam]\label{cor:adam}
Suppose that the conditions in Theorem~\ref{thm:reduction} are satisfied, and further $|\nabla_j f_i(\xx)|\leq H$ for any coordinate $j \in [d]$. Then let us run $T$ rounds using Adam as the base optimizer with $K$ local steps, $\beta_1 =0$, $\varepsilon_0 > 0$, $\eta \leq \varepsilon_0^2/ KL(H+\varepsilon_0)$, and any $\beta_2 \in [0,1)$. Output $\xx^{\text{out}}$ chosen randomly from $\{\xx_1,\dots\xx_T\}$ satisfies $\expect\norm{\nabla f(\xx^{\text{out}})}^2 \leq \e$ for \vspace{-1mm}
\[
 T =
 \begin{cases}
    \cO\rbr[\Big]{\frac{L F(H +\varepsilon_0)^2}{\varepsilon_0^2 (\e - \tilde G^2/S)}}&\textnormal{\mimebox{\mime Adam}}\,,\vspace{2mm}\\
     \cO\rbr[\Big]{\frac{L F(H +\varepsilon_0)^2\sqrt{S}}{\varepsilon_0^2(\e - \tilde G^2/S)}}&\textnormal{\litebox{\mimelite Adam}}\,.
 \end{cases}
\vspace{-3mm}
\]
where $F := f(\xx_0) - f^\star$, $\tilde G^2 := G^2 + \sigma^2/K$.
\end{corollary}
Note that here $\varepsilon_0$ represents a small positive parameter used in Adam for regularization, and is different from the accuracy $\epsilon$. Similar to the \serveronly analysis of Adam~\cite{zaheer2018adaptive}, we assume $\beta_1 =0$ and that batch size is large enough such that $S \geq G^2/\e$. A similar analysis can also be carried out for AdaGrad, and other novel variants of Adam \cite{liu2019variance}.

\subsection{Circumventing server-only lower bounds}
 The rates obtained above, while providing a safety-check, do not beat those of the \serveronly approach. The previous best rates for cross-device FL correspond to MimeLiteSGD which is $\cO(\frac{LG^2}{S\e^2} + \frac{L^2 G}{\e^{3/2}})$ \cite{khaled2020tighter,koloskova2020unified,woodworth2020minibatch}. While, using a separate server-learning rate can remove the effect of the second term \cite{karimireddy2019error}, this at best matches the rate of \serveronly SGD $\cO(\frac{LG^2}{S\e^2})$. This is significantly slower than simply using momentum based variance reduction (MVR) as in in the FL setting (\serveronly MVR) which has a communication complexity of $\cO(\frac{LG}{\sqrt{S}\e^{3/2}})$ \cite{cutkosky2019momentum}. Thus, even though the main reason for studying local-step methods was to improve the communication complexity, none thus far show such improvement. The above difficulty of beating \serveronly may not be surprising given the two sets of strong lower bounds known.

 \paragraph{Necessity of local steps.} Firstly, \cite{arjevani2019lower} show a gradient oracle lower bound of $\Omega(\frac{LG}{\sqrt{S}\e^{3/2}})$. This matches the complexity of MVR, and hence at first glance it seems that \serveronly MVR is optimal. However, the lower bound is really only on the number of gradients computed and not on the number of clients sampled (sample complexity) \cite{foster2019complexity}, or number of rounds of communication required. In particular, multiple local updates which increases number of gradients computed \emph{without needing additional communication} offers us a potential way to side-step such lower bounds. A careful analysis of the bias introduced as a result of such local steps is a key part of our analysis.
 
 \paragraph{Necessity of $\delta$-BHD.} A second set of lower bounds directly study the number of communication rounds required in heterogeneous optimization~\cite{arjevani2015communication,woodworth2020minibatch}. These results prove that there exist settings where local steps provide no advantage and \serveronly methods are optimal. This however contradicts real world experimental evidence~\cite{mcmahan2017communication}. As before, the disparity arises due to the contrived settings considered by the lower bounds. For distributed optimization (with full client participation) and convex quadratic objectives, $\delta$-BHD \eqref{asm:hessian-similarity} was shown to be a sufficient \cite{shamir2014communication,reddi2016aide} and \emph{necessary} \cite{arjevani2015communication} condition to circumvent these lower bounds and yield highly performant methods. We similarly leverage $\delta$-BHD \eqref{asm:hessian-similarity} to design novel methods which significantly extend prior results to i) all smooth non-convex functions (not just quadratics), and ii) cross-device FL with client sampling.
 
 We now state our convergence results with momentum based variance reduction (MVR) as the base-algorithm since it is known to be optimal in the \serveronly setting.
\begin{theorem}\label{thm:interpolation}
  For $L$-smooth $f$ with $G^2$ gradient dissimilarity \eqref{asm:heterogeneity}, $\delta$ Hessian dissimilarity \eqref{asm:hessian-similarity} and $F: = (f(\xx^0) - f^\star)$, let us run MVR as the base algorithm for $T$ rounds with $K \geq L/\delta$ local steps and generate an output $\xx^{\text{out}}$. This output satisfies $\expect\norm{\nabla f(\xx^{\text{out}})}^2 \leq \e$ for \vspace{-2mm}
       \begin{itemize}[topsep=0pt,itemsep=-1ex,partopsep=0ex,parsep=0ex]
       \item \mimebox{\textbf{MimeMVR}}: $\eta = \cO\rbr[\Big]{\min\rbr[\big]{\frac{1}{\delta K} \,, \rbr{\frac{S F}{G^2 T K^3}}^{1/3} }}$, momentum $\beta = 1 - \cO(\tfrac{\delta^2S^{2/3}}{(T G^2)^{2/3}})$, and \vspace{-2mm}
        \[
         T = \cO\rbr[\big]{\frac{\delta G F}{\sqrt{S} \e^{3/2}} + \frac{ G^2}{S\e} + \frac{\delta F}{\e}}\,.
       \]\vspace{-2mm}
       \item \litebox{\textbf{MimeLiteMVR}}: $\eta = \cO\rbr[\Big]{\min\rbr[\big]{\frac{1}{\delta K} \,, \rbr{\frac{ F}{\hat G^2 T K^3}}^{1/3} }}$, momentum $\beta = 1 - \cO(\tfrac{\delta^2}{(T\hat G^2)^{2/3}})$, and \vspace{-2mm}
        \[
         T = \cO\rbr[\big]{\frac{\delta \hat G F}{\e^{3/2}} + \frac{\hat G^2}{\e} + \frac{\delta F}{\e}}\,.
       \]
       \vspace{-3mm}
       \end{itemize}
   \end{theorem}\vspace{-5mm}
Here, we define $\hat G^2 := G^2 + \sigma^2$ and the expectation in $\expect\norm{\nabla f(\xx^{\text{out}})}^2 \leq \e$ is taken both over the sampling of the clients during the running of the algorithm, the sampling of the mini-batches in local updates, and the choice of $\xx^{\text{out}}$ (which is chosen randomly from the client iterates $\yy_i$).


Remarkably, the rates of our methods are independent of $L$ and only depend on $\delta$. Thus, when $\delta \leq L$ and $\delta \leq \nicefrac{L}{S}$ for MimeMVR and MimeLiteMVR, the rates beat the server only lower bound of $\Omega(\frac{LG}{\sqrt{S}\e^{3/2}})$. In fact, if the Hessian variance is small and $\delta \approx 0$, our methods only need $\cO(\nicefrac{1}{\e})$ rounds to communicate. Intuitively, our results show that local steps are very useful when heterogeneity (represented by $\delta$) is smaller than optimization difficulty (captured by smoothness constant $L$).

MimeMVR uses a momentum parameter $\beta$ of the order of $(1 - \cO(TG^2)^{-2/3})$ i.e.\ as $T$ increases, $\beta$ asymptotically approaches 1. In contrast, previous analyses of distributed momentum (e.g. \cite{yu2019linear}) prove rates of the form $\frac{G^2}{S (1 - \beta) \e^2}$, which are worse than that of standard SGD by a factor of $\frac{1}{1 -\beta}$. Thus, ours is also the first result which theoretically showcases the usefulness of using large momentum in distributed and federated learning.

Our analysis is highly non-trivial and involves two crucial ingredients: i) computing the momentum at the server level to ensure that it remains unbiased and then applying it locally during every client update to reduce variance, and ii) carefully keeping track of the bias introduced via additional local steps. Our experiments (Sec.~\ref{sec:experiments}) verify our theoretical insights are indeed applicable in deep learning settings as well. See App.~\ref{sec:proof-overview} for a proof sketch and App.~\ref{sec:reduction-analysis}--\ref{sec:improvement-analysis} detailed proofs.\vspace{-2mm}


\section{Experimental analysis on real world datasets}\label{sec:experiments}

  We run experiments on \emph{natively} federated datasets to confirm our theory and accurately measure real world performance. Our main findings are i) \mime and \mimelite consistently outperform \fedavg, and ii) momentum and adaptivity significantly improves performance.
\vspace{-3mm}
\subsection{Setup}
\vspace{-2mm}
  \paragraph{Algorithms.} We consider three (meta) algorithms: \fedavg, \mime, and \mimelite. Each of these adapt four base optimizers: SGD, momentum, Adam, and Adagrad.\\
  \fedavg follows \cite{reddi2020adaptive} who run multiple epochs of SGD on each client sampled, and then aggregate the net client updates. This aggregated update is used as a pseudo-gradient in the base optimizer (called server optimizer). The learning rate for the server optimizer is fixed to 1 as in \cite{wang2019slowmo}. This is done to ensure all algorithms have the same number of hyper-parameters. \\
  \mime and \mimelite follow Algorithm~\ref{alg:mime} and also run a fixed number of epochs on the client. However, note that this requires communicating both the full local-batch gradient as well as the parameter updates doubling the communication required to be sent by the client. For a fairer comparison, we split the sampled clients in \mime and \mimelite into two groups--the first communicates only full local-batch gradient and the latter communicates only parameter updates. Thus, all methods have \textbf{equal client communication} to the server. This variant retains the convergence guarantees up to constants (details in the Appendix).
   We also run Loc-\mime where instead of keeping the global optimizer state fixed, we update it locally within the client. The optimizer state is reset after the round finishes.
  In all methods, aggregation is weighted by the number of samples on the clients.\vspace{0mm}



  \textbf{Datasets and models.} We run five simulations on three real-world federated datasets: EMNIST62 with i) a linear classifier, ii) an MLP, and iii) a CNN, iv) a charRNN on Shakespeare, and v) an LSTM for next word prediction on StackOverflow, all accessed through Tensorflow Federated~\cite{tfd2020}. The learning rates were individually tuned and other optimizer hyper-parameters such as $\beta$ for momentum, $\beta_1$, $\beta_2$, $\varepsilon_0$ for Adam and AdaGrad were left to their default values, unless explicitly stated otherwise. We refer to Appendix~\ref{sec:add-experiments-setup} for additional setup details and discussion.

    \begin{figure}[!t]
      \centering
      \captionsetup[subfigure]{position=b,format=myformat}
      \begin{subfigure}{.36\columnwidth}
        \centering
        \includegraphics[width=\linewidth]{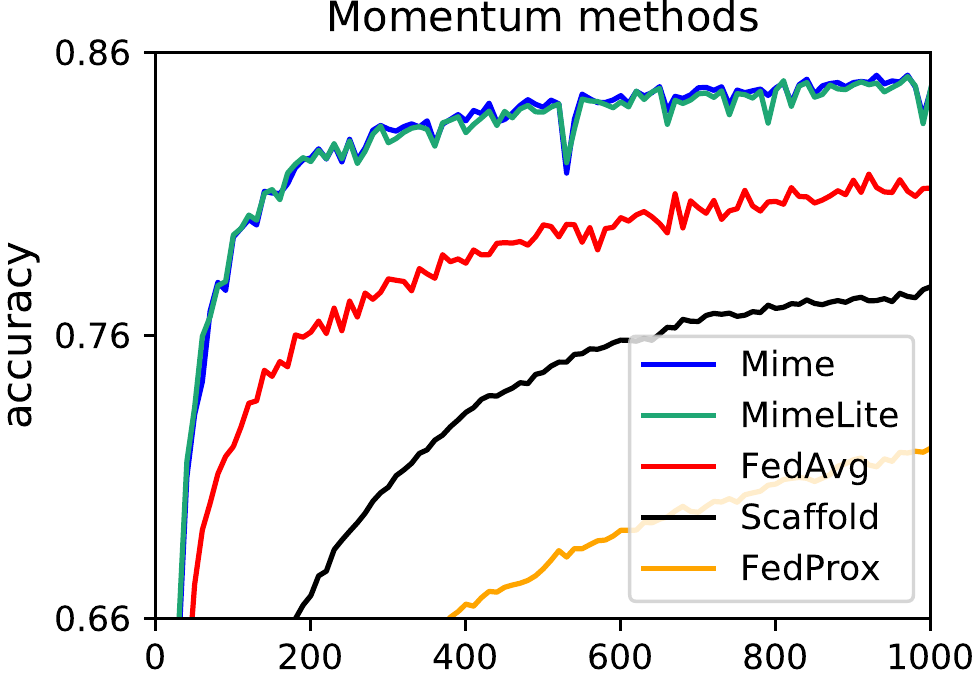}
      \end{subfigure}
      \begin{subfigure}{.31\columnwidth}
        \centering
        \includegraphics[width=\linewidth]{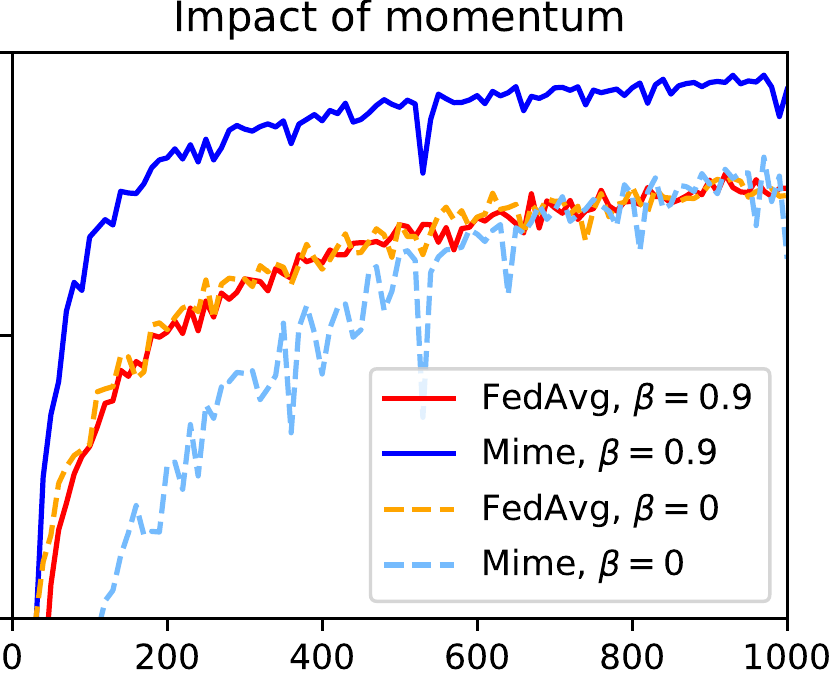}
      \end{subfigure}
      \begin{subfigure}{.31\columnwidth}
        \centering
        \includegraphics[width=\linewidth]{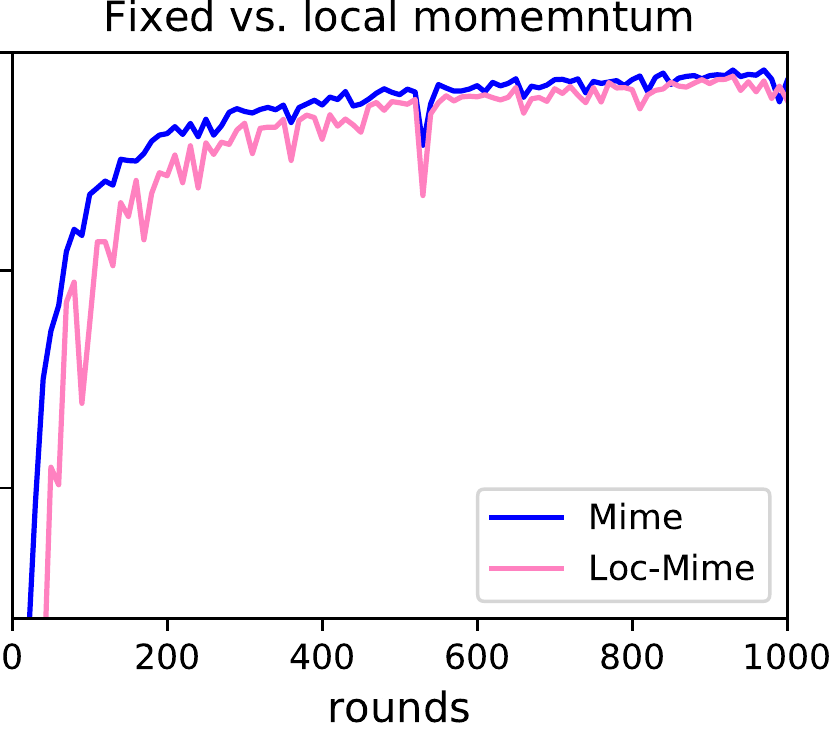}
      \end{subfigure}
      \caption{{\color{myblue}\textbf{Mime}}, {\color{mygreen}\textbf{MimeLite}}, {\color{myred}\textbf{FedAvg}}, \textbf{Scaffold}, {\color{orange}\textbf{FedProx}}, and {\color{pink}\textbf{Loc-Mime}} with SGD+momentum using 10 local epochs, run on EMNIST62 and a 2 hidden layer (300u-100) MLP. (Left) Mime and MimeLite are nearly identical and outperform the rest ($7\times$ faster). (Center) Mime makes better use of momentum than FedAvg, with a large increase in performance. (Right) Locally adapting momentum slows down convergence and makes it more unstable.}
      \label{fig:experiments-curves}\vspace{-5mm}
    \end{figure}
\vspace{-3mm}    
\subsection{Ablation and comparative study}
\vspace{-3mm}
In order to study the different algorithms, we train a 2 hidden layer ($300\mu$-$100$) MLP on EMNIST62 with 10 local epochs for 1k rounds and use SGD+momentum (with tuned $\beta$) as the base optimizer.
\vspace{-3mm}

\paragraph{Mime $\approx$ MimeLite $>$ FedAvg $>$ SCAFFOLD $>$ FedProx.} Fig.~\ref{fig:experiments-curves} (left) shows \mime and \mimelite have nearly identical performance, and are about $7\times$ faster than FedAvg. This implies our strategy of applying momentum to client updates is faster than simply using server momentum.
FedProx~\cite{li2018fedprox} uses an additional regularizer $\mu$ tuned over $[0.1, 0.5, 1]$ ($\mu =0$ is the same as FedAvg). Regularization does not seem to reduce client drift but still slows down convergence \cite{wang2020tackling}. SCAFFOLD \cite{karimireddy2019scaffold} is also slower than Mime and FedAvg in this setup. This is because in cross-device setting with a large number of clients ($N = 3.4k$) means that each client is visited less than 6 times during the entire training (20 clients per round for 1k rounds). Hence, the client control variate stored is quite stale (from about 200 rounds ago) which slows down the convergence. 
\vspace{-4mm}

\paragraph{With momentum $>$ without momentum.} Fig.~\ref{fig:experiments-curves} (center) examines the impact of momentum on FedAvg and Mime. Momentum slightly improves the performance of FedAvg, whereas it has a significant impact on the performance of Mime. This is also in line with our theory and confirms that Mime’s strategy of applying it locally at every client update makes better use of momentum.
\vspace{-4mm}

\paragraph{Fixed $>$ locally updated optimizer state.} Finally, we check how the performance of Mime changes if instead of keeping the momentum fixed throughout a round, we let it change. The latter is a way to combine global and local momentum. The momentum is reset at the end of the round ignoring the changes the clients make to it. Fig.~\ref{fig:experiments-curves} (right) shows that this \emph{worsens} the performance, confirming that it is better to keep the global optimizer state fixed as predicted by our theory.

Together, the above observations validate all aspects of Mime (and MimeLite) design: compute
statistics at the server level, and apply them unchanged at every client update.
\vspace{-3mm}

\begin{table*}[t!]
  \centering
    \caption{Validation \% accuracies after training for 1000 rounds. Best results for each dataset is underlined and the best within each base optimizer is bolded. The number of clients sampled per round has been reduced for \mime and \mimelite to ensure all methods have \textbf{equal client and server communication}. Final accuracies obtained by \mime and \mimelite are competitive with \fedavg, especially with adaptive base optimizers. \fedavg seems unstable with Adam.
    }
    \renewcommand{\tabcolsep}{3pt}
    \begin{tabular}{@{}l p{1.2cm} r r r r@{}}
        \toprule
         &  & \raggedleft EMNIST logistic &\raggedleft EMNIST CNN & \raggedleft Shakespeare & \raggedleft StackOverflow \tabularnewline
        \cmidrule{2-6}
        {\sc SGD} & FedAvgSGD & 66.8 & \textbf{85.8} & \textbf{56.7} & \textbf{23.8}\tabularnewline
         & MimeLiteSGD & 66.8 & \textbf{85.8} & \textbf{56.7} & \textbf{23.8}\tabularnewline
        & MimeSGD & \textbf{67.4} & 85.3 & 56.1 & 12.5\vspace{2mm}\\
        {\sc Momentum} & FedAvgMom  & 67.4 & 85.7 & \textbf{55.4} & \textbf{22.2} \tabularnewline
        & MimeLiteMom & 67.4 & \textbf{86.0} & 49.8 & 19.9 \tabularnewline
        & MimeMom  & \textbf{67.5} & 85.9 & 53.6 & 19.3\vspace{2mm}\\
        {\sc Adam} & FedAvgAdam & 67.3 & 85.9 & 18.5 & 3.2 \tabularnewline
        & MimeLiteAdam & \textbf{\underline{68.0}} & 86.4 & 54.0 & 21.5 \tabularnewline
        & MimeAdam & \textbf{\underline{68.0}} & \textbf{\underline{86.6}} & \textbf{54.1} & \textbf{22.8}\vspace{2mm}\\
        {\sc Adagrad} & FedAvgAdagrad & \textbf{67.6} & 86.3 & 55.5 & \textbf{\underline{24.2}} \tabularnewline
        & MimeLiteAdagrad & 66.6 & 85.5 & 56.8 & 23.8 \tabularnewline
        & MimeAdagrad & 67.4 & \textbf{86.3} & \textbf{\underline{57.1}} & 14.7 \tabularnewline
         \bottomrule
    \end{tabular}
    \label{tab:accuracies}\vspace{-3mm}
    \end{table*}
\subsection{Large scale comparison with equal server and client communication}\vspace{-3mm}
We perform a larger scale study closely matching the setup of \cite{reddi2020adaptive}.
For both \mime and \mimelite, only half the clients compute and transmit the updated parameters, and other half transmit the full local-batch gradients. Hence, client to server communication cost is the same for all methods for all clients. However, \mime and \mimelite require sending additional optimization state to the clients. Hence, we also reduce the number of clients sampled in each round to ensure \emph{sum total} of communication at each round is $40\times$ model size for EMNIST and Shakespeare experiments, and $100\times$ model size for the StackOverflow next word prediction experiment. \vspace{-1mm}

Since we only perform 1 local epoch, the hyper-parameters (e.g. epsilon for adaptive methods) are more carefully chosen following \cite{reddi2020adaptive}, and \mime and \mimelite use significantly fewer clients per round, the difference between \fedavg and \mime is smaller here.  Table~\ref{tab:accuracies} summarizes the results.\vspace{-1mm} 

For the image classification tasks of EMNIST62 logistic and EMNIST62 CNN, Mime and MimeLite with Adam achieve the best performance. Using momentum (both with SGDm, and in Adam) significantly improves their performance. In contrast, FedAvgAdam is more unstable with worse performance.
This is because FedAvg is excessively sensitive to hyperparameters (cf. App.~\ref{addsec:stability}).\vspace{-1mm} 
  
We next consider the character prediction task on Shakespeare dataset, and next word prediction on StackOverflow. Here, the momentum based methods (SGDm and Adam) are slower than their non-momentum counterparts (SGD and AdaGrad). This is because the mini-batch gradients in these tasks are \emph{sparse}, with the gradients corresponding to tokens not in the mini-batch being zero. This sparsity structure is however destroyed when using momentum or Adam. For the same reason, Mime which uses an SVRG correction also significantly increases the gradient density.
\vspace{-3mm}  

\paragraph{Discussion.} For traditional deep learning tasks such as image classification, we observe that Mime outperforms MimeLite which in turn outperforms FedAvg. These methods are able to successfully leverage momentum to improve performance. For tasks where the client gradients are sparse, the SVRG correction used by Mime hinders performance. Adapting our techniques to work with sparse gradients (\`{a} la Yogi \cite{zaheer2018adaptive}) could lead to further improvements. Also, note that we reduce communication by na\"{i}vely reducing the number of participating clients per round. More sophisticated approaches to save on client communication including quantization or sparsification  \cite{suresh2017distributed,alistarh2017qsgd}, or even novel algorithmic innovations \cite{acar2021federated} could be explored. Further, server communication could be reduced using memory efficient optimizers e.g. AdaFactor~\cite{shazeer2018adafactor} or SM3~\cite{anil2019memory}.

\vspace{-4mm}

\section{Conclusion}\vspace{-3mm}
Our work initiated a formal study of the cross-device federated learning problem and provided theoretically justified algorithms.
 We introduced a new framework \mime which overcomes the natural client-heterogeneity in such a setting, and can adapt arbitrary centralized algorithms such as Adam without additional hyper-parameters. We demonstrated the superiority of \mime via strong convergence guarantees and empirical evaluations. Further, we proved that a particular instance of our method, MimeMVR, beat centralized lower-bounds, demonstrating that additional local steps can yield asymptotic improvements for the first time. We believe our analysis will be of independent interest beyond the federated setting for understanding the sample complexity of non-convex optimization, and for yielding improved analysis of decentralized optimization algorithms.
{
  \bibliography{papers}
  \bibliographystyle{plain}
}


\newpage
\part*{Supplementary material for \mime}
\appendix

\renewcommand{\contentsname}{Contents of Appendix}
\tableofcontents
\addtocontents{toc}{\protect\setcounter{tocdepth}{3}} 
\clearpage


\section{How momentum can help reduce client drift}
\begin{figure*}[ht]
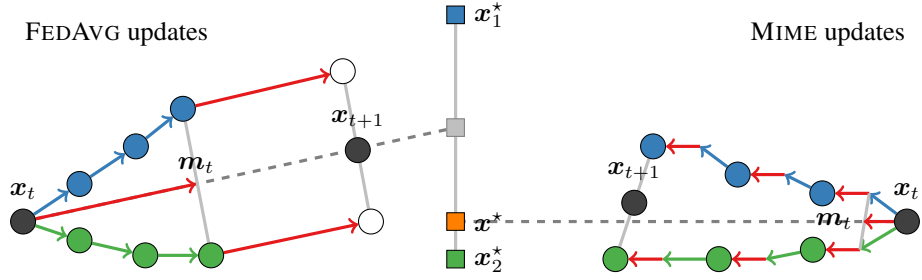

  \centering
  \ctikzfig{sgd}   
  \caption{Client-drift in \fedavg (left) and \mime (right) is illustrated for 2 clients with 3 local steps and momentum parameter $\beta=0.5$. The local SGD updates of \fedavg (shown using arrows for {\color{blue}client 1} and {\color{mygreen}client2}) move towards the {\color{gray} average of client optima $\frac{\xx_1^\star + \xx_2^\star}{2}$} which can be quite different from the true {\color{orange} global optimum $\xx^\star$}. Server {\color{myred}momentum} only speeds up the convergence to the wrong point in this case. In contrast, \mime uses unbiased {\color{myred}momentum} and applies it locally at every update. This keeps the updates of \mime closer to the true {\color{orange} optimum $\xx^\star$}.}\label{fig:client drift}
\end{figure*}
In this section we examine the tension between reducing communication by running multiple client updates each round, and degradation in performance due to client drift \cite{karimireddy2019scaffold}.
To simplify the discussion, we assume a single client is sampled each round and that clients use full-batch gradients.
\vspace{-2mm}
\paragraph{Server-only approach.} A simple way to avoid the issue of client drift is to take no local steps. We sample a client $i \sim \cD$ and run SGD with momentum (SGDm) with momentum parameter $\beta$ and step size $\eta$:
\begin{equation}\label{eqn:cent-sgdm}
  \begin{split}
    \xx_{t} &= \xx_{t-1} - \eta\, ((1-\beta)\nabla f_i(\xx_{t-1}) +\beta\mm_{t-1})\,,\\
    \mm_{t} &= (1-\beta)\nabla f_i(\xx_{t-1}) +\beta\mm_{t-1}\,.
  \end{split}
\end{equation}
Here, the gradient $\nabla f_i(\xx_t)$ is \emph{unbiased} i.e. $\expect[\nabla f_i(\xx_t)] = \nabla f(\xx_t)$ and hence we are guaranteed convergence. However, this strategy can be communication-intensive and we are likely to spend all our time waiting for communication with very little time spent on computing the gradients.

\paragraph{\fedavg approach.} To reduce the overall communication rounds required, we need to make more progress in each round of communication. Starting from $\yy_{0} = \xx_{t-1}$, \fedavg \cite{mcmahan2017communication} runs multiple SGD steps on the sampled client $i \sim \cD$
\begin{equation}\label{eqn:fed-sgdm}
  \begin{split}
    \yy_{k} &= \yy_{k-1} - \eta\nabla f_i(\yy_{k-1}) \text{ for } k \in [K]\,,
  \end{split}
\end{equation}
and then a pseudo-gradient $\tilde \gg_t = -(\yy_K - \xx_t)$ replaces $\nabla f_i(\xx_{t-1})$ in the SGDm algorithm \eqref{eqn:cent-sgdm}. This is referred to as server-momentum since it is computed and applied only at the server level~\cite{hsu2019measuring}. 
However, such updates give rise to \emph{client-drift}
resulting in performance worse than the na\"ive server-only strategy \eqref{eqn:cent-sgdm}. This is because by using multiple local updates, \eqref{eqn:fed-sgdm} starts over-fitting to the local client data, optimizing $f_i(\xx)$ instead of the actual global objective $f(\xx)$. The net effect is that \fedavg moves towards an incorrect point (see Fig~\ref{fig:client drift}, left). If $K$ is sufficiently large, approximately 
\begin{align*}
    \yy_K &\rightsquigarrow \xx_i^\star \,,\  \text{ where } \xx_i^\star := \argmin_{\xx} f_i(\xx)\\
    \Rightarrow  \expect_{i \sim \cD}[\tilde \gg_t] &\rightsquigarrow (\xx_t - \expect_{i \sim \cD}[\xx_i^\star])\,.
\end{align*}
Further, the server momentum is based on $\tilde \gg_t$ and hence is also biased. Thus, it cannot correct for the client drift. We next see how a different way of using momentum can mitigate client drift.

\paragraph{Mime approach.} \fedavg experiences client drift because both the momentum and the client updates are biased. To fix the former, we compute momentum using only global optimizer state as in~\eqref{eqn:cent-sgdm} using the sampled client $i \sim \cD$:
\begin{equation}\label{eqn:mimesgdm-momentum}
  \mm_t = (1- \beta)\nabla f_i(\xx_{t-1}) + \beta \mm_{t-1}\,.
\end{equation}
To reduce the bias in the local updates, we will apply this unbiased momentum every step $k \in [K]$:
\begin{equation}\label{eqn:mimesgdm-update}
  \begin{split}
    \yy_{k} &= \yy_{k-1} - \eta ((1-\beta)\nabla f_i(\yy_{k-1}) + \beta \mm_{t-1})\,.
  \end{split}
\end{equation}
Note that the momentum term is kept fixed during the local updates i.e.\ there is no local momentum used, only global momentum is applied locally. Since $\mm_{t-1}$ is a moving average of unbiased gradients computed over multiple clients, it intuitively is a good approximation of the general direction of the updates. By taking a convex combination of the local gradient with $\mm_{t-1}$, the update \eqref{eqn:mimesgdm-update} is potentially also less biased. In this way \mime combines the communication benefits of taking multiple local steps and prevents client-drift (see Fig~\ref{fig:client drift}, right). Appendix~\ref{sec:proof-overview} makes this intuition precise.


\section{Proof sketch}\label{sec:proof-overview}

In this section, we give proof sketches of the main components of Theorem~\ref{thm:interpolation}: i)~how momentum reduces the effect of client drift, ii) how local steps can take advantage of Hessian similarity, and iii) why the SVRG correction improves constants. \vspace{-3mm}



\paragraph{Improving the statistical term via momentum.} Note that the statistical (first) term in Theorem~\ref{thm:interpolation} without momentum $(\beta=0)$ for the convex case is $\frac{LG^2}{\mu S \e}$. This is (up to constants) optimal and cannot be improved. For the non-convex case however using $\beta=0$ gives the usual rate of $\frac{LG^2}{S \e^2}$. However, this can be improved to $\rbr[\Big]{\frac{(1 + \delta) G^2 F}{S\e^2}}^{3/4}$ using momentum. This matches a similar improvement in the centralized setting \cite{cutkosky2019momentum,tran2019hybrid} and is in fact optimal \cite{arjevani2019lower}.
Let us examine why momentum improves the statistical term. Assume that we sample a single client $i_t$ in round $t$ and that we use full-batch gradients. Also let the local client update at step $k$ round $t$ be of the form \vspace{-1mm}
\begin{equation}\label{eqn:local-update}
  \yy \leftarrow \yy - \eta \dd_k\,.
\end{equation}
 The ideal choice of update is of course $\dd_k^\star = \nabla f(\yy)$ but however this is unattainable.
Instead, \mime with momentum $\beta = 1-a$ uses $\dd_k^{\text{SGDm}} = \tilde\mm_k \leftarrow a \nabla f_{i}(\yy) + (1-a) \mm_{t-1}$ where $\mm_{t-1}$ is the momentum computed at the server. The variance of this update can then be bounded as
\vspace{-1mm}
\begin{align*}
  \expect\norm{\tilde\mm_k - \nabla f(\yy)}^2  &\lesssim a^2 \expect\norm{\nabla f_{i_t}(\yy) - \nabla f(\yy)}^2 + (1-a) \expect\norm{\mm_{t-1} - \nabla f(\yy)}^2 \\
  &\approx a^2 G^2 + (1-a) \expect\norm{\mm_{t-1} - \nabla f(\xx_{t-2})}^2 \approx a G^2\,.
\end{align*}
The last step follows by unrolling the recursion on the variance of $\mm$. We also assumed that $\eta$ is small enough that $\yy \approx \xx_{t-2}$. This way, momentum can reduce the variance of the update from $G^2$ to $(a G^2)$ by using past gradients computed on different clients. To formalize the above sketch requires slightly modifying the momentum algorithm similar to \cite{cutkosky2019momentum}.

\vspace{-1mm}
\paragraph{Improving the optimization term via local steps.} The optimization (second) term in Theorem~\ref{thm:interpolation} for the convex case is $\frac{\delta K + L}{\mu K}$ and for the non-convex case (with or without momentum) is $\frac{\delta K + L}{\e K}$. In contrast, the optimization term of the server-only methods is $L/\mu$ and $L/\e$ respectively. Since in most cases $\delta \lll L$, the former can be significantly smaller than the latter. This rate also suggests that the best choice of number of local updates is $L/\delta$ i.e. we should perform more client updates when they have more similar Hessians. This generalizes results of \cite{karimireddy2019scaffold} from quadratics to all functions.


This improvement is due to a careful analysis of the \emph{bias} in the gradients computed during the local update steps. Note that for client parameters $\yy_{k-1}$, the gradient $\expect[\nabla f_{i_t}(\yy_{k-1})] \neq \expect[\nabla f(\yy_{k-1})]$ since $\yy_{k-1}$ was also computed using the same loss function $f_{i_t}$. In fact, only the first gradient computed at $\xx_{t-1}$ is unbiased. Dropping the subscripts $k$ and $t$, we can bound this bias as:
\begin{align*}
  \expect[\nabla f_{i}(\yy) - \nabla f(\yy)] &=\expect [\underbrace{\nabla f_{i}(\yy) - \nabla f_{i}(\xx)}_{\approx \nabla^2 f_{i}(\xx) (\yy - \xx)} + \underbrace{\nabla f(\xx) - \nabla f(\yy_i)}_{\approx \nabla^2 f(\xx)(\xx - \yy_i)}] + \underbrace{\expect_i[\nabla f_i(\xx)] - \nabla f(\xx)}_{=0 \text{ since unbiased}}\\
  &\approx \expect[(\nabla^2 f_i(\xx) - \nabla^2 f(\xx)) (\yy_i - \xx)] \approx \delta \expect[(\yy_i - \xx)]\,.
\end{align*}
Thus, the Hessian dissimilarity \eqref{asm:hessian-similarity} control the bias, and hence the usefulness of local updates. This intuition can be made formal using Lemma~\ref{lem:similarity}.
\vspace{-1mm}

\paragraph{Mini-batches via SVRG correction.}  In our previous discussion about momentum and local steps, we assumed that the clients compute full batch gradients and that only one client is sampled per round. However, in practice a large number ($S$) of clients are sampled and further the clients use mini-batch gradients. The SVRG correction reduces this within-client variance since
\begin{equation*}
  \text{Var}\rbr*{\nabla f_i(\yy_i; \zeta) - \nabla f_i(\xx; \zeta) + \tfrac{1}{\abs{\cS}} \tsum_{i \in \cS} \nabla f_i(\xx)} \lesssim L^2\norm{\yy_i - \xx}^2 + \frac{G^2}{S} \approx \frac{G^2}{S}\,.
\end{equation*}
Here, we used the smoothness of $f_i(\cdot; \zeta)$ and assumed that $\yy_i \approx \xx$ since we don't move too far within a single round. Thus, the SVRG correction allows us to use minibatch gradients in the local updates while still ensuring that the variance is of the order $G^2/S$.\vspace{-2mm}



\section{Experimental setup}\label{sec:add-experiments-setup}
\subsection{Description of ablation study}
We train a 2 hidden layer MLP with 300u-100 neurons on the EMNIST62 (extended MNIST) dataset \cite{cohen2017emnist}. The clients' data is separated according to the original authors of the characters \cite{caldas2018leaf}. All methods are augmented with momentum--Mime and MimeLite use momentum in the client updates, and the others use server momentum. The momentum parameter is searched over
$\beta \in [0, 0.9, 0.99]$.
For Adam, we fix $\beta_1 = 0.9$, $\beta_2 = 0.99$, and $\epsilon = 10^{-3}$. For both FedProx and SCAFFOLD, $\beta=0$ (no server momentum) yielded the best performance. For FedAvg, Mime, and MimeLite $\beta=0.9$ was the fastest. For FedProx, the regularization parameter $\mu$ was searched over $[0.1, 0.5, 1]$ and $\mu = 0.1$ had highest test accuracy.

\subsection{Description of large scale experiments}
We perform 4 tasks over 3 datasets: 
i) On the EMNIST62 dataset \cite{cohen2017emnist} we run a convex multi-class (62 classes) logistic regression model, and ii) a convolution model with two CNN layers and two dense layers and dropout. 
iii) On the {\sc Shakespeare} dataset, we train a single layer LSTM model with state size of 256 and embedding size of 8 to predict the next character \cite{mcmahan2017communication}. 
iv) Finally, on the {\sc StackOverflow} dataset \cite{so2021}, we train a next word prediction language model with embedding size of 96, a LSTM layer of size 670, and a vocabulary size of 1000. 
In all cases we report the top-1 test accuracy in our experiments. 

All datasets use the metadata indicating the original authors to separate them into multiple clients yielding naturally partitioned datasets. Table~\ref{tab:datasets} summarizes the statistics about the different datasets. Note that the average number of rounds a client participates in (computed as sampled clients$\times$number of rounds$/$number of clients) provides an indication of how much of the training data is seen with {\sc Shakespeare} being closest to the cross-silo setting and {\sc StackOverflow} representing the most cross-device in nature.

\begin{table}[!h]
    \caption{Details about the datasets used and experiment setting.}
    \centering
    \begin{tabular}{@{}l l l l l@{}} 
    \toprule
     & EMNIST62 & {\sc Shakespeare} & {\sc StackOverflow} \\
    \midrule
      Clients & 3,400 & 715 & 342,477 \\
      Examples & 671,585  & 16,068 & 135,818,730\\
      Batch size & 10 & 10 & 10\\
      Number of local epochs & 1 & 1 & 1\\
      Total number of rounds & 1000  & 1000  & 1000 \\
      Avg. rounds each client participates & 5.9 & 28 & 0.15\\
    
    
    \bottomrule
    \end{tabular}
    \label{tab:datasets}
    \end{table}
  
 \vspace{5mm}
\begin{table}[!h]
    \caption{Effective number of sampled clients.}
    \centering
    \begin{tabular}{@{}l l l l l l@{}} 
    \toprule
     & Total Comm. & EMNIST62 & {\sc Shakespeare} & {\sc StackOverflow} \\
    \midrule
      FedAvg & $2\times$ & 20 & 20 & 50 \\
      MimeLiteMom & $5\times$  & 8 & 8 & 20\\
      MimeLiteAdagrad & $5\times$ & 8 & 8 & 20\\
      MimeLiteAdam & $6\times$ & 6 & 6 & 16\\
      MimeMom & $6\times$ & 6 & 6 & 16\\
      MimeAdagrad & $6\times$ & 6 & 6 & 16\\
      MimeAdam & $7\times$ & 5 & 5 & 14\\
    
    
    \bottomrule
    \end{tabular}
    \label{tab:sampling}
    \end{table}
\vspace{5mm}

We use Tensorflow federated datasets \cite{tfd2020} to generate the datasets. Our federated learning simulation code is written in Jax \cite{frostig2018compiling} and is open-sourced at \url{redacted for anonymity}. 
Black and white was reversed in EMNIST62 (i.e.\ subtracted from 1) to make them similar to MNIST. 
The preprocessing for {\sc Shakespeare} and {\sc StackOverflow} datasets exactly matches that of \cite{reddi2020adaptive}.

  \subsection{Practicality of experiments} \label{sec:add-practical}
  In the experiments we only cared about the number of communication rounds, ignoring that \mime actually needs twice the number of bits per round and that the \serveronly methods have a much smaller computational requirement. This is standard in the federated learning setting as introduced by \cite{mcmahan2017communication} and is justified because most of the time in cross-device FL is spent in establishing connections with devices rather than performing useful work such as communication or computation. In other words, \emph{latency} and not bandwidth or computation are critical in cross device FL. However, one can certainly envision cases where this is not true. Incorporating communication compression strategies \cite{suresh2017distributed,alistarh2017qsgd,karimireddy2019error,vogels2019powersgd} or client-model compression strategies \cite{caldas2018expanding,frankle2019lottery,hamer2020fedboost} into our \mime framework can potentially address such issues and are important future research directions.

  As we already discussed previously, we believe both the datasets and the tasks being studied here are close to real world settings since they contain natural heterogeneity. We now discuss our choice of other parameters in the experiment setup (number of training rounds, sampled clients, batch-size, etc.)
  Each round of federated learning takes 3~mins in the real world and is relatively independent of the size of communication~\cite{bonawitz2019towards} implying that training 1000 rounds takes \textbf{2~days} even for small models. In contrast, running a centralized simulation takes about 15 mins. This underscores the importance of ensuring that the algorithms for federated learning converge in as few rounds as possible, as well as have very easy to set default hyper-parameters. Thus, in our experimental setup we keep all parameters other than the learning rate to their default values. In practice, this learning rate can be set by set using a small centralized dataset on the server (as in \cite{hard2020training}). Thus, it is crucial for federated frameworks to be able to translate algorithms which work well in centralized settings directly to the federated setting without additinal hyper-parameter tuning. The choice of batch size being 10 was made both keeping in mind the limited memory available to each client as well as to match prior work. Finally, while we limit ourselves to sampling 20--50 workers per round due to computational constraints, in real world FL thousands of devices are often available for training simultaneously each round \cite{bonawitz2019towards}. They also note that the probability of each of these devices being available has clear patterns and is far from uniform sampling. Conducting a large scale experimental study which mimics these alternate forms of heterogeneity is an important direction for future work.
  
  \subsection{Hyperparameter search}
  
  We run two hyper-parameter sweeps in our experiments: first a \emph{light} setup which is reported in the main paper, and one we believe reflects the real world performance, and second a \emph{heavy} tuning setting to showcase the performance of the methods as we vary the hyper-parameters.
  
  \paragraph{Light-sweep setting ($\mathbf{ 9\times}$).}
  For all SGDm methods, we pick momentum $\beta = 0.9$. For Adam methods, we fix $\beta_1=0.9$ and $\beta_2=0.99$, and $\varepsilon_0 =\num{1e-7}$. For Adagrad we use the default initialization value of $0.1$ and use $\varepsilon_0 =\num{1e-7}$.
  None of the algorithms use weight decay, clipping etc. The learning rate is then tuned to obtain the best test accuracy.
  For all experiments, unless explicitly mentioned otherwise, the learning rate is searched over a grid ($9\times$): 
  \[\eta \in [\num{1e0}, \num{1e-0.5}, \num{1e-1}, \num{1e-1.5}, \num{1e-2}, \num{1e-2.5}, \num{1e-3}, \num{1e-3.5}, \num{1e-4}]\,.\]
  The server learning rate for all methods is kept at its default value of $1$.
  
  \paragraph{Heavy-sweep setting ($\mathbf{567 \times}$).}
  For all SGDm methods, we pick momentum $\beta = 0.9$. For Adam methods, we fix $\beta_1=0.9$ and $\beta_2=0.99$. For Adagrad we use the default initialization value of $0.1$.
  None of the algorithms use weight decay, clipping etc. The learning rate is then tuned to obtain the best test accuracy.
  
  For all experiments, unless explicitly mentioned otherwise, the \textbf{client} learning rate is searched over a grid ($9\times$): 
  \[\eta_{\text{client}} \in [\num{1e0}, \num{1e-0.5}, \num{1e-1}, \num{1e-1.5}, \num{1e-2}, \num{1e-2.5}, \num{1e-3}, \num{1e-3.5}, \num{1e-4}]\,.\]
  
  Further, we also search for the \textbf{server} learning rate is searched over a grid ($9\times$): 
  \[\eta_{\text{server}} \in [\num{1e1}, \num{1e0.5}, \num{1e0}, \num{1e-0.5}, \num{1e-1}, \num{1e-1.5}, \num{1e-2}, \num{1e-2.5}, \num{1e-3}]\,.\]

  Finally, for the \textbf{adaptive} methods such as Adam and Adagrad, we also tune the $\varepsilon_0$ parameter over a grid ($7\times$):
  \[\varepsilon_0 \in [\num{1e0}, \num{1e-1}, \num{1e-2}, \num{1e-3}, \num{1e-4}, \num{1e-5}, \num{1e-6}, \num{1e-7}]\,.\]
  
  \subsection{Comparison with previous results}
    As far as we are aware, \cite{reddi2020adaptive} is the only prior work which conducts a systematic experimental study of federated learning algorithms over multiple realistic datasets. The algorithms comparable across the two works (e.g. FedSGD, FedSGDm, and FedAdam) have qualitatively similar performance except with one exception: FedAdam consistently underperforms FedSGDm. This difference, as we show later, is because FedAdam does not work with the default choices of hyper-parameters such as $\e$ and requires additional tuning. As we explain in Section~\ref{sec:add-practical}, we chose to keep these parameters to the default values of their centralized counterparts to compare methods in a `low-tuning' setting. We also point that while FedAdam struggles to perform in this setup, MimeAdam and MimeLiteAdam are very stable and even often outperform their SGD counterparts. 


\subsection{Additional algorithmic details}\label{sec:add-alg-details}

\begin{table}[H]
    \caption{Decomposing base algorithms into a parameter update ($\cU$) and statistics tracking ($\cV$).}
    \centering
    \begin{tabular}{@{}l l l l@{}} 
    \toprule
    Algorithm & Tracked statistics $\ss$ & Update step $\cU$ & Tracking step $\cV$\\
    \midrule
      SGD     & -- & $\xx - \eta \gg$ & --\\ [2mm]
      SGDm/Mom    & $\mm$ & $\xx - \eta ((1-\beta)\gg +\beta\mm)$ & $\mm = (1-\beta)\gg +\beta\mm$\\   [2mm]
      AdaGrad    & $\vv$ & $\xx - \frac{\eta}{\e + \sqrt{\vv}} \gg$ & $\vv =\gg^2 + \vv$\\ [2mm]
      Adam    & $\mm, \vv$ & $\xx - \frac{\eta}{\e + \sqrt{\vv}} ((1-\beta_1)\gg +\beta_1\mm)$ & {$\begin{aligned}
        \mm &= (1-\beta_1)\gg +\beta_1\mm\\[-1.5mm]
        \vv &= (1-\beta_2)\gg^2 +\beta_2\vv
      \end{aligned}$}\\  
    \bottomrule
    \end{tabular}
    \label{tab:add-alg-details}
    \end{table}

\section{Stability of methods to hyper-parameters}\label{addsec:stability}

    \begin{figure}[H]
      \centering
      \captionsetup[subfigure]{position=b,format=myformat}
      \begin{subfigure}{.31\columnwidth}
        \centering
        \includegraphics[width=\linewidth]{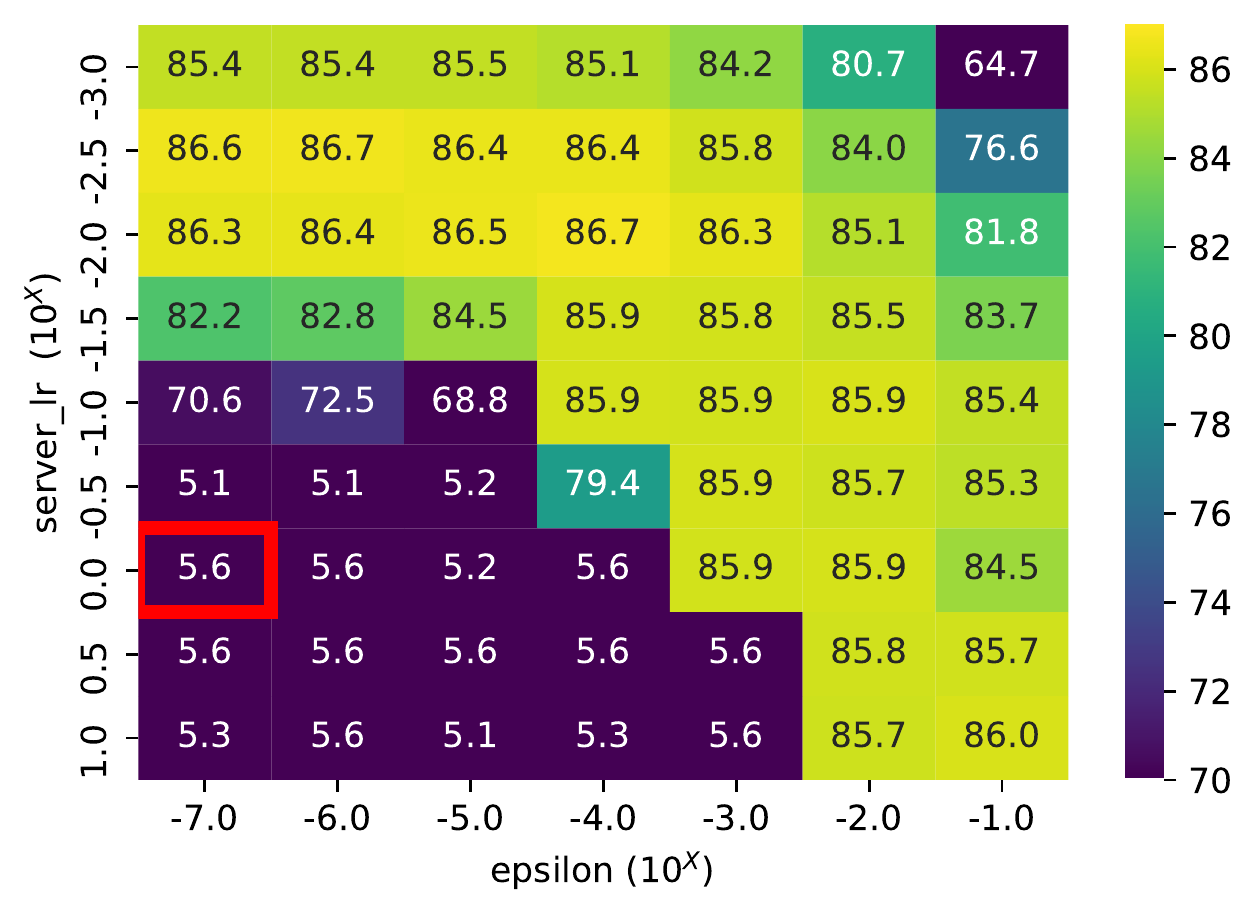}
      \end{subfigure}
      \begin{subfigure}{.31\columnwidth}
        \centering
        \includegraphics[width=\linewidth]{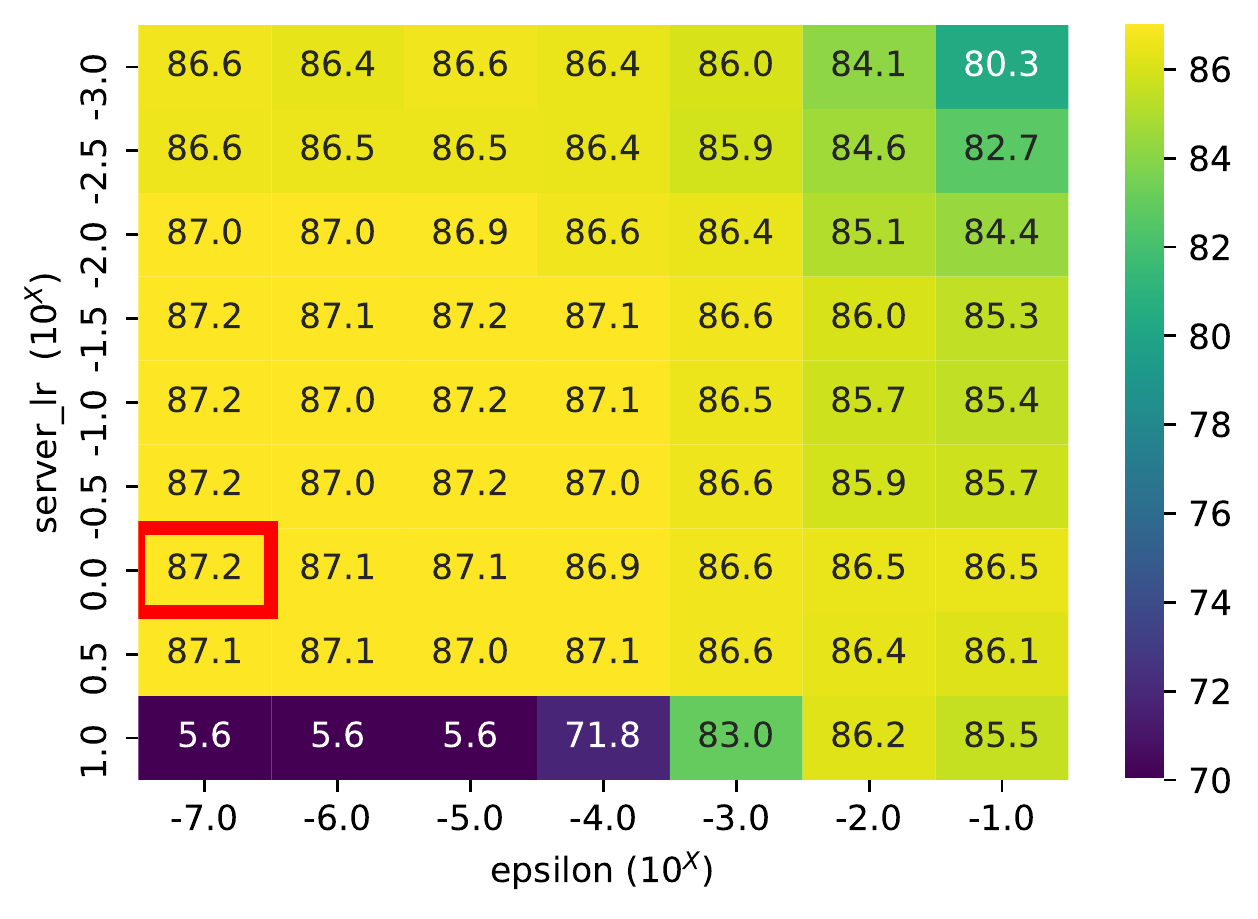}
      \end{subfigure}
      \begin{subfigure}{.31\columnwidth}
        \centering
        \includegraphics[width=\linewidth]{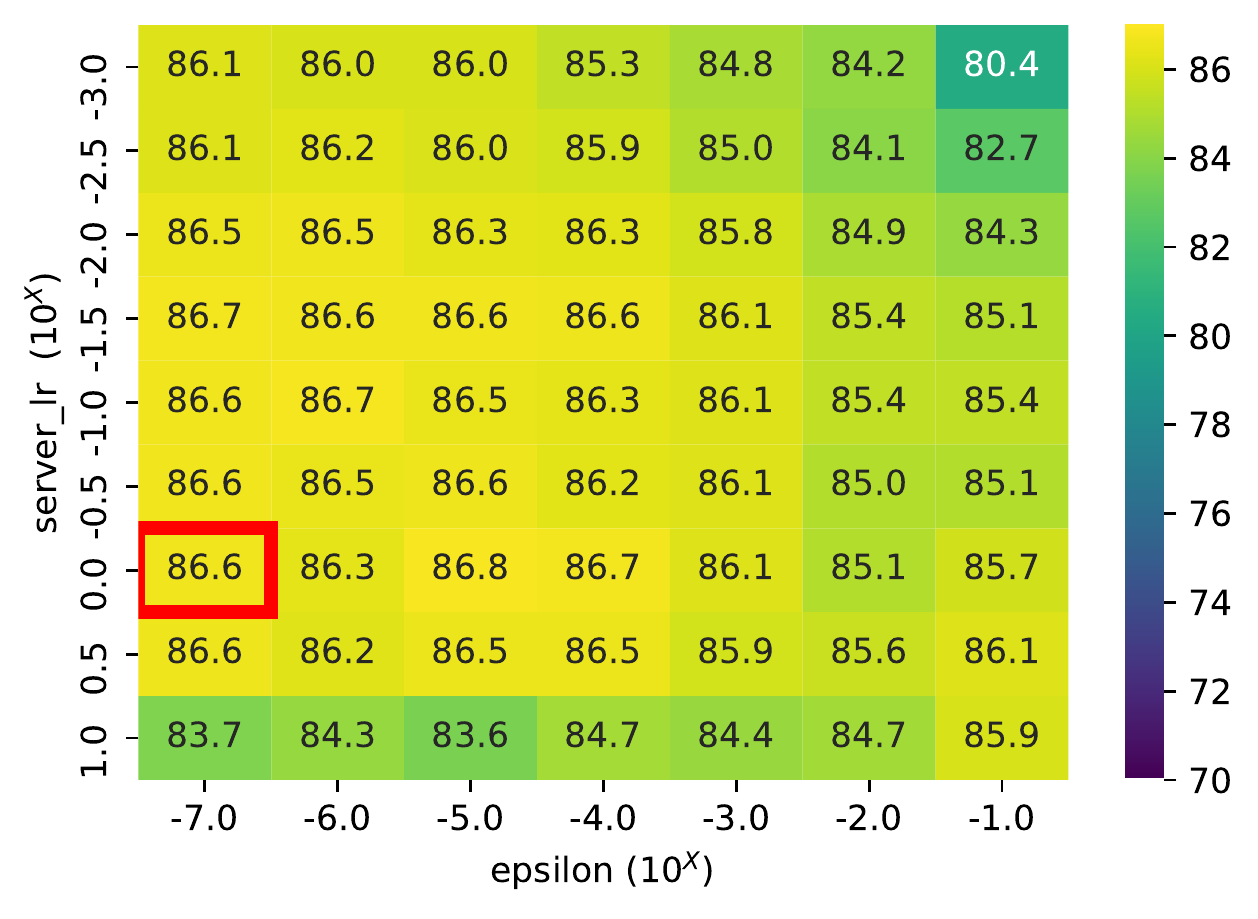}
      \end{subfigure}
      \\
      \begin{subfigure}{.31\columnwidth}
        \centering
        \includegraphics[width=\linewidth]{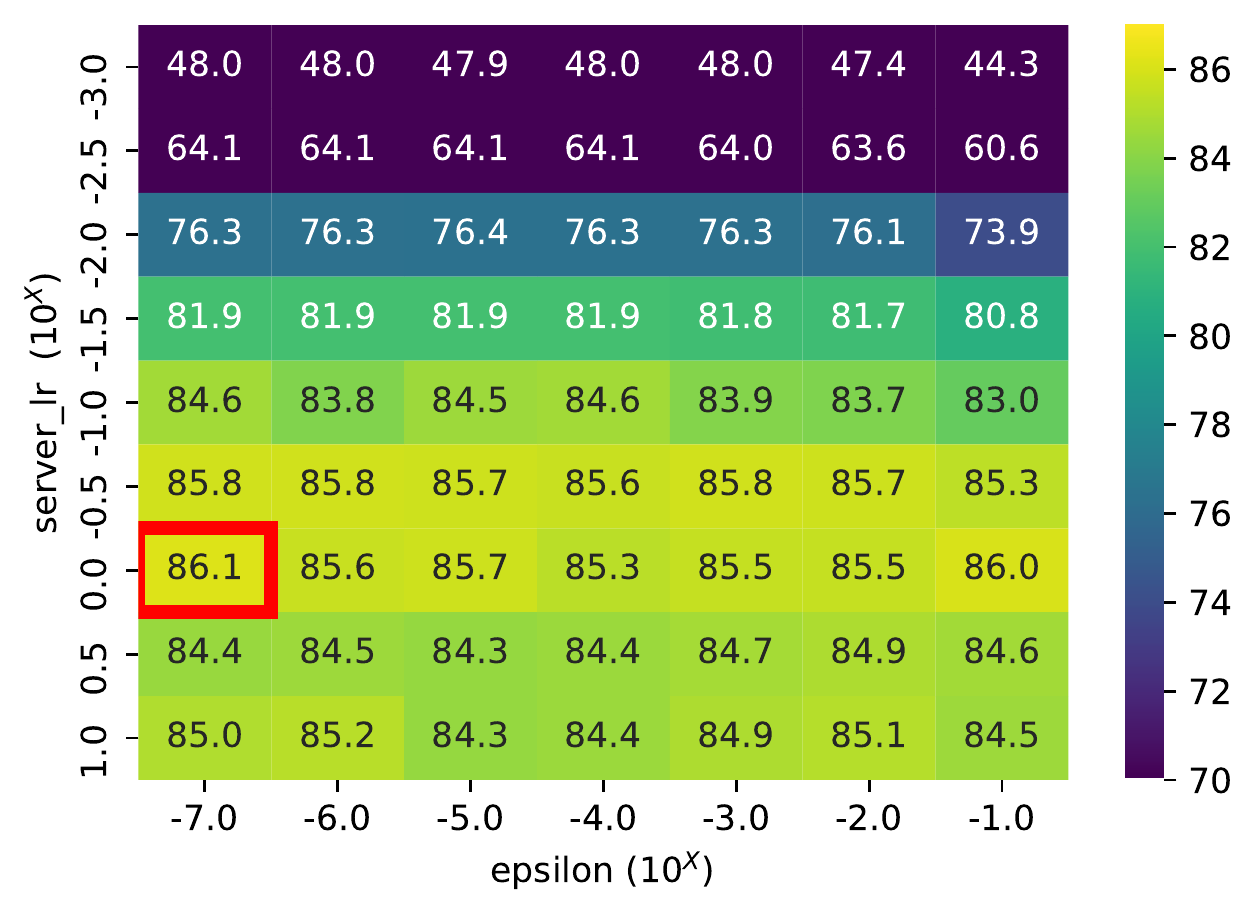}
      \end{subfigure}
      \begin{subfigure}{0.31\columnwidth}
        \centering
        \includegraphics[width=\linewidth]{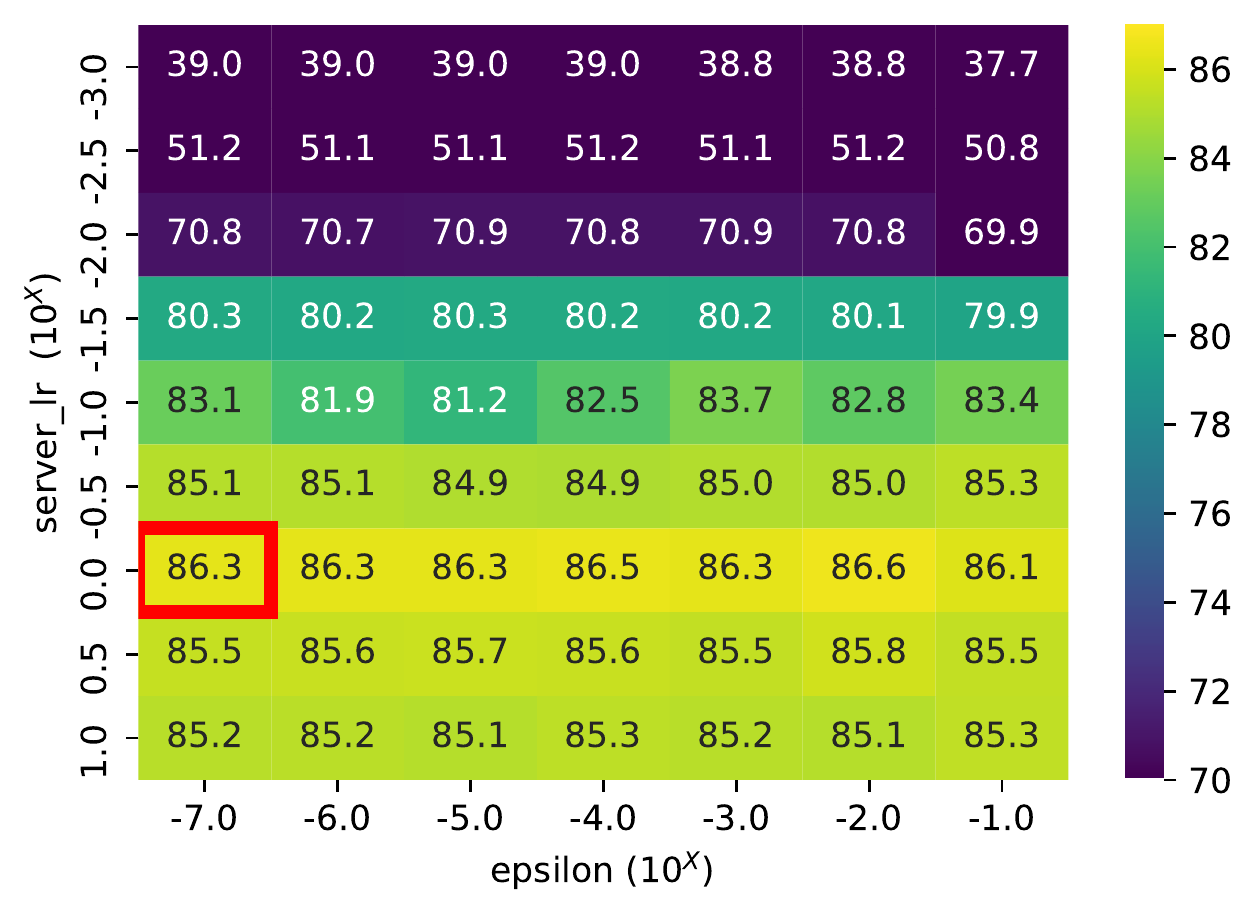}
      \end{subfigure}
      \begin{subfigure}{.31\columnwidth}
        \centering
        \includegraphics[width=\linewidth]{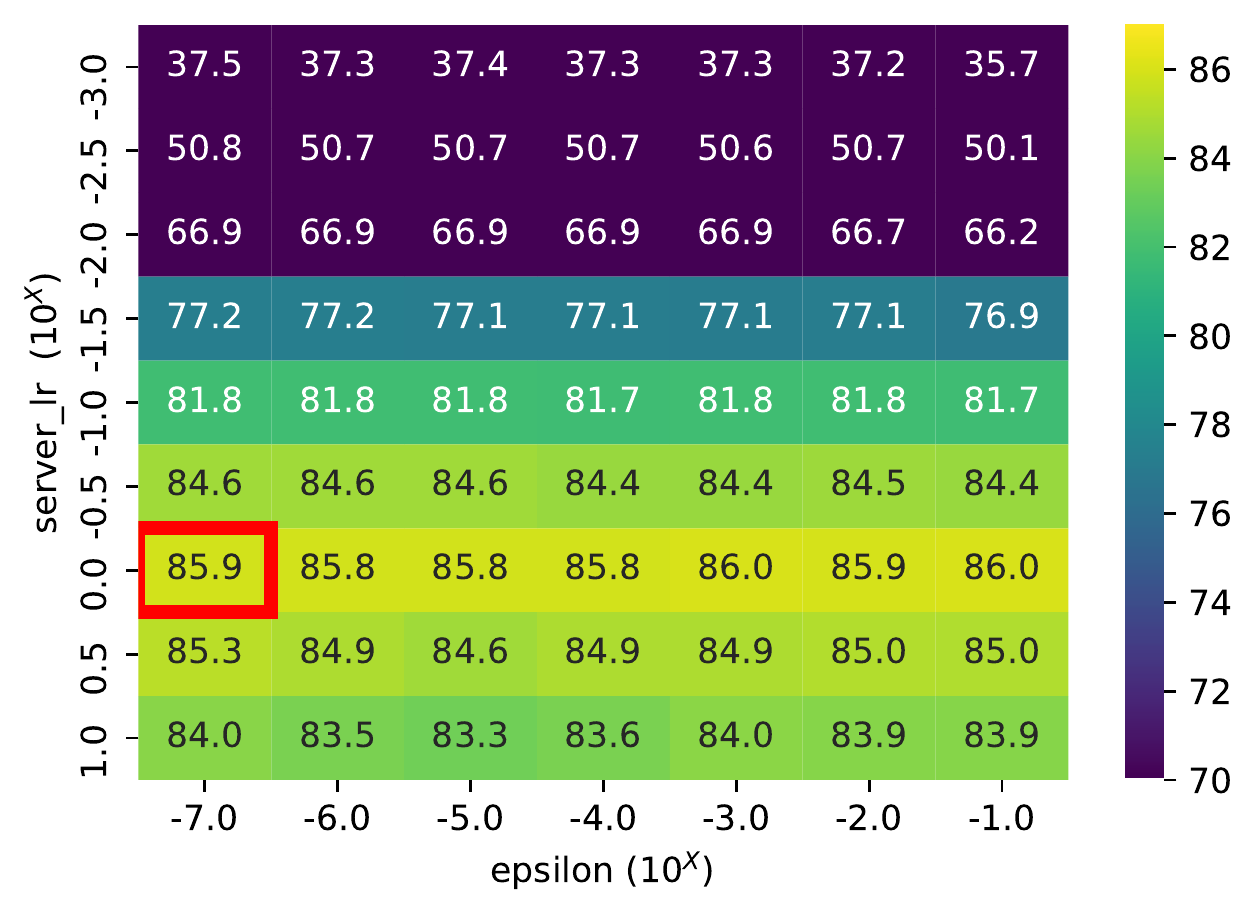}
      \end{subfigure}
      \caption{Stability of adaptive methods with varying server learning: FedAvg (left), Mime (middle) and MimeLite (right) with Adam (top) and Adagrad (bottom) as base algorithms are run on EMNIST62 with CNN. For each value of server learning rate ($y$-axis) and $\varepsilon_0$ ($x$-axis), the client learning rate was tuned over the $9\times$ grid and the accuracy reported. The red box highlights the default configuration in a centralized setting. We see that FedAdam is very sensitive to the server learning rate and $\varepsilon_0$, performing poorly in the default centralized parameter regimes. Mime and MimeLite acheive their best performance with the centralized parameters. This justifies our claim that Mime and MimeLite can \textbf{adapt} any centralized method with the same hyper-parameters and only require tuning of a single learning rate. This, we believe, is crucial for real world deployment.}
      \label{fig:experiments}
    \end{figure}
    
\begin{figure}[H]
    \centering
    \includegraphics[width=.5\linewidth]{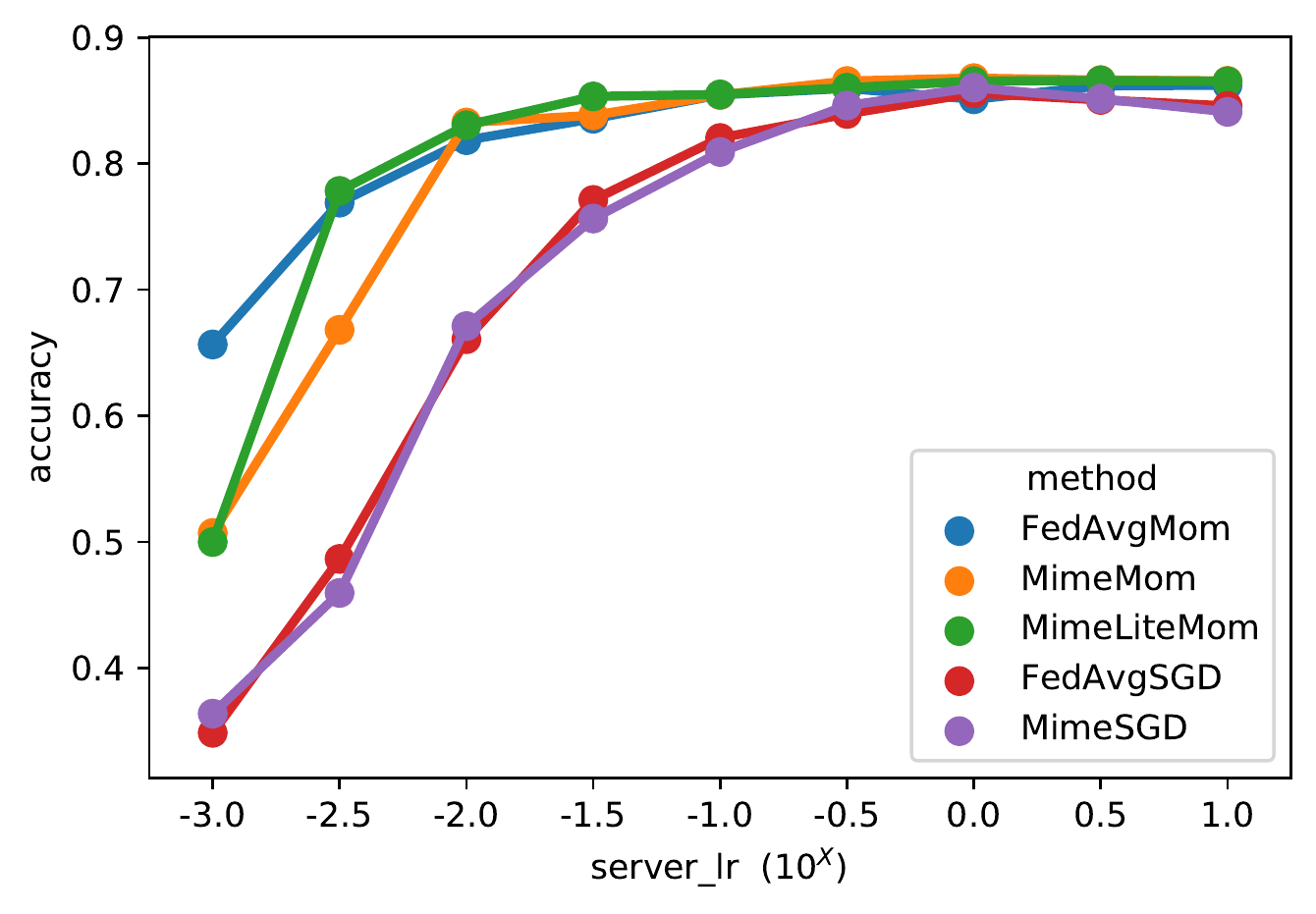}
  \caption{Stability of non-adaptive methods with varying server learning: FedAvg, Mime and MimeLite with SGD and momentum ($\beta=0.9$) as base algorithms are run on EMNIST62 with CNN. For each value of server learning rate, the client learning rate was tuned over the $9\times$ grid. The momentum methods are more insensitive to the server learning rate than the SGD methods. Server learning rate of 1 (default value) seems to work well for all methods.
}
  \label{fig:add-prox-scaffold-experiments}
\end{figure}

\section{Technicalities}\label{sec:technicalities}
We examine some additional definitions and introduce some technical lemmas.

\subsection{Assumptions and definitions}
We make precise a few definitions and explain some of their implications. We first discuss the two assumptions on the dissimilarity between the gradients \eqref{asm:heterogeneity} and the Hessians \eqref{asm:hessian-similarity}. Loosely, these two quantities are an extension of the concepts of \textbf{variance} and \textbf{smoothness} which occur in centralized SGD analysis to the federated learning setting. Just as the variance and smoothness are completely orthogonal concepts, we can have settings where $G^2$ (gradient dissimilarity) is large while $\delta$ (Hessian dissimilarity) is small, or vice-versa.

Our assumption about the bound on the $G$ gradient dissimilarity can easily be extended to $(G,B)$ gradient dissimilarity used by \cite{karimireddy2019error}:
\begin{equation}\label{eqn:heterogeneity}
  \expect_i\norm{\nabla f_i(\xx)}^2 \leq G^2 + B^2 \norm{\nabla f(\xx)}^2 \,.
\end{equation}
All the proofs in the paper extend in a straightforward manner to the above weaker notion. Since this notion does not present any novel technical challenge, we omit it in the rest of the proofs. Note however that the above weaker notion can potentially capture the fact that by increasing the model capacity, we can reduce $G$. In the extreme case, by taking a sufficiently over-parameterized model, it is possible to make $G=0$ in certain settings \cite{vaswani2018fast}. However, this comes both at a cost of increased resource requirements (i.e. higher memory and compute requirements per step) but can also result in other constants increasing (e.g. $B$ and $L$). 

The second crucial definition we use in this work is that of $\delta$ bounded \emph{Hessian} dissimilarity \eqref{asm:hessian-similarity}. This has been used previously in the analyses of distributed \cite{shamir2014communication,arjevani2015communication,reddi2016aide} and federated learning \cite{karimireddy2019scaffold}, but has been restricted to quadratics. Here, we show how to extend both the notion as well as the analysis to general smooth functions. The main manner we will use this assumption is in Lemma~\ref{lem:similarity} to claim that for any $\xx$ and $\yy$ the following holds:
\begin{equation}\label{eqn:similarity}
        \E\norm{\nabla f_i(\yy; \zeta) - \nabla f_i(\xx; \zeta) + \nabla f(\xx) - \nabla f(\yy)}^2 \leq \delta^2 \norm{\yy - \xx}^2\,.
\end{equation}
Here the expectation is over the choice of client $i$. To understand what the above condition means, it is illuminating to define $\Psi_i(\zz) = f_i(\zz; \zeta) - f(\zz)$. Then, we can rewrite \eqref{asm:hessian-similarity} and \eqref{eqn:similarity} respectively as 
\[
\norm{\nabla^2 \Psi_i(\zz)} \leq \delta \quad \text{ and } \quad \E\norm{\nabla \Psi_i(\yy) - \nabla \Psi_i(\xx)}^2 \leq \delta^2 \norm{\yy - \xx}^2\,.
\]
Thus \eqref{eqn:similarity} and \eqref{asm:hessian-similarity} are both different notions of smoothness of $\Psi_i(\xx)$ (formal definition of smoothness will follow soon). The latter definition closely matches the notion of \emph{squared-smoothness} used by \cite{arjevani2019lower} and is a promising relaxation of \eqref{asm:hessian-similarity}. However, we run into some technical issues since in our case the variable $\yy$ can also be a random variable and depend on the choice of the client $i$. Extending our results to this weaker notion of Hessian-similarity and proving tight non-convex lower bounds is an exciting theoretical challenge.

Finally note that if the functions $f_i(\xx; \zeta)$ are assumed to be smooth as in \cite{shamir2014communication,arjevani2015communication,karimireddy2019scaffold}, then $\Psi_i((\xx)$ is $2L$-smooth. Thus, we \emph{always} have that $\delta \leq 2L$.
But, as shown in \cite{shamir2014communication}, it is possible to have $\delta \lll L$ if the data distribution amongst the clients is similar. Further, the lower bound from \cite{arjevani2015communication} proves that Hessian-similarity is the crucial quantity capturing the number of rounds of communication required for distributed/federated optimization.

We next define the terms smoothness and strong-convexity which we repeatedly use in the paper.
\begin{enumerate}[leftmargin=26pt]
  \myitem{A2*}\label{asm:smoothness} $f_i$ is almost surely \textbf{L-smooth} and satisfies:
      \begin{equation}
        \label{eqn:lip-grad}
        \norm{\nabla f_i(\xx; \zeta) - \nabla f_i(\yy; \zeta)} \leq L \norm{\xx - \yy}\,, \text{ for any } \xx, \yy\,.
      \end{equation}
      The assumption \eqref{asm:smoothness} also implies the following quadratic upper bound on $f_i$
      \begin{equation}\label{eqn:quad-upper}
          f_i(\yy) \leq f_i(\xx) + \inp{\nabla f_i(\xx)}{\yy - \xx} + \frac{L}{2}\norm{\yy - \xx}^2\,.
      \end{equation}
        Further, if $f_i$ is twice-differentiable, \eqref{asm:smoothness} implies that $\norm{\nabla^2 f_i(\xx; \zeta)} \leq \beta$ for any $\xx$.
    \myitem{A3}\label{asm:noise}
    We assume that the \textbf{intra-client gradient variance} is bounded by $\sigma^2$. For any client $i$, the following holds almost surely at any fixed $\xx$:
    \begin{equation*}
        \E_{\zeta_i}[\nabla f_i(\xx; \zeta)] = \nabla f_i(\xx)\,, \quad \text{and} \quad \E_{\zeta_i}\norm{\nabla f_i(\xx; \zeta) - \nabla f_i(\xx)}^2 \leq \sigma^2\,.
      \end{equation*}
      Note that we expect the intra-client variance to be smaller than inter-client variance and so typically $\sigma^2 \leq G^2$.
      \myitem{A4}
      \label{asm:strong-convexity} $f$ satisfies the \textbf{$\mu$-PL inequality} \cite{karimi2016linear} for $\mu > 0$ if:
      \begin{equation*}\label{eqn:strong-convexity}
        \norm{\nabla f(\xx)}^2 \geq 2\mu(f(\xx) - f^\star)\,.
      \end{equation*}
        Note that PL-inequality is much weaker than the standard notion of strong-convexity, and in fact is even satisfied by some non-convex functions \cite{karimi2016linear}.      
      \end{enumerate}


\subsection{Some technical lemmas}

Now we cover some technical lemmas which are useful for
computations later on.
    First, we state a relaxed triangle inequality true for the squared
    $\ell_2$ norm.
\begin{lemma}[relaxed triangle inequality]\label{lem:norm-sum}
    Let $\{\vv_1,\dots,\vv_\tau\}$ be $\tau$ vectors in $\R^d$. Then the following are true:
    \begin{enumerate}
\item $\norm{\vv_i + \vv_j}^2 \leq (1 + c)\norm{\vv_i}^2 + (1 + \tfrac{1}{c})\norm{\vv_j}^2$ for any $c >0$, and
\item $\norm{\sum_{i=1}^\tau \vv_i}^2 \leq \tau \sum_{i=1}^\tau\norm{\vv_i}^2$.
    \end{enumerate}
\end{lemma}
\begin{proof}
The proof of the first statement for any $c > 0$ follows from the identity:
\[
    \norm{\vv_i + \vv_j}^2 = (1 + c)\norm{\vv_i}^2 + (1 + \tfrac{1}{c})\norm{\vv_j}^2 - \norm{\sqrt{c} \vv_i + \tfrac{1}{\sqrt{c}}\vv_j}^2\,.
\]
For the second inequality, we use the convexity of
$\xx \rightarrow \norm{\xx}^2$ and Jensen's inequality
\[
     \norm[\bigg]{\frac{1}{\tau}\sum_{i=1}^\tau \vv_i }^2 \leq \frac{1}{\tau}\sum_{i=1}^\tau\norm[\big]{ \vv_i }^2\,. \qedhere
\]
\end{proof}

Next we state an elementary lemma about expectations of norms of
random vectors.
\begin{lemma}[separating mean and variance]\label{lem:independent}
Let $\{\Xi_1,\dots,\Xi_{\tau}\}$ be $\tau$ random variables in $\R^d$ which are not necessarily independent. First suppose that their mean is $\E[\Xi_i] = \xi_i$ and variance is bounded as $\E[\norm{\Xi_i - \xi_i}^2]\leq \sigma^2$. Then, the following holds
\[
    \E[\norm{\sum_{i=1}^\tau \Xi_i}^2] \leq \norm{\sum_{i=1}^\tau \xi_i}^2+ \tau^2 \sigma^2\,.
\]
Now instead suppose that their \emph{conditional mean} is $\E[\Xi_i | \Xi_{i-1}, \dots \Xi_{1}] = \xi_i$ i.e. the variables $\{\Xi_i - \xi_i\}$ form a martingale difference sequence, and the variance is bounded by $\E[\norm{\Xi_i - \xi_i}^2]\leq \sigma^2$ as before. Then we can show the tighter bound
\[
    \E[\norm{\sum_{i=1}^\tau \Xi_i}^2] \leq 2\norm{\sum_{i=1}^\tau \xi_i}^2+ 2\tau \sigma^2\,.
\]
\end{lemma}
\begin{proof}
For any random variable $X$, $\E[X^2] = (\E[X - \E[X]])^2 + (\E[X])^2$ implying
\[
     \E[\norm{\sum_{i=1}^\tau \Xi_i}^2] = \norm{\sum_{i=1}^\tau \xi_i}^2 +  \E[\norm{\sum_{i=1}^\tau \Xi_i - \xi_i}^2]  \,.
\]
Expanding the above expression using relaxed triangle inequality (Lemma~\ref{lem:norm-sum}) proves the first claim:
\[
   \E[\norm{\sum_{i=1}^\tau \Xi_i - \xi_i}^2]  \leq \tau\sum_{i=1}^\tau \E[\norm{ \Xi_i - \xi_i}^2] \leq \tau^2\sigma^2\,.
\]
For the second statement, $\xi_i$ is not deterministic and depends on $\Xi_{i-1}, \dots, \Xi_1$. Hence we have to resort to the cruder relaxed triangle inequality to claim
\[
    \E[\norm{\sum_{i=1}^\tau \Xi_i}^2] \leq 2\norm{\sum_{i=1}^\tau \xi_i}^2 +  2\E[\norm{\sum_{i=1}^\tau \Xi_i - \xi_i}^2] 
\]
and then use the tighter expansion of the second term:
\[
   \E[\norm{\sum_{i=1}^\tau \Xi_i - \xi_i}^2]  = \sum_{i, j} \E\sbr*{(\Xi_i - \xi_i)^\top(\Xi_j - \xi_j)} = \sum_{i} \E\sbr*{\norm{\Xi_i - \xi_i}^2} \leq \tau \sigma^2\,.
\]
The cross terms in the above expression have zero mean since $\{\Xi_i - \xi_i\}$ form a martingale difference sequence.
\end{proof}


\subsection{Properties of functions with bounded Hessian dissimilarity}

We now study two lemmas which hold for any functions which satisfy \eqref{asm:hessian-similarity} and \eqref{asm:noise}. The first is closely related to the notion of smoothness \eqref{asm:smoothness}.
 \begin{lemma}[similarity]\label{lem:similarity}
    The following holds for any two functions $f_i(\cdot)$ and $f(\cdot)$ satisfying \eqref{asm:hessian-similarity} and \eqref{asm:noise}, and any $\xx, \yy$:
    \[
        \norm{\nabla f_i(\yy; \zeta) - \nabla f_i(\xx; \zeta) + \nabla f(\xx) - \nabla f(\yy)}^2 \leq \delta^2 \norm{\yy - \xx}^2\,.
    \]
\end{lemma}
\begin{proof}
    Consider the function $\Psi(\zz) : = f_i(\zz; \zeta) - f(\zz)$. By the assumption \eqref{asm:hessian-similarity}, we know that $\norm{\nabla^2 \Psi(\zz)} \leq \delta$ for all $\zz$ i.e. $\Psi$ is $\delta$-smooth. By standard arguments based on taking limits \cite{nesterov2018lectures}, this implies that
    \[
      \norm{\nabla \Psi(\yy) - \nabla \Psi(\xx)} \leq \delta\norm{\yy - \xx}\,.
      \]
      Plugging back the definition of $\Psi$ into the above inequality proves the lemma.
\end{proof}

Next, we see how weakly-convex functions satisfy a weaker notion of ``averaging does not hurt''. This is used to get a handle on the effect of averaging of parameters in FedAvg.

\begin{lemma}[averaging]\label{lem:averaging}
  Suppose $f$ is $\delta$-weakly convex. Then, for any $\gamma \geq \delta$, and a sequence of parameters $\{\yy_i\}_{i \in \cS}$ and $\xx$:
  \[
      \frac{1}{\abs{\cS}}\sum_{i \in \cS} f(\yy_i) + \frac{\gamma}{2}\norm{\xx - \yy_i}^2 \geq f(\bar\yy) + \frac{\gamma}{2}\norm{\xx - \bar\yy}^2  \,, \text{ where } \bar \yy := \frac{1}{\abs{\cS}}\sum_{i \in \cS}\yy_i\,.
  \]
\end{lemma}
\begin{proof}
  Since $f$ is $\delta$-weakly convex, $\Phi(\zz) := f(\zz) + \frac{\gamma}{2}\norm{\zz - \xx}^2$ is convex. This proves the claim since $\frac{1}{\abs{\cS}}\sum_{i \in \cS}\Phi(\yy_i) \leq \Phi(\bar \yy)$.
\end{proof}


\newpage
\section{Convergence with a generic base optimizer}\label{sec:reduction-analysis}

Let us rewrite the Mime and MimeLite updates using notation convenient for analysis. In each round $t$, we sample clients $\cS^{t}$ such that $\abs{\cS^t} = S$. The server communicates the server parameters $\xx^{t-1}$ as well as the average gradient across the sampled clients $\cc^t$ defined as
\begin{equation}\label{eqn:def-mimesgd-control}
  \cc^t = \frac{1}{S}\sum_{i\in \cS^t} \nabla f_i(\xx^{t-1})\,.
\end{equation} 
Note that computing $\cc^t$ (required only by Mime but not by MimeLite) itself requires additional communication. In this proof, we do not make any assumption on how $\cc^t$ is computed as long as it is unbiased and is computed over $S$ clients. In particular, it can either be computed on the sampled $\cS^t$ or a different set of an independent sampled clients $\tilde\cS^t$.

Then each client $i \in \cS^t$  makes a copy $\yy_{i, 0}^t = \xx^{t-1}$ and perform $K$ local client updates. In each local client update $k \in [K]$, the client samples a dataset $\zeta_{i,k}^t$ and 
\begin{align}\label{eqn:def-mimesgd-update}
  \yy_{i, k}^t &= \yy_{i, k-1}^t - \eta \cU(\nabla f_i(\yy_{i, k-1}^t; \zeta_{i,k}^t) - \nabla f_i(\xx^{t-1}; \zeta_{i,k}^t) + \cc^t; \ss^{t-1})\tag{Mime client update}\\
  &= \yy_{i, k-1}^t - \eta \cU(\nabla f_i(\yy_{i, k-1}^t; \zeta_{i,k}^t); \ss^{t-1})\tag{MimeLite client update}\,.
\end{align}
After $K$ such local updates, the server then aggregates the new client parameters as
\begin{align}\label{eqn:def-mimesgd-aggregation}
  \xx^t &= \frac{1}{S}\sum_{i\in\cS^t}\yy_{i,K}^t \tag{Update server parameters}\\
  \ss^t &= \cV(\cc^t, \ss^{t-1}) \tag{Update server statistics}\,.
\end{align}

\subsection{Proof of Theorem~\ref{thm:reduction} (generic reduction)}
\paragraph{Computing server update.}

\begin{lemma}[Deviation from central update.]\label{lem:reduction-server}
  For a linear updater $\cU$ for both Mime and MimeLite the server update can be written as
  \[
        \xx^{t} = \xx^{t-1} - \tilde\eta\cU\rbr*{\frac{1}{S}\sum_{i} \nabla f_i(\xx) + \badbox{\text{$\ee^t$}}; \ss^{t-1}}\,,
  \]
  for $\tilde\eta := K\eta$ and $\badbox{\text{$\ee^t$}} = \frac{1}{KS}\sum_{i, k} (\nabla f_i(\yy_{i, k-1}; \zeta_{i,k}) - \nabla f_i(\xx; \zeta_{i,k}))$.
\end{lemma}
\begin{proof}
Because the updater $\cU$ is linear in its first parameter, we can rewrite the update to the server for MimeLite as
\begin{align*}
    \xx^t - \xx^{t-1} &= \frac{1}{S}\sum_{i \in \cS^t}\sum_{k=1}^K -\eta \cU(\nabla f_i(\yy^t_{i, k-1}; \zeta^t_{i,k}); \ss^{t-1})\\
    &= \eta K \cU\rbr*{\frac{1}{KS}\sum_{i, k} \nabla f_i(\yy^t_{i, k-1}; \zeta^t_{i,k}); \ss^{t-1}}
\end{align*}
We drop the dependence on $t$ when obvious from context and $i$ by default sums over $\cS^t$ and $k$ over $[K]$ by default. Since $K$ represents a multiple of epochs, we have $\sum_k \nabla f_i(\xx; \zeta^t_{i,k}) = K \nabla f_i(\xx)$. Continuing,
\begin{align*}
    \xx^t - \xx^{t-1} &= \eta K \cU\rbr*{\frac{1}{KS}\sum_{i, k} \nabla f_i(\yy^t_{i, k-1}; \zeta^t_{i,k}); \ss^{t-1}}\\
    &= \tilde\eta\cU\rbr*{\frac{1}{S}\sum_{i} \nabla f_i(\xx) + \ee^t; \ss^{t-1}}
\end{align*}
where
\[
\ee^t = \frac{1}{KS}\sum_{i, k} (\nabla f_i(\yy_{i, k-1}; \zeta_{i,k}) - \nabla f_i(\xx; \zeta_{i,k}))
\]
Now let us examine the update of Mime. Again assuming $K$ is a multiple of epoch, we have $\sum_{i,k}\nabla f_i(\xx; \zeta^t_{i,k}) = K\sum_{i}\nabla f_i(\xx) = KS \xx$. Hence,
\begin{align*}
    \xx^t - \xx^{t-1} &= \frac{1}{S}\sum_{i \in \cS^t}\sum_{k=1}^K -\eta \cU(\nabla f_i(\yy_{i, k-1}; \zeta^t_{i,k})- \nabla f_i(\xx; \zeta^t_{i,k}) + \cc; \ss^{t-1})\\
    &= \eta K \cU\rbr*{\frac{1}{KS}\sum_{i, k} \nabla f_i(\yy^t_{i, k-1}; \zeta^t_{i,k}); \ss^{t-1}}\\
    &= \eta K \cU\rbr*{\frac{1}{S}\sum_{i} \nabla f_i(\xx) + \ee^t; \ss^{t-1}}\,.
\end{align*}
Thus we showed the lemma for both Mime and MimeLite.
\end{proof}

\begin{lemma}[Defining error]\label{lem:reduction-error}
  For $\ee^t$ defined in Lemma~\ref{lem:reduction-server}, assuming all functions $f_i(\,\cdot\,, \zeta)$ are $L$-smooth, we have
  \[
    \E\norm{\ee^t}^2 \leq L^2\underbrace{\frac{1}{KS}\sum_{i,k}\E\norm{\yy_{i,k-1} - \xx}^2}_{=: \cE_K^t }\,.
  \]
\end{lemma}
\begin{proof}
  Using the smoothness of the individual functions and the definition of $\ee^t$,
  \begin{align*}
  \E\norm{\ee^t}^2 &= \E\norm{ \frac{1}{KS}\sum_{i, k} (\nabla f_i(\yy_{i, k-1}; \zeta_{i,k}) - \nabla f_i(\xx; \zeta_{i,k}))}^2 \\
  &\leq \frac{1}{KS}\sum_{i,k}\E\norm{\nabla f_i(\yy_{i, k-1}; \zeta_{i,k}) - \nabla f_i(\xx; \zeta_{i,k})}^2 \leq L^2\cE^t_K\,.
  \end{align*}
\end{proof}

Henceforth, we will call $\cE^t_K$ as the error, or as the client-drift following \cite{karimireddy2019scaffold}.
  
\paragraph{Bounding error in MimeLite.} 
Now we will try bound the client drift $\cE^t$ for MimeLite.
\begin{lemma}[MimeLite error]\label{lem:reduction-mimelite-error}
  Suppose that all functions $f_i(\,\cdot\,, \zeta)$ are $L$-smooth \eqref{asm:smoothness}, $\sigma^2$ variance \eqref{asm:noise}, and \eqref{asm:heterogeneity} is satisfied, and the updater $\cU$ has $B$-Lipschitz updates. Then using step-size $\tilde\eta\leq \frac{1}{2BL}$,
  \[
    \frac{1}{18 B^2 \tilde\eta^2} \cE^K \leq \E\norm{\nabla f(\xx)}^2 + G^2 + \frac{\sigma^2}{2K}\,.
  \]
\end{lemma}
\begin{proof}
  For $K=1$, we have $\E\norm{\yy_{i,1} - \xx}^2 \leq B^2\eta^2(G^2 + \sigma^2) + B^2\eta^2\E\norm{\nabla f(\xx)}^2$. The lemma is easily shown to be true. Assuming $K \geq 2$ henceforth, and starting from the client update of MimeLite we have
  \begin{align*}
      \E\norm{\yy_{i,k} - \xx}^2 &= \E\norm{\yy_{i,k-1} - \eta \cU(\nabla f_i(\yy_{i, k-1}^t; \zeta_{i,k}^t); \ss^{t-1}) - \xx}^2\\
      &\leq \E\norm{\yy_{i,k-1} - \eta \cU(\nabla f_i(\yy_{i, k-1}^t; \ss^{t-1}) - \xx}^2 + B^2\eta^2\sigma^2\\
      &\leq \rbr*{1+\frac{1}{K-1}}\E\norm{\yy_{i,k-1} - \xx}^2 + K\eta^2\E\norm{\cU(\nabla f_i(\yy_{i, k-1}^t; \ss^{t-1})}^2 + B^2\eta^2\sigma^2\\
      &\leq \rbr*{1+\frac{1}{K-1}}\E\norm{\yy_{i,k-1} - \xx}^2 + KB^2 \eta^2\E\norm{\nabla f_i(\yy_{i, k-1}) \pm \nabla f_i(\xx)}^2 + B^2\eta^2\sigma^2\\
      &\leq \rbr*{1+\frac{1}{K-1}}\E\norm{\yy_{i,k-1} - \xx}^2 \\&\hspace{2cm}+ 2KB^2 \eta^2\E\norm{\nabla f_i(\xx)}^2+ 2KB^2L^2\eta^2\E\norm{\yy_{i,k-1} - \xx}^2 + B^2\eta^2\sigma^2\\
      &\leq \rbr*{1+\frac{2}{K-1}}\E\norm{\yy_{i,k-1} - \xx}^2 + 2KB^2 \eta^2\E\norm{\nabla f(\xx)}^2+ 2KB^2\eta^2G^2 + B^2\eta^2\sigma^2\,.
  \end{align*}
  Here, we used the condition on our step size that $\tilde\eta = K\eta \leq \frac{1}{2LB}$, which implies that $2KB^2L^2\eta^2 \leq \frac{1}{K-1}$. Unrolling this recursion, we have
  \[
    \E\norm{\yy_{i,k} - \xx}^2 \leq \rbr*{2KB^2 \eta^2\E\norm{\nabla f(\xx)}^2+ 2KB^2\eta^2G^2 + B^2\eta^2\sigma^2}\sum_{k=1}^K\rbr*{1+\frac{2}{K-1}}^k\,.
  \]
  Note that $\rbr*{1+\frac{2}{K-1}}^k \leq 9$. Averaging then over $k$ and $i$, we get
  \[
    \cE^t_K \leq 18 K^2 B^2 \eta^2\E\norm{\nabla f(\xx)}^2+ 18 K^2 B^2\eta^2G^2 + 9 K B^2\eta^2\sigma^2\,.
  \]
  Finally, recalling that $\tilde\eta = K \eta$ finishes the lemma.
\end{proof}

\paragraph{Bounding error in Mime.} 
Next we will try bound the client drift $\cE^t$ for Mime. The additional SVRG correction term used in Mime improves the bound on the error.
\begin{lemma}[Mime Error]\label{lem:reduction-error-mime}
  Suppose that all functions $f_i(\,\cdot\,, \zeta)$ are $L$-smooth \eqref{asm:smoothness}, $\sigma^2$ variance \eqref{asm:noise}, and \eqref{asm:heterogeneity} is satisfied, and the updater $\cU$ has $B$-Lipschitz updates. Then using step-size $\tilde\eta\leq \frac{1}{2BL}$,
  \[
     \cE^K \leq 18 B^2 \tilde\eta^2\E\norm*{\frac{1}{S}\sum_{i}\nabla_i f(\xx)}^2 \,.
  \]
\end{lemma}
\begin{proof}
  For $K=1$, the Mime update loos like
  \begin{align*}
      \E\norm{\yy_{i,1} - \xx}^2 &= \eta^2\E\norm{ \cU\rbr*{ \cc ; \ss^{t-1}}}^2\\
      &\leq \eta^2B^2\E\norm{\cc}^2\,.
  \end{align*} 
  Assuming $K \geq 2$ henceforth, and starting from the client update of Mime we have
  \begin{align*}
      \E\norm{\yy_{i,k} - \xx}^2 &= \E\norm{\yy_{i,k-1} - \eta \cU(\nabla f_i(\yy_{i, k-1}; \zeta_{i,k}^t) - \nabla f_i(\xx; \zeta_{i,k}^t) + \cc^t; \ss^{t-1}) - \xx}^2\\
      &\leq \rbr*{1 + \frac{1}{K - 1}} \E\norm{\yy_{i,k-1} - \xx}^2 
      \\&\hspace{2cm}+ K\eta^2\E\norm{\cU(\nabla f_i(\yy_{i, k-1}; \zeta_{i,k}^t) - \nabla f_i(\xx; \zeta_{i,k}^t) + \cc^t; \ss^{t-1})}^2\\
      &\leq \rbr*{1 + \frac{1}{K - 1}} \E\norm{\yy_{i,k-1} - \xx}^2 + K\eta^2B^2\E\norm{\nabla f_i(\yy_{i, k-1}; \zeta_{i,k}^t) - \nabla f_i(\xx; \zeta_{i,k}^t) + \cc^t}^2\\
      &\leq \rbr*{1 + \frac{1}{K - 1}} \E\norm{\yy_{i,k-1} - \xx}^2 
      \\&\hspace{2cm}+ 2K\eta^2B^2\E\norm{\nabla f_i(\yy_{i, k-1}; \zeta_{i,k}^t) - \nabla f_i(\xx; \zeta_{i,k}^t)}^2 +  2K\eta^2B^2\E\norm{\cc^t}^2\\
      &\leq \rbr*{1 + \frac{1}{K - 1} + 2K\eta^2B^2 L^2} \E\norm{\yy_{i,k-1} - \xx}^2 + 2K\eta^2B^2\E\norm{\cc^t}^2\\
      &\leq \rbr*{1 + \frac{2}{K - 1}} \E\norm{\yy_{i,k-1} - \xx}^2 + 2K\eta^2B^2\E\norm{\cc^t}^2\,.
  \end{align*}
  Here, we used the condition on our step size that $\tilde\eta = K\eta \leq \frac{1}{2LB}$, which implies that $2KB^2L^2\eta^2 \leq \frac{1}{K-1}$. Unrolling this recursion, we have
  \[
    \E\norm{\yy_{i,k} - \xx}^2 \leq 2KB^2 \eta^2\E\norm{\cc^t}^2\sum_{k=1}^K\rbr*{1+\frac{2}{K-1}}^k \leq 18 K^2 B^2 \eta^2\E\norm{\cc^t}^2\,.
  \]
  Note that $\rbr*{1+\frac{2}{K-1}}^k \leq 9$. Averaging then over $k$ and $i$, recalling that $\tilde\eta = K \eta$ get
  \[
    \cE^t_K \leq 18 B^2 \tilde\eta^2\E\norm{\cc^t}^2\,.
  \]
\end{proof}

\paragraph{Putting it together (Theorem~\ref{thm:reduction}).}
\begin{lemma}\label{lem:reduction-final}
  The updates of Mime and MimeLite for $\cc^t$ satisfying $\E[\cc^t] = \nabla f(\xx^{t-1}) \text{ and } \E\norm{\cc^t - \nabla f(\xx^{t-1})}^2 \leq \frac{G^2}{S}$, we have for $\tilde\eta \leq \frac{1}{2BL}$
  \begin{align*}
    \xx^t &= \xx^{t-1} - \tilde\eta\cU(\cc^t + \ee^t; \ss^{t-1})\\
    \ss^t &= \cV(\cc^t; \ss^{t-1})\,.
\end{align*}
Where, we have
\[
\tfrac{1}{18 B^2L^2\tilde\eta^2}\E_t\norm{\text{\badbox{$\ee_t$}}}^2 \leq 
    \begin{cases}
        \E\norm{\cc_t}^2 & \text{\mimebox{\mime}\,,}\\
        \E\norm{\nabla f(\xx^t)}^2 + G^2 + \frac{\sigma^2}{2 K} &\text{\litebox{\mimelite}.}
    \end{cases}
\]
\end{lemma}
\begin{proof}
Now, combining Lemmas~\ref{lem:reduction-server},~\ref{lem:reduction-error}, shows that running Mime or MimeLite is equivalent to 
\begin{align*}
    \xx^t &= \xx^{t-1} - \tilde\eta\cU(\cc^t + \ee^t; \ss^{t-1})\\
    \ss^t &= \cV(\cc^t; \ss^{t-1})\,,
\end{align*}
where
\[
    \E[\cc^t] = \nabla f(\xx^{t-1}) \text{ and } \E\norm{\cc^t - \nabla f(\xx^{t-1})}^2 \leq \frac{G^2}{S}\,.
\]
This shows the first part of the theorem. For the second part of the theorem, using the bound from Lemma~\ref{lem:reduction-error-mime} for Mime,
\[
\E\norm{\ee^t} \leq L^2\cE^t_K \leq 18L^2B^2\tilde\eta^2\E\norm{\cc^t}^2\,.
\]
For MimeLite, we will instead use the bound from Lemma~\ref{lem:reduction-mimelite-error},
\[
\E\norm{\ee^t} \leq L^2\cE^t_K \leq 18L^2B^2\tilde\eta^2\E\norm{\nabla f(\xx^t)}^2 + 18L^2B^2\tilde\eta^2G^2 + \frac{9 L^2B^2\tilde\eta^2 \sigma^2}{K}\,.
\]
\end{proof}
Note that the Lemma we proved here is slightly stronger than the theorem in the main section (up to constants which were suppressed).

\subsection{Convergence of MimeSGD and MimeLiteSGD (Corollary~\ref{cor:sgd})}
Theorem~\ref{thm:reduction} shows that Mime and MimeLite mimic a centralized algorithm quite closely up to error $\cO(\tilde\eta^2)$. Then, analyzing the sensitivity of the base algorithm to such perturbation yields specific rates of convergence. We perform such an analysis using SGD as our base optimizer.

Properties of SGD as the base optimizer:
\begin{itemize}[nosep]
    \item $\ss^t$ is empty i.e. there are no global statistics used.
    \item $\cU(\gg; \ss^{t-1}) = \gg$ for any $\gg$ and $B=1$.
\end{itemize}

With this in mind, we proceed.
\begin{lemma}[Progress in one round]\label{lem:reduction-sgd-mime}
    Given that $f$ is $L$-smooth, and for any step-size $\tilde\eta \leq \frac{1}{2(B +2)L}$ for $B \geq 1$ we have
    \[
        f(\xx^t) \leq f(\xx^{t-1}) - \frac{\tilde\eta}{4}\E\norm{\nabla f(\xx^{t-1})}^2 + \tilde\eta\E\norm{\ee^t}^2 + \frac{L \tilde\eta^2 G^2}{S}\,.
    \]
\end{lemma}
\begin{proof}
  Starting from the update equation and the smoothness of $f$, we have
  \begin{align*}
      \E f(\xx^t) &\leq \E f(\xx^{t-1}) + \E \inp{\nabla f(\xx^{t-1})}{\xx^t - \xx^{t-1}} + \frac{L}{2}\E\norm{\xx^t - \xx^{t-1}}^2\\
      &= \E f(\xx^{t-1}) -\tilde\eta \E \norm{\nabla f(\xx^{t-1})}^2 + \tilde\eta \inp{\nabla f(\xx^{t-1})}{\ee^t} + \frac{L\tilde\eta^2}{2}\E\norm{\cc^t + \ee^t}^2 \\
      &\leq \E f(\xx^{t-1}) -\frac{\tilde\eta}{2} \E \norm{\nabla f(\xx^{t-1})}^2 + \frac{\tilde\eta}{2} \norm{\ee^t}^2 + \frac{2L\tilde\eta^2}{2}\E\norm{\cc^t}^2 + \frac{2L\tilde\eta^2}{2}\E\norm{\ee^t}^2\\
      &\leq \E f(\xx^{t-1}) - \rbr*{\frac{\tilde\eta}{2} - \frac{2L\tilde\eta^2}{2}} \E \norm{\nabla f(\xx^{t-1})}^2 + \rbr*{L\tilde\eta^2 + \frac{\tilde\eta}{2}} \E\norm{\ee^t}^2 + \frac{2L\tilde\eta^2 G^2}{2S}\,.
  \end{align*}
  Using the bound on the step size that $\tilde\eta \leq \frac{1}{4L}$ yields the lemma.
\end{proof}

\paragraph{One round progress for MimeSGD.}
Next, we specialize the convergence rate for Mime. 
\begin{lemma}\label{lem:reduction-mimesgd-final}
  Suppose $f$ is a $L$-smooth function satisfying PL-inequality for $\mu \geq 0$ ($\mu =0$ corresponds to the general case). Running MimeSGD for $\tilde\eta \leq \frac{1}{12 B L}$ satisfies
  \[
         \frac{\tilde\eta}{16}\E\norm{\nabla f(\xx^{t-1})}^2 \leq (1- \tfrac{\mu \tilde\eta}{8})(f(\xx^{t-1}) - f^\star) - (f(\xx^{t}) - f^\star) + \frac{3L\tilde\eta^2 G^2}{S}\,.
  \]
\end{lemma}
\begin{proof}
Recall from Lemma~\ref{lem:reduction-final} that for Mime,
\[
    \E\norm{\ee^t}^2 \leq 18 L^2 B^2 \tilde\eta^2 \E\norm{\cc^t}^2 \leq 18 L^2 B^2 \tilde\eta^2 \E\norm{\nabla f(\xx^{t-1}) }^2 + \frac{18 L^2 B^2 \tilde\eta^2 G^2}{S}\,.
\]
Combining this with Lemma~\ref{lem:reduction-sgd-mime} yields the following progress for Mime
\begin{align*}
    f(\xx^t) &\leq f(\xx^{t-1}) - \rbr*{\frac{\tilde\eta}{4} - 18L^2B^2\tilde\eta^3}\E\norm{\nabla f(\xx^{t-1})}^2  + \frac{(L \tilde\eta^2 + 18L^2B^2\tilde\eta^3)G^2}{S}\\
    &\leq f(\xx^{t-1}) - \frac{\tilde\eta}{8}\E\norm{\nabla f(\xx^{t-1})}^2 + \frac{3L \tilde\eta^2 G^2}{S}\,.
\end{align*}
Here, we used the bound on the step size that $\tilde\eta \leq \frac{1}{12 L B}$ implies $18L^2B^2\tilde\eta^2 \leq \frac{1}{8}$. Now using PL-inequality, we can write
\begin{align*}
    f(\xx^t) - f^\star &\leq f(\xx^{t-1}) - f^\star - \frac{\mu\tilde\eta}{8}(f(\xx^{t-1})- f^\star) - \frac{\tilde\eta}{16}\E\norm{\nabla f(\xx^{t-1})}^2 + \frac{3L \tilde\eta^2 G^2}{S}\,.
\end{align*}
This yields the lemma.
\end{proof}

We are now ready to derive the convergence rate.
\paragraph{Convergence rate of MimeSGD on general non-convex functions.} Set $\mu=0$ in Lemma~\ref{lem:reduction-mimesgd-final} and sum over $t$
\begin{align*}
    \frac{1}{T}\sum_{t=1}^T\E\norm{\nabla f(\xx^{t-1})}^2 &\leq \frac{16(f(\xx^0) - f^\star)}{\tilde\eta T} + \frac{48 L \tilde\eta G^2}{S}\\
    &\leq 16 \sqrt{\frac{3LG^2(f(\xx^0) - f^\star)}{ST}} + \frac{192 BL (f(\xx^0) - f^\star)}{T}\,.
\end{align*}
The final step used a step-size of $\tilde\eta = \min\rbr*{\frac{1}{12BL}, \frac{1}{4L}, \sqrt{\frac{S(f(\xx^0) - f^\star)}{3L TG^2}}}$. Here, we used $\xx^{\text{out}} = \xx^\tau$ where $\tau$ is uniformly at random chosen in $[T]$.

\paragraph{Convergence rate of MimeSGD on PL-inequality.} Multiply Lemma~\ref{lem:reduction-mimesgd-final} by $(1 - \tfrac{\mu\tilde\eta}{8})^{T-t}$ and sum over $t$
\begin{align*}
    \sum_{t=1}^T (1 - \tfrac{\mu\tilde\eta}{8})^{T-t}\E\norm{\nabla f(\xx^{t-1})}^2 &\leq \sum_{t=1}^T (1 - \tfrac{\mu\tilde\eta}{8})^{T-(t-1)}\frac{16(f(\xx^{t-1}) - f^\star)}{\tilde\eta}
    \\&\hspace{1cm} - (1 - \tfrac{\mu\tilde\eta}{8})^{T-t}\frac{16(f(\xx^{t}) - f^\star)}{\tilde\eta} + (1 - \tfrac{\mu\tilde\eta}{8})^{T-t}\frac{48 L \tilde\eta G^2}{S}\\
    &\leq  (1 - \tfrac{\mu\tilde\eta}{8})^{T}\frac{16(f(\xx^{0}) - f^\star)}{\tilde\eta} + \sum_{t=1}^T(1 - \tfrac{\mu\tilde\eta}{8})^{T-t}\frac{48 L \tilde\eta G^2}{S}\,.
\end{align*}
Output $\xx^{\text{out}} = \xx^\tau$ where $\tau$ is chosen with probability proportional to $(1 - \tfrac{\mu\tilde\eta}{8})^{T-t}$. Then, this yields
\[
    \E\norm{\nabla f(\xx^{\text{out}})}^2 \leq (1 - \tfrac{\mu\tilde\eta}{8})^{T}\frac{16(f(\xx^{0}) - f^\star)}{\tilde\eta} +  \frac{48 L \tilde\eta G^2}{S} \leq \tilde\cO\rbr*{\frac{\sigma^2}{\mu T} + L(f(\xx^0) - f^\star)\exp\rbr*{- \frac{\mu T}{12 BL}}}\,.
\]
Using an appropriate step-size $\tilde\eta$ yields the final rate (see Lemma~1 of \cite{karimireddy2019scaffold}).


\paragraph{One round progress for MimeLiteSGD.}
Next, we specialize the convergence rate for MimeLite. 
\begin{lemma}\label{lem:reduction-mimelitesgd-final}
  Suppose $f$ is a $L$-smooth function satisfying PL-inequality for $\mu \geq 0$ ($\mu =0$ corresponds to the general case). Running MimeLiteSGD for $\tilde\eta \leq \frac{1}{12 B L}$ satisfies
  \[
         \frac{\tilde\eta}{16}\E\norm{\nabla f(\xx^{t-1})}^2 \leq (1- \tfrac{\mu \tilde\eta}{8})(f(\xx^{t-1}) - f^\star) - (f(\xx^{t}) - f^\star) + \frac{L\tilde\eta^2 G^2}{S} + 18L^2B^2\tilde\eta^3\rbr*{G^2 + \sigma^2/K}\,.
  \]
\end{lemma}
\begin{proof}
Recall from Lemma~\ref{lem:reduction-final} that,
\[
    \E\norm{\ee^t}^2 \leq 18 L^2 B^2 \tilde\eta^2 \E\norm{\cc^t}^2 \leq 18 L^2 B^2 \tilde\eta^2 \E\norm{\nabla f(\xx^{t-1}) }^2 + 18 L^2 B^2 \tilde\eta^2 G^2 + \frac{9 L^2 B^2 \tilde\eta^2 \sigma^2}{K} \,.
\]
Combining this with Lemma~\ref{lem:reduction-sgd-mime} yields the following progress for Mime
\begin{align*}
    f(\xx^t) &\leq f(\xx^{t-1}) - \rbr*{\frac{\tilde\eta}{4} - 18L^2B^2\tilde\eta^3}\E\norm{\nabla f(\xx^{t-1})}^2  + \frac{L \tilde\eta^2 G^2}{S}  + 18L^2B^2\tilde\eta^3\rbr*{G^2 + \sigma^2/K}\\
    &\leq f(\xx^{t-1}) - \frac{\tilde\eta}{8}\E\norm{\nabla f(\xx^{t-1})}^2 + \frac{L \tilde\eta^2 G^2}{S}  + 18L^2B^2\tilde\eta^3\rbr*{G^2 + \sigma^2/K}\,.
\end{align*}
Here, we used the bound on the step size that $\tilde\eta \leq \frac{1}{12 L B}$ implies $18L^2B^2\tilde\eta^2 \leq \frac{1}{8}$. Now using PL-inequality, we can write
\begin{align*}
    f(\xx^t) - f^\star - &(f(\xx^{t-1}) - f^\star) \leq\\  & - \frac{\mu\tilde\eta}{8}(f(\xx^{t-1})- f^\star) - \frac{\tilde\eta}{16}\E\norm{\nabla f(\xx^{t-1})}^2 + \frac{L \tilde\eta^2 G^2}{S} +  18L^2B^2\tilde\eta^3\rbr*{G^2 + \sigma^2/K}\,.
\end{align*}
This yields the lemma.
\end{proof}

We are now ready to derive the convergence rate.
\paragraph{Convergence rate of MimeLiteSGD on general non-convex functions.} Define $\tilde G^2 = G^2 + \sigma^2/K$. Set $\mu=0$ in Lemma~\ref{lem:reduction-mimelitesgd-final} and sum over $t$
\begin{align*}
    \frac{1}{T}\sum_{t=1}^T\E\norm{\nabla f(\xx^{t-1})}^2 &\leq \frac{16(f(\xx^0) - f^\star)}{\tilde\eta T} + \frac{16 L \tilde\eta G^2}{S} + 288L^2B^2\tilde\eta^2\tilde G^2\\
    &\hspace{-1cm}\leq 16 \sqrt{\frac{LG^2(f(\xx^0) - f^\star)}{ST}} + 84\rbr*{\frac{L\tilde G (f(\xx^0)- f^\star)}{T}}^{2/3} + \frac{192 BL (f(\xx^0) - f^\star)}{T}\,.
\end{align*}
The final step used an appropriate step-size of $\tilde\eta$, see Lemma 2 of \cite{karimireddy2019scaffold}. Here, we used $\xx^{\text{out}} = \xx^\tau$ where $\tau$ is uniformly at random chosen in $[T]$. Finally note that if $K \geq \frac{\sigma^2}{G^2}$, then $\tilde G^2 \leq 2G^2$.

\paragraph{Convergence rate of MimeLiteSGD on PL-inequality.} Multiply Lemma~\ref{lem:reduction-mimelitesgd-final} by $(1 - \tfrac{\mu\tilde\eta}{8})^{T-t}$ and sum over $t$
\begin{align*}
    \sum_{t=1}^T (1 - \tfrac{\mu\tilde\eta}{8})^{T-t}\E\norm{\nabla f(\xx^{t-1})}^2 &\leq \sum_{t=1}^T (1 - \tfrac{\mu\tilde\eta}{8})^{T-(t-1)}\frac{16(f(\xx^{t-1}) - f^\star)}{\tilde\eta}
        \\&\hspace{2cm} - (1 - \tfrac{\mu\tilde\eta}{8})^{T-t}\frac{16(f(\xx^{t}) - f^\star)}{\tilde\eta}
        \\&\hspace{2cm} + \sum_{t=1}^T (1 - \tfrac{\mu\tilde\eta}{8})^{T-t}\rbr*{\frac{16 L \tilde\eta G^2}{S} +288L^2B^2\tilde\eta^2\tilde G^2}\\
    &\leq  (1 - \tfrac{\mu\tilde\eta}{8})^{T}\frac{16(f(\xx^{0}) - f^\star)}{\tilde\eta} 
        \\&\hspace{2cm} + \sum_{t=1}^T(1 - \tfrac{\mu\tilde\eta}{8})^{T-t}\rbr*{\frac{16 L \tilde\eta G^2}{S} + 288L^2B^2\tilde\eta^2\tilde G^2}\,.
\end{align*}
Output $\xx^{\text{out}} = \xx^\tau$ where $\tau$ is chosen with probability proportional to $(1 - \tfrac{\mu\tilde\eta}{8})^{T-t}$. Then, this yields with appropriate step-size $\tilde\eta$ yields the final rate (see Lemma~1 of \cite{karimireddy2019scaffold}).
\[
    \E\norm{\nabla f(\xx^{\text{out}})}^2 \leq \tilde\cO\rbr*{\frac{\sigma^2}{\mu T} + \frac{L^2 \tilde G^2}{\mu^2 T^2}+ L(f(\xx^0) - f^\star)\exp\rbr*{- \frac{\mu T}{12 BL}}}\,.
\]

\subsection{Convergence of MimeAdam and MimeLiteAdam (Corollary~\ref{cor:adam})}
We will largely follow the convergence analysis of \cite{zaheer2018adaptive} for the analysis of Adam. A crucial difference between their setting and ours is that in our algorithm we use the global statistics (second order moment) corresponding to $t-1$ i.e. $\sqrt{\vv^{t-1}}$ instead of $\sqrt{\vv^t}$ where the $\sqrt{\cdot}$ operator is applied element wise. Practically, this does not make a significant difference since the discount (momentum) factor for the second momentum is very large. Theoretically however, this difference simplifies our proof significantly removing otherwise hard to handle stochastic dependencies.

In this section, we will use Adam as our base optimizer with $\varepsilon_0 > 0$ parameter for stability and $\beta_1 =0$ (i.e. RMSProp). This is identical to the setting in the centralized algorithm analyzed by \cite{zaheer2018adaptive}. The properties of our base optimizer are then:
\begin{itemize}[nosep]
    \item $\ss^t = \vv^t$ which is a running average estimate of the second moment and satisfies $\vv^{t} > 0$.
    \item $\cU(\gg; \vv^{t-1}) = \frac{\gg}{\sqrt{\vv^{t-1}} + \varepsilon_0}$ for any $\gg$. This update for any $\vv^{t-1}$ is $B$-Lipschitz for $B = \frac{1}{\varepsilon_0}$.
\end{itemize}

In this sub-section, all operations on vectors (multiplication, division, addition, comparison) are applied element-wise with appropriate broad-casting.
\paragraph{One round progress of Adam.}
\begin{lemma}[Effective step-sizes]\label{lem:reduction-adam-step}
    Suppose that $|\nabla_j f_i(\xx)| \leq H$. Then Adam has effective step-sizes
    \[
        \frac{1}{H+\varepsilon_0}\gg  \leq \cU(\gg; \vv^{t-1}) \leq \frac{1}{\varepsilon_0}\gg\,.
    \]
\end{lemma}
\begin{proof}
  Recall that $\vv^t = \beta_2\vv^{t-1} + (1-\beta_2)(\cc^t)^2$ starting from $\vv^0 = 0$. Thus for any $t\geq 0$, we have $\vv^t \geq 0$ and hence $\sqrt{\vv^{t-1}} + \varepsilon_0 \geq \varepsilon_0$. For the other side, recall that $\vv^t$ is updated with centralized stochastic gradients $\cc^t = \frac{1}{S}\sum_{i}\nabla f_i(\xx) $. 
  \[
    [\cc^t]_j = \frac{1}{S}\sum_{i}\sbr*{\nabla f_i(\xx)}_j \leq H\,.
  \]
  Further,
  \[
    [\vv^t]_j = \beta_2[\vv^{t-1}]_j + (1-\beta_2)[\cc^t]_j^2 \leq \beta_2[\vv^{t-1}]_j + (1-\beta_2)H^2 \leq H^2\,. 
  \]
  Hence $\sqrt{\vv^{t-1} + \varepsilon_0} \leq H + \varepsilon_0$.
\end{proof}

\begin{lemma}[One round progress]\label{lem:reduction-adam-one-round}
  For one round of Adam with error $\ee^t$ in the update $\cU$ and using $\cc^t$ for update $\cV$, we have
  \[
    \E f(\xx^t) \leq \E f(\xx^{t-1}) - \frac{\tilde\eta}{4(H + \varepsilon_0)}\norm{\nabla f(\xx^{t-1})}^2 + \frac{\tilde\eta\rbr*{(H+\varepsilon_0) + \varepsilon_0/(H + \varepsilon_0)}}{2 \varepsilon_0^2}\E \norm{\ee^t}^2 +  \frac{L \tilde\eta^2G^2}{S \varepsilon_0^2}\,.
  \]
\end{lemma}
\begin{proof}
  Starting from Lemma~\ref{lem:reduction-adam-step} and the smoothness of $f$, we have
  \begin{align*}
      \E f(\xx^{t}) &\leq \E f(\xx^{t-1}) - \tilde\eta \E\inp{\nabla f(\xx^{t-1})}{\E_t[\cU\rbr*{\cc^t + \ee^t}]} + \frac{L \tilde\eta^2}{2}\E\norm{\cU\rbr*{\cc^t + \ee^t; \vv^{t-1}}}^2\\
      &\leq  \E f(\xx^{t-1}) - \tilde\eta \E\inp{\nabla f(\xx^{t-1})}{\E_t\sbr*{\frac{\cc^t + \ee^t}{\sqrt{\vv^{t-1}} + \varepsilon_0}}} + \frac{L \tilde\eta^2}{2}\E\norm{\cU\rbr*{\cc^t + \ee^t; \vv^{t-1}}}^2\\
      &\leq  \E f(\xx^{t-1}) - \tilde\eta \E\inp{\nabla f(\xx^{t-1})}{\sbr*{\frac{\nabla f(\xx^{t-1}) + \ee^t}{\sqrt{\vv^{t-1}} + \varepsilon_0}}} + \frac{L \tilde\eta^2}{2 \varepsilon_0^2}\E\norm{\cc^t + \ee^t}^2\\
      &\leq  \E f(\xx^{t-1}) - \frac{\tilde\eta}{H + \varepsilon_0} \norm{\nabla f(\xx^{t-1})}^2 - \tilde\eta\E \inp{\nabla f(\xx^{t-1})}{\frac{\ee^t}{\sqrt{\vv^{t-1}} + \varepsilon_0}}+ \frac{L \tilde\eta^2}{2 \varepsilon_0^2}\E\norm{\cc^t + \ee^t}^2\\
      &\leq  \E f(\xx^{t-1}) - \frac{\tilde\eta}{2(H + \varepsilon_0)} \norm{\nabla f(\xx^{t-1})}^2 + \frac{\tilde\eta(H+\varepsilon_0)}{2}\E \norm{\frac{\ee^t}{\sqrt{\vv^{t-1}} + \varepsilon_0}}^2+ \frac{L \tilde\eta^2}{2 \varepsilon_0^2}\E\norm{\cc^t + \ee^t}^2\\
      &\leq  \E f(\xx^{t-1}) - \rbr*{\frac{\tilde\eta}{2(H + \varepsilon_0)} - \frac{L \tilde\eta^2}{ \varepsilon_0^2}}\norm{\nabla f(\xx^{t-1})}^2 + \frac{\tilde\eta(H+\varepsilon_0) + 2L \tilde\eta^2}{2 \varepsilon_0^2}\E \norm{\ee^t}^2 +  \frac{L \tilde\eta^2G^2}{ S\varepsilon_0^2}\\
      &\leq  \E f(\xx^{t-1}) - \frac{\tilde\eta}{4(H + \varepsilon_0)}\norm{\nabla f(\xx^{t-1})}^2 + \frac{\tilde\eta\rbr*{(H+\varepsilon_0) + \varepsilon_0/(H + \varepsilon_0)}}{2 \varepsilon_0^2}\E \norm{\ee^t}^2 +  \frac{L \tilde\eta^2G^2}{S \varepsilon_0^2}
  \end{align*}
  Here we used our bound on the step-size that $\tilde\eta \leq \frac{\varepsilon_0}{4L (H + \varepsilon_0)}$.
\end{proof}

\paragraph{Convergence of MimeAdam.}
\begin{lemma}
  Suppose that assumptions~\ref{asm:heterogeneity}--\eqref{asm:noise} hold and further $|\nabla_j f_i(\xx)| \leq H$. Then, running MimeAdam with step-size $\tilde\eta \leq \frac{\varepsilon_0^2}{12 L (H+\varepsilon_0)}$, we have
  \[
        \frac{1}{T}\sum_{t=1}^T \E\norm{\nabla f(\xx^{t-1})}^2 \leq \frac{96L(H+\varepsilon_0)^2 (f(\xx_0) - f^\star)}{\varepsilon_0^2 T} + \frac{2G^2}{S}\,.
  \]
\end{lemma}
Combining Lemma~\ref{lem:reduction-adam-one-round} with the bound on $\ee^t$ from Lemma~\ref{lem:reduction-final} we get,
\begin{align*}
    \E f(\xx^t) &\leq \E f(\xx^{t-1}) - \frac{\tilde\eta}{4(H + \varepsilon_0)}\norm{\nabla f(\xx^{t-1})}^2 + \frac{\tilde\eta\rbr*{(H+\varepsilon_0) + \varepsilon_0/(H + \varepsilon_0)}}{2 \varepsilon_0^2}\E \norm{\ee^t}^2 +  \frac{L \tilde\eta^2G^2}{S \varepsilon_0^2}\\
    &\leq \E f(\xx^{t-1}) - \frac{\tilde\eta}{4(H + \varepsilon_0)}\norm{\nabla f(\xx^{t-1})}^2 + \frac{9 L^2\tilde\eta^3 \rbr*{(H+\varepsilon_0) + \varepsilon_0/(H + \varepsilon_0)}}{ \varepsilon_0^4}\E \norm{\cc^t}^2 \\&\hspace{2cm} +  \frac{L \tilde\eta^2G^2}{S \varepsilon_0^2}\\
    &\leq \E f(\xx^{t-1}) - \rbr*{\frac{\tilde\eta}{4(H + \varepsilon_0)} - \frac{9 L^2\tilde\eta^3 \rbr*{(H+\varepsilon_0) + \varepsilon_0/(H + \varepsilon_0)}}{ \varepsilon_0^4}} \norm{\nabla f(\xx^{t-1})}^2 \\&\hspace{2cm} +  \frac{L \tilde\eta^2G^2}{S \varepsilon_0^2} + \frac{9 L^2\tilde\eta^3 \rbr*{(H+\varepsilon_0) + \varepsilon_0/(H + \varepsilon_0)} G^2}{ S\varepsilon_0^4}\\
    &\leq \E f(\xx^{t-1}) - \rbr*{\frac{\tilde\eta}{4(H + \varepsilon_0)} - \frac{18 L^2\tilde\eta^3(H+\varepsilon_0) }{ \varepsilon_0^4}} \norm{\nabla f(\xx^{t-1})}^2 \\&\hspace{2cm} +  \frac{L \tilde\eta^2G^2}{S \varepsilon_0^2} + \frac{18 L^2\tilde\eta^3 (H+\varepsilon_0) G^2}{ S\varepsilon_0^4}\\
    &\leq \E f(\xx^{t-1}) - \frac{\tilde\eta}{8(H + \varepsilon_0)}\norm{\nabla f(\xx^{t-1})}^2 + \frac{ \tilde\eta G^2}{4 S (H+ \varepsilon_0)}\,.
\end{align*}
To simplify computations, here we assumed we assumed $(H + \varepsilon_0)^2 \geq \varepsilon_0$ without loss of generality. If this is not true, we can replace $H$ by $\max(H, \sqrt{\varepsilon_0} - \varepsilon_0)$. Assuming $\tilde\eta \leq \frac{\varepsilon_0^2}{12 L (H + \varepsilon_0)}$, we have $\frac{18 L^2\tilde\eta^2(H+\varepsilon_0) }{ \varepsilon_0^4} \leq \frac{1}{8(H+\varepsilon_0)}$. Rearranging the terms and substituting the bounds on the step-size yields the lemma.

\paragraph{Convergence of MimeLiteAdam.}
\begin{lemma}
  Suppose that assumptions~\ref{asm:heterogeneity}--\eqref{asm:noise} hold and further $|\nabla_j f_i(\xx)| \leq H$. Then, running MimeLiteAdam with step-size $\tilde\eta \leq \frac{\varepsilon_0^2}{12 L \sqrt{S}(H+\varepsilon_0)}$, we have for $\tilde G^2:= G^2 + \sigma^2/K$,
  \[
        \frac{1}{T}\sum_{t=1}^T \E\norm{\nabla f(\xx^{t-1})}^2 \leq \frac{96L\sqrt{S}(H+\varepsilon_0)^2 (f(\xx_0) - f^\star)}{\varepsilon_0^2 T} + \frac{2\tilde G^2}{S}\,.
  \]
\end{lemma}
Combining Lemma~\ref{lem:reduction-adam-one-round} with the bound on $\ee^t$ from Lemma~\ref{lem:reduction-final} we get for $\tilde G^2:= G^2 + \sigma^2/K$,
\begin{align*}
    \E f(\xx^t) &\leq \E f(\xx^{t-1}) - \frac{\tilde\eta}{4(H + \varepsilon_0)}\norm{\nabla f(\xx^{t-1})}^2 + \frac{\tilde\eta(H+\varepsilon_0)}{\varepsilon_0^2}\E \norm{\ee^t}^2 +  \frac{L \tilde\eta^2G^2}{S \varepsilon_0^2}\\
    &\leq \E f(\xx^{t-1}) - \frac{\tilde\eta}{4(H + \varepsilon_0)}\norm{\nabla f(\xx^{t-1})}^2 
        \\&\hspace{2cm} + \frac{18L^2\tilde\eta^3(H+\varepsilon_0)}{\varepsilon_0^4}\E \norm{\nabla f(\xx^{t-1})}^2 + \frac{18L^2\tilde\eta^3(H+\varepsilon_0) (\tilde G^2)}{\varepsilon_0^4} +  \frac{L \tilde\eta^2G^2}{S \varepsilon_0^2}\\
    &\leq \E f(\xx^{t-1}) - \frac{\tilde\eta}{8(H + \varepsilon_0)}\norm{\nabla f(\xx^{t-1})}^2 + \frac{\tilde\eta \tilde G^2}{4 S (H + \varepsilon_0)}
\end{align*}
Again as before to simplify computations, here we assumed $(H + \varepsilon_0)^2 \geq \varepsilon_0$ without loss of generality. If this is not true, we can replace $H$ by $\max(H, \sqrt{\varepsilon_0} - \varepsilon_0)$. Assuming $\tilde\eta \leq \frac{\varepsilon_0^2}{12 L (H + \varepsilon_0)\sqrt{S}}$, we have $\frac{18 L^2\tilde\eta^2(H+\varepsilon_0) }{ \varepsilon_0^4} \leq \frac{1}{8S(H+\varepsilon_0)}$. Rearranging the terms and substituting the bounds on the step-size yields the lemma.


\section{Circumventing server-only lower bounds}\label{sec:improvement-analysis}

In this section we see how to use momentum based variance reduction \cite{cutkosky2019momentum,tran2019hybrid} to reduce the variance of the updates and improve convergence. It should be noted that MVR does not exactly fit the \mime framework \eqref{eqn:base-alg} since it requires computing gradients at two points on the same batch. However, it is straightforward to extend the idea of \mime to MVR as we will now do. We use MVR as a theoretical justification for why the usual momentum works well in practice. An interesting future direction would be to adapt the algorithm and analysis of \cite{cutkosky2020momentum}, which does fit the framework of \mime.

For the sake of convenience, we summarize the notation used in the proof in a table.
\begin{table}[H]
  \caption{Summary of all notation used in the MVR proofs}
  \centering
  \begin{tabular}{@{}cl@{}} 
  \toprule
   $\sigma^2$, $G^2$, and $\delta$ & intra-client gradient, inter-client gradient, and inter-client Hessian variance \\
   $\eta$, $a$ & step-size, $(1-\beta)$ momentum parameters\\
   $T$, $t$ & total number, index of communication rounds\\ 
   $K$, $k$ & total number, index of client local update steps\\ 
   $\cS^t$, $S$, and $i$ & sampled set, size, and index of clients in round $t$\\ 
   $\xx^t$ & aggregated server model \emph{after} round $t$\\
   $\mm^t$ & server momentum computed \emph{after} round $t$\\
   $\cc^t$ & control variate of server \emph{after} round $t$ (only \mime)\\ 
   $\yy^t_{i,k}$ & model parameters of $i$th client in round $t$ \emph{after} step $k$\\ 
   $\zeta^t_{i,k}$ & mini-batch data used by $i$th client in round $t$ and step $k$\\ 
   $\dd^t_{i,k}$ & parameter update by $i$th client in round $t$, step $k$\\ 
   $\ee^t$ & error in momentum $\mm^t - \nabla f(\xx^{t-1})$\\ 
   $\Delta^t_{i,k}$, $\Delta^{t-1}$ & $\E\norm{\yy^t_{i,k} - \xx^{t-2}}^2$, $\E\norm{\xx^{t-1} - \xx^{t-2}}^2 = \Delta_{i,0}^t$\\ 
  \bottomrule
  \end{tabular}
  \end{table}

\subsection{Algorithm descriptions}

Now, we formally describe the \mime MVR and \mimelite MVR algorithms. In each round $t$, we sample clients $\cS^{t}$ such that $\abs{\cS^t} = S$. The server communicates the server parameters $\xx^{t-1}$, the past parameters $\xx^{t-2}$, and the momentum $\mm^{t-1}$ term.
\mime additionally uses a control variate $\cc^{t-1}$ as we describe next.
\paragraph{Control variate in Mime.}
\mime uses an additional control variate $\cc^{t-1}$ to reduce the variance.
\begin{equation}\label{eqn:def-mimemvr-control}
  \cc^{t-1} = \frac{1}{S}\sum_{i\in \cS^t} \nabla f_i(\xx^{t-2})\,.
\end{equation} 
Note that both $\cc^{t-1}$ and $\mm^{t-1}$ use gradients and parameters from previous rounds (different from the previous section). A naive implementation of this method requires two steps of communication per round to implement this algorithm. Alternatively, we can reserve some clients in the previous round for computing $\cc^{t-1}$ which can then be used in the current round, removing the need for two steps of communication. In particular, it can be computed on a different set of an independent sampled clients $\tilde\cS^{t-1}$. In fact, all our theoretical results hold even if we use \emph{a single client} to perform the local updates and the rest of clients are used only to compute $\cc^{t-1}$ each round.

\paragraph{Local client updates.}
Then each client $i \in \cS^t$  makes a copy $\yy_{i, 0}^t = \xx^{t-1}$ and perform $K$ local client updates. In each local client update $k \in [K]$, the client samples a dataset $\zeta_{i,k}^t$. \mime performs the following update:
\begin{equation}\label{eqn:def-mimemvr-update}
  \begin{split}
    \yy_{i, k}^t &= \yy_{i, k-1}^t - \eta \dd_{i,k}^t \,, \text{ where }\\
    \dd_{i,k}^t &= a(\nabla f_i(\yy_{i, k-1}^t; \zeta_{i,k}^t) - \nabla f_i(\xx^{t-1}; \zeta_{i,k}^t) + \cc^{t-1}) + (1-a)\mm^{t-1} \\&\hspace{2cm}+ (1-a)(\nabla f_i(\yy_{i, k-1}^t; \zeta_{i,k}^t) - \nabla f_i(\xx^{t-1}; \zeta_{i,k}^t))    \,.
  \end{split}
\end{equation}

\mimelite on the other hand uses a very similar but simpler update scheme which does not rely on $\cc^{t-1}$:
\begin{equation}\label{eqn:def-mimelitemvr-update}
  \begin{split}
    \yy_{i, k}^t &= \yy_{i, k-1}^t - \eta \dd_{i,k}^t \,, \text{ where }\\
    \dd_{i,k}^t &= a\nabla f_i(\yy_{i, k-1}^t; \zeta_{i,k}^t) + (1-a)\mm^{t-1} \\&\hspace{2cm}+ (1-a)(\nabla f_i(\yy_{i, k-1}^t; \zeta_{i,k}^t) - \nabla f_i(\xx^{t-1}; \zeta_{i,k}^t))    \,.
  \end{split}
\end{equation}

\paragraph{Server updates.}
After $K$ such local updates, the server then aggregates the new client parameters as
\begin{equation}\label{eqn:def-mimemvr-aggregation}
  \xx^t = \frac{1}{S}\sum_{j\in\cS^t}\yy_{j,K}^t\,.
\end{equation}

The momentum term is updated at the end of the round for $a \geq 0$ as
\begin{equation}\label{eqn:def-mimemvr-momentum}
  \mm^{t} = \underbrace{a (\tfrac{1}{S}\tsum_{j\in \cS^t} \nabla f_j(\xx^{t-1}) ) + (1-a)\mm^{t-1}}_{\text{SGDm}} + \underbrace{(1-a) (\tfrac{1}{S}\tsum_{j\in \cS^t} \nabla f_j(\xx^{t-1}) -\nabla f_j(\xx^{t-2}) )}_{\text{correction}}\,.
\end{equation}
As we can see, the momentum update of MVR can be broken down into the usual SGDm update, and a correction. Intuitively, this correction term is very small since $f_i$ is smooth and $\xx^{t-1} \approx \xx^{t-2}$. Another way of looking at the update \eqref{eqn:def-mimemvr-momentum} is to note that if all functions are identical i.e. $f_j = f_k$ for any $j,k$, then \eqref{eqn:def-mimemvr-momentum} just becomes the usual gradient descent. Thus MimeMVR tries to maintain an exponential moving average of only the variance terms, reducing its bias. We refer to \cite{cutkosky2019momentum} for more detailed explanation of MVR.

\subsection{Bias in updates}

The main difference in MimeMVR from the centralized versions of \cite{tran2019hybrid,cutkosky2019momentum} is the additional local steps which are biased. In particular, for $k \geq 1$ the expected gradient $\E[\nabla f_i(\yy_{i,k}^t)] \neq \nabla f(\yy_{i,k}^t)$ because $\yy_{i,k}^t$ also depends on the sample $i$. This bias is in fact the underlying cause of client drift and controlling it is a crucial step for our analysis.
\begin{lemma}[Mime bias]\label{lem:mimemvr-bias-control}
  For any values of $\xx$ and $\yy_i$ where $\yy_i$ may depend on $i$, the following holds for any client $i$ almost surely given that \eqref{asm:heterogeneity} and \eqref{asm:hessian-similarity} hold:
  \[
    \E_{\cS, \zeta}\norm*{\nabla f_i(\yy_i; \zeta) + \frac{1}{\abs{\cS}}\sum_{j\in\cS}\nabla f_j(\xx) - \nabla f_i(\xx; \zeta) \quad - \quad \nabla f(\yy_i)}^2 \leq 2\delta^2\E_{\cS}\norm{\yy_i - \xx}^2 + \frac{2G^2}{S}\,.
  \]
\end{lemma}
\begin{proof}
  We can separate the noise from the rest of the terms and expand as
  \begin{align*}
      \E_{\zeta, \cS}\norm*{\nabla f_i(\yy_i; \zeta) + \frac{1}{\abs{\cS}}\sum_{j\in\cS}\nabla f_j(\xx) - \nabla f_i(\xx; \zeta)  - \nabla f(\yy_i)}^2 \hspace{-8cm}&\\
      &\leq 2\E_{\cS}\norm*{\nabla f_i(\yy_i; \zeta) +\nabla f(\xx)  - \nabla f_i(\xx; \zeta)  - \nabla f(\yy_i)}^2 +
        2\E_{\cS}\norm*{ \frac{1}{\abs{\cS}}\sum_{j\in\cS}\nabla f_j(\xx) - \nabla f(\xx)}^2 \\
      &\leq 2\E_{\cS}\norm*{\nabla f_i(\yy_i; \zeta) +\nabla f(\xx)  - \nabla f_i(\xx; \zeta)  - \nabla f(\yy_i)}^2 +
        \frac{2G^2}{S} \\
      &\leq 2\E_{\cS}\delta^2\norm{\yy_i - \xx}^2 + \frac{2G^2}{S}\,.
  \end{align*}
  The first inequality used Young's inequality, the second used \eqref{asm:heterogeneity}, and the last used \eqref{asm:hessian-similarity} in the form of Lemma~\ref{lem:similarity}.
\end{proof}

We can perform a similar analysis of the bias of local updates encountered by \mimelite.
\begin{lemma}[MimeLite bias]\label{lem:mimelitemvr-bias-control}
  For any values of $\xx$ and $\yy_i$ where $\yy_i$ may depend on $i$, the following holds for any client $i$ randomly chosen from $\cD$ given that \eqref{asm:heterogeneity}, \eqref{asm:hessian-similarity} and \eqref{asm:noise} hold:
  \[
    \E_{i, \zeta}\norm*{\nabla f_i(\yy_i; \zeta) \quad - \quad \nabla f(\yy_i)}^2 \leq 2\delta^2\E_{i}\norm{\yy_i - \xx}^2 + 2G^2 + \sigma^2 \,.
  \]
\end{lemma}
\begin{proof}
  We can separate the noise from the rest of the terms and expand as
  \begin{align*}
      \E_{\zeta, i}\norm*{\nabla f_i(\yy_i; \zeta) - \nabla f(\yy_i)}^2 &= \E_{\zeta, i}\norm*{\nabla f_i(\yy_i; \zeta) \pm \nabla f_i(\xx) \pm \nabla f(\xx) - \nabla f(\yy_i)}^2 \\
      &\leq \E_{i}\norm*{\nabla f_i(\yy_i) \pm \nabla f_i(\xx) \pm \nabla f(\xx) - \nabla f(\yy_i)}^2 + \sigma^2\\
      &\leq 2\E_{i}\norm*{\nabla f_i(\yy_i) +\nabla f(\xx)  - \nabla f_i(\xx)  - \nabla f(\yy_i)}^2 
      \\&\hspace{2cm} +
        2\E_{i}\norm*{\nabla f_i(\xx) - \nabla f(\xx)}^2 +    \sigma^2\\
      &\leq 2\E_{i}\norm*{\nabla f_i(\yy_i) +\nabla f(\xx)  - \nabla f_i(\xx)  - \nabla f(\yy_i)}^2 +
        2G^2 +    \sigma^2\\
      &\leq 2\delta^2\E_{i}\norm{\yy_i - \xx}^2 + 2G^2 + \sigma^2 \,.
  \end{align*}
  The first inequality used \eqref{asm:noise}, the second used Young's inequality, the third used \eqref{asm:heterogeneity}, and the last used \eqref{asm:hessian-similarity} in the form of Lemma~\ref{lem:similarity}.
\end{proof}
Note that the bias for MimeLite is very similar to that of Mime, except that Mime has dependence of $\frac{G^2}{S}$, whereas MimeLite has $G^2 + \sigma^2$. Hence, the rate of convergence of MimeLite will depend on $G^2$ wheras Mime will have the optimal dependency of $G^2/S$. Hence, in the rest of the proof, we will consider \textbf{only Mime} and simply replace $G^2/S$ with $(G^2 + \sigma^2)$ to obtain the corresponding results for MimeLite.

\subsection{Change in each client update}

\paragraph{Client update variance.} Now we examine the variance of our update in each local step $\dd_{i,k}^t$.

\begin{lemma}\label{lem:mimemvr-update-bound}
  For the client update \eqref{eqn:def-mimemvr-update}, given \eqref{asm:heterogeneity} and \eqref{asm:hessian-similarity}, the following holds for any $a \in [0,1]$  where $\ee^{t} := \mm^{t} - \nabla f(\xx^{t-1})$ and $\Delta_{i,k}^t := \expect\norm{\yy_{i,k}^t - \xx^{t-2}}^2 $:
  \[
    \expect\norm{\dd_{i,k}^t - \nabla f(\yy_{i,k-1}^{t})}^2 \leq 3 \expect\norm{\ee^{t-1}}^2 + 3 \delta^2 \Delta_{i, k-1}^t + \frac{3a^2 G^2}{S}\,.
  \]
\end{lemma}
\begin{proof}
  Starting from the client update \eqref{eqn:def-mimemvr-update}, we can rewrite it as
  \begin{align*}
    \dd_{i,k}^t - \nabla f(\yy_{i,k-1}^{t}) &= (1-a)\ee^{t-1}  \\&\hspace{.5cm} + \rbr*{\nabla f_i(\yy_{i,k-1}^{t}; \zeta_{i,k}^{t}) - \nabla f_i(\xx^{t-2}; \zeta_{i,k}^{t})) - \nabla f(\yy_{i,k-1}^{t}) + \nabla f(\xx^{t-2})} \\&\hspace{.5cm} + a\rbr*{\frac{1}{S}\sum_{j \in \cS^t} \nabla f_j(\xx^{t-2}) - \nabla f(\xx^{t-2})}\,.
  \end{align*}
  We can use the relaxed triangle inequality Lemma~\ref{lem:norm-sum} to claim
  \begin{align*}
    \expect\norm{\dd_{i,k}^t - \nabla f(\yy_{i,k-1}^{t})}^2 \\
    &\hspace{-2cm} = 3(1-a)^2 \expect\norm{\ee^{t-1}}^2  \\&\hspace{-1.5cm} + 3(1-a)^2\norm*{(\nabla f_i(\yy_{i,k-1}^{t}; \zeta_{i,k}^{t}) - \nabla f_i(\xx^{t-2}; \zeta_{i,k}^{t})) - (\nabla f(\yy_{i,k-1}^{t}) - \nabla f(\xx^{t-2}))}^2 \\&\hspace{-1.5cm} + 3a^2 \norm*{\frac{1}{S}\sum_{j \in \cS^t} \nabla f_j(\xx^{t-2}) - \nabla f(\xx^{t-2})}^2\\
    &\hspace{-2cm}\leq 3 \expect\norm{\ee^{t-1}}^2 + 3\delta^2\norm{\yy_{i,k-1}^t - \xx^{t-2}}^2 + \frac{3a^2 G^2}{S}\,.
  \end{align*}
  The last inequality used the Hessian similarity Lemma~\ref{lem:similarity} to bound the second term and the heterogeneity bound \eqref{asm:heterogeneity} to bound the last term. Also, $(1-a)^2 \leq 1$ since $a \in [0,1]$.
\end{proof}

\paragraph{Distance moved in each step.} We show that the distance moved by a client in each step during the client update can be controlled.
\begin{lemma}\label{lem:mimemvr-distance-bound}
  For MimeMVR updates \eqref{eqn:def-mimemvr-update} with $\eta \leq \frac{1}{6 K \delta}$ and given \eqref{asm:heterogeneity} and \eqref{asm:hessian-similarity}, the following holds
  \[
    \Delta_{i,k}^t \leq \rbr*{1 + \frac{1}{K}}\Delta_{i,k-1}^t + 18\eta^2 K a^2 \frac{G^2}{S} + 18 \eta^2 K \expect\norm{\ee^{t-1}}^2 + 6\eta^2 K \norm{\nabla f(\yy_{i,k-1}^t)}^2\,,
  \]
  where we define $\Delta_{i,k}^t := \expect\norm{\yy_{i,k}^t - \xx^{t-2}}^2 $.
\end{lemma}
\begin{proof}
  Starting from the MimeMVR update \eqref{eqn:def-mimemvr-update} and the relaxed triangle inequality with $c = 2K$,
  \begin{align*}
    \expect\norm{\yy_{i,k}^t - \xx^{t-2}}^2 &= \expect\norm{\yy_{i,k-1}^t -\eta \dd_{i,k}^t - \xx^{t-2}}^2\\
    &\leq \rbr*{1 + \frac{1}{2K}}\expect\norm{\yy_{i,k-1}^t - \xx^{t-2}}^2 + (2K + 1)\eta^2 \expect\norm{\dd_{i,k}^t}^2\\
    &\leq \rbr*{1 + \frac{1}{2K}}\expect\norm{\yy_{i,k-1}^t - \xx^{t-2}}^2 + 6K\eta^2 \expect\norm{\dd_{i,k}^t - \nabla f(\yy_{i,k-1}^{t})}^2 \\&\hspace{2cm}+ 6K\eta^2\expect\norm{\nabla f(\yy_{i,k-1}^{t})}^2\\
    &\leq \rbr*{1 + \frac{1}{2K} + 18K\eta^2 \delta^2}\expect\norm{\yy_{i,k-1}^t - \xx^{t-2}}^2 \\&\hspace{2cm}+ 18K\eta^2  \expect\norm{\ee^{t-1}}^2 +  \frac{18K\eta^2 a^2 G^2}{S} +6K\eta^2\expect\norm{\nabla f(\yy_{i,k-1}^{t})}^2 \,.
  \end{align*}
  The last inequality used the update variance bound Lemma~\ref{lem:mimemvr-update-bound}. We can simplify the expression further since $\eta \leq \frac{1}{6 K \delta}$ implies $18K\eta^2 \delta^2 \leq \frac{1}{2K}$.
\end{proof}

\paragraph{Progress in one step.} Now we can compute the progress made in each step.
\begin{lemma}\label{lem:mimemvr-step-progress}
  For any client update step with step size $\eta \leq \min\rbr*{\frac{1}{L}, \frac{1}{192 \delta K}}$ and given that \eqref{asm:heterogeneity}, \eqref{asm:hessian-similarity} hold, we have
  \begin{align*}
    \E f(\yy_{i,k}^t) + \delta\rbr*{1 + \frac{2}{K}}^{K - k}\Delta_{i,k}^t 
    &\leq \E f(\yy_{i,k-1}^t) + \delta\rbr*{1 + \frac{2}{K}}^{K - (k-1)}\Delta_{i,k-1}^t\\
    &\hspace{2cm}  - \frac{\eta}{4}\expect\norm{\nabla f(\yy_{i,k-1}^t)}^2 + 3\eta\E\norm{\ee^{t-1} }^2  + \frac{3\eta a^2 G^2}{S}\,.
  \end{align*}
\end{lemma} 
\begin{proof}
  The assumption that $f$ is $L$-smooth implies a quadratic upper bound \eqref{eqn:quad-upper}.
  \begin{align*}
    f(\yy_{i,k}^t) - f(\yy_{i,k-1}^t) &\leq -\eta\inp{\nabla f(\yy_{i,k-1}^t)}{\dd_{i,k}^t} + \frac{L\eta^2}{2}\norm{\dd_{i,k}^t}^2\\
    &=-\frac{\eta}{2}\norm{\nabla f(\yy_{i,k-1}^t)}^2 + \frac{L\eta^2 - \eta}{2}\norm{\dd_{i,k}^t}^2 + \frac{\eta}{2}\norm{\dd_{i,k}^t - \nabla f(\yy_{i,k-1}^t)}^2\,.
  \end{align*}
  The second equality used the fact that for any $a,b$, $-2ab = (a-b)^2 - a^2 - b^2$. The second term can be removed since $\eta \leq \frac{1}{L}$. Taking expectation on both sides and using the update variance bound Lemma~\ref{lem:mimemvr-update-bound},
  \begin{align*}
    \expect f(\yy_{i,k}^t) - \expect f(\yy_{i,k-1}^t) &\leq - \frac{\eta}{2}\expect\norm{\nabla f(\yy_{i,k-1}^t)}^2 + \frac{3 \eta a^2 G^2}{2 S} \\&\hspace{2cm} + \frac{3\eta}{2} \expect\norm{\ee^{t-1}}^2 + \frac{3\eta\delta^2}{2} \Delta^t_{i,k-1}\\
    &\leq - \frac{\eta}{2}\expect\norm{\nabla f(\yy_{i,k-1}^t)}^2 + \frac{3 \eta a^2 G^2}{2 S} \\&\hspace{2cm} + \frac{3\eta}{2} \expect\norm{\ee^{t-1}}^2 + \frac{3\eta\delta^2}{2} \Delta^t_{i,k-1}
  \end{align*}


  Multiplying the distance bound Lemma~\ref{lem:mimemvr-distance-bound} by $\delta\rbr*{1 + \frac{2}{K}}^{K - k}$. Note that for any $K\geq 1$ and $k \in [K]$, we have $1 \leq \rbr*{1 + \frac{2}{K}}^{K - k} \leq 8$. Then we get
  \begin{align*}
      \delta\rbr*{1 + \frac{2}{K}}^{K - k}\Delta_{i,k}^t &\leq \delta\rbr*{1 + \frac{2}{K}}^{K - k}\Bigg(\rbr[\Big]{1 + \frac{1}{K}}\Delta_{i,k-1}^t + 18\eta^2 K a^2 \frac{G^2}{S} 
        \\&\hspace{3cm}+ 18 \eta^2 K \expect\norm{\ee^{t-1}}^2 + 6\eta^2 K \norm{\nabla f(\yy_{i,k-1}^t)}^2 \Bigg)\\
      &\leq \delta\rbr*{1 + \frac{2}{K}}^{K - (k-1)}\Delta_{i,k-1}^t -\frac{\delta}{K}\rbr*{1 + \frac{2}{K}}^{K - k}\Delta_{i,k-1}^t  \\
        &\hspace{1cm} + 48 \eta^2 \delta K\E\norm{\nabla f(\yy_{i,k-1}^t)}^2+\frac{144 \eta^2 \delta K a^2 G^2}{S} + 144 \eta^2 \delta K \E\norm{\ee^{t-1}}^2\\
    &\leq \delta\rbr*{1 + \frac{2}{K}}^{K - (k-1)}\Delta_{i,k-1}^t -\frac{\delta}{K}\Delta_{i,k-1}^t + 48 \eta^2 \delta K\E\norm{\nabla f(\yy_{i,k-1}^t)}^2 \\
        &\hspace{1cm}+\frac{144 \eta^2 \delta K a^2 G^2}{S} + 144 \eta^2 \delta K \E\norm{\ee^{t-1}}^2\,.
  \end{align*}
Adding these two inequalities together yields
\begin{align*}
    \E f(\yy_{i,k}^t) + \delta\rbr*{1 + \frac{2}{K}}^{K - k}\Delta_{i,k}^t &\leq \E f(\yy_{i,k-1}^t) + \delta\rbr*{1 + \frac{2}{K}}^{K - (k-1)}\Delta_{i,k-1}^t\\
    &\hspace{1cm} -\rbr*{\frac{\eta}{2} - 48\eta^2\delta K}\expect\norm{\nabla f(\yy_{i,k-1}^t)}^2\\
    &\hspace{1cm} +\rbr*{\frac{3\eta}{2} + 144\eta^2\delta K}\expect\norm{\ee^{t-1}}^2 \\
    &\hspace{1cm} +\rbr*{\frac{3\eta}{2}+ + 144\eta^2\delta K}\frac{a^2G^2}{S}\,.
\end{align*}
    Using our bound on the step-size that $\eta \leq \frac{1}{192 \delta K}$ implies that $\eta \delta K \leq \frac{1}{48*4}$.
\end{proof}

\subsection{Change in each round}
We now see how the quantities we defined change across rounds.

\paragraph{Distance moved in a round.}
\begin{lemma}\label{lem:mimemvr-distance-round-bound}
  For MimeMVR updates \eqref{eqn:def-mimemvr-update} with $\eta \leq \frac{1}{6 K \delta}$ and given \eqref{asm:heterogeneity} and \eqref{asm:hessian-similarity}, the following holds
  \[
    \Delta^t \leq 54K^2\eta^2  \expect\norm{\ee^{t-1}}^2 +  \frac{54 K^2 \eta^2 a^2 G^2}{S} + \frac{1}{KS} \sum_{i,k}18K^2\eta^2\expect\norm{\nabla f(\yy_{i,k-1}^{t})}^2\,,
  \]
  where we define $\Delta^t := \expect\norm{\xx^t - \xx^{t-1}}^2$.
\end{lemma}
\begin{proof}
  Starting from the MimeMVR update \eqref{eqn:def-mimemvr-update} and following the proof of Lemma~\ref{lem:mimemvr-distance-bound},
  \begin{align*}
    \expect\norm{\yy_{i,k}^t - \xx^{t-1}}^2 &= \expect\norm{\yy_{i,k-1}^t -\eta \dd_{i,k}^t - \xx^{t-1}}^2\\
    &\leq \rbr*{1 + \frac{1}{2K}}\expect\norm{\yy_{i,k-1}^t - \xx^{t-1}}^2 + (2K + 1)\eta^2 \expect\norm{\dd_{i,k}^t}^2\\
    &\leq \rbr*{1 + \frac{1}{2K}}\expect\norm{\yy_{i,k-1}^t - \xx^{t-1}}^2 + 6K\eta^2 \expect\norm{\dd_{i,k}^t - \nabla f(\yy_{i,k-1}^{t})}^2 \\&\hspace{2cm}+ 6K\eta^2\expect\norm{\nabla f(\yy_{i,k-1}^{t})}^2\\
    &\leq \rbr*{1 + \frac{1}{K}}\expect\norm{\yy_{i,k-1}^t - \xx^{t-1}}^2 \\&\hspace{2cm}+ 18K\eta^2  \expect\norm{\ee^{t-1}}^2 +  \frac{18K\eta^2 a^2 G^2}{S} +6K\eta^2\expect\norm{\nabla f(\yy_{i,k-1}^{t})}^2 \,.
  \end{align*}
  Note that $\xx^t = \frac{1}{S}\sum_{i \in \cS}\yy_{i,K}^t$ and so,
  \begin{align*}
      \expect\norm{\xx^t - \xx^{t-1}}^2& \\
      &\hspace{-2cm}\leq \frac{1}{S}\sum_{i \in \cS}\expect\norm{\yy_{i,K}^t - \xx^{t-1}}^2\\
      &\hspace{-2cm}\leq \frac{1}{S}\sum_{i \in \cS}\sum_{k}\rbr*{18K\eta^2  \expect\norm{\ee^{t-1}}^2 +  \frac{18K\eta^2 a^2 G^2}{S} +6K\eta^2\expect\norm{\nabla f(\yy_{i,k-1}^{t})}^2 }\rbr*{1 + \frac{1}{K}}^{K-k}\\
      &\hspace{-2cm}\leq 54K^2\eta^2  \expect\norm{\ee^{t-1}}^2 +  \frac{54 K^2 \eta^2 a^2 G^2}{S} + \frac{1}{KS} \sum_{i,k}18K^2\eta^2\expect\norm{\nabla f(\yy_{i,k-1}^{t})}^2\,.
  \end{align*}
  Here we used the inequality that for all $k$, $\rbr*{1 + \frac{1}{K}}^{K-k} \leq 3$.
\end{proof}

\paragraph{Server momentum variance.} We compute the error of the server momentum $\mm^{t-1}$ defined as $\ee^{t} = \mm^{t} - \nabla f(\xx^{t-1})$. Its expected norm can be bounded as follows.
\begin{lemma}\label{lem:mimemvr-momentum-bound}
  For the momentum update \eqref{eqn:def-mimemvr-momentum}, given \eqref{asm:heterogeneity} and \eqref{asm:hessian-similarity}, the following holds for any $\eta \leq \frac{1}{51 \delta K}$ and $1 \geq a \geq 2592 K^2 \delta^2 \eta^2$,
  \[
    \expect\norm{\ee^{t}}^2 \leq (1-\tfrac{23a}{24})\expect\norm{\ee^{t-1}}^2 + \frac{3a^2 G^2}{S} + \frac{1}{KS} \sum_{i,k}36 K^2 \delta^2 \eta^2\expect\norm{\nabla f(\yy_{i,k-1}^{t})}^2\,.
  \]
\end{lemma}
\begin{proof}
  Starting from the momentum update \eqref{eqn:def-mimemvr-momentum},
  \begin{align*}
    \ee^{t} &= (1-a)\ee^{t-1}  \\&\hspace{.5cm} + (1-a)\rbr*{\frac{1}{S}\sum_{j \in \cS^t}( \nabla f_j(\xx^{t-1}) - \nabla f_j(\xx^{t-2})) - \nabla f(\xx^{t-1}) + \nabla f(\xx^{t-2})} \\&\hspace{.5cm} + a\rbr*{\frac{1}{S}\sum_{j \in \cS^t}( \nabla f_j(\xx^{t-1}) - \nabla f(\xx^{t-1})}\,.
  \end{align*}
  Now, the term $\ee^{t-1}$ does not have any information from round $t$ and hence is statistically independent of the rest of the terms. Further, the rest of the terms have mean 0. Hence, we can separate out the zero mean noise terms from the $\ee^{t-1}$ following Lemma~\ref{lem:independent} and then the relaxed triangle inequality Lemma~\ref{lem:norm-sum} to claim
  \begin{align*}
    \expect\norm{\ee^{t}}^2 &\leq (1-a)^2 \expect\norm{\ee^{t-1}}^2  \\&\hspace{.5cm} + 2(1-a)^2\norm*{\frac{1}{S}\sum_{j \in \cS^t}( \nabla f_j(\xx^{t-1}) - \nabla f_j(\xx^{t-2})) - \nabla f(\xx^{t-1}) + \nabla f(\xx^{t-2})}^2 \\&\hspace{.5cm} + 2a^2 \norm*{\frac{1}{S}\sum_{j \in \cS^t}( \nabla f_j(\xx^{t-1}) - \nabla f(\xx^{t-1})}^2\\
    &\leq (1-a)^2 \expect\norm{\ee^{t-1}}^2 + 2(1-a)^2\delta^2\norm{\xx^{t-1} - \xx^{t-2}}^2 + \frac{2a^2 G^2}{S}\,.
  \end{align*}
  The inequality used the Hessian similarity Lemma~\ref{lem:similarity} to bound the second term and the heterogeneity bound \eqref{asm:heterogeneity} to bound the last term. Finally, note that $(1-a)^2 \leq (1-a)\leq 1$ for $a \in [0,1]$.
  We can continue by bounding $\Delta^{t-1}$ using Lemma~\ref{lem:mimemvr-distance-round-bound}.
  \begin{align*}
    \expect\norm{\ee^{t}}^2 &\leq (1-a) \expect\norm{\ee^{t-1}}^2 + 2\delta^2\Delta^{t-1} + \frac{2a^2 G^2}{S}\\
    &\leq (1-a)\expect\norm{\ee^{t-1}}^2 + \frac{2a^2 G^2}{S} \\&\hspace{0.5cm}+ 108 K^2\delta^2\eta^2  \expect\norm{\ee^{t-1}}^2 +  \frac{108 K^2\delta^2 \eta^2 a^2 G^2}{S} + \frac{1}{KS} \sum_{i,k}36 K^2 \delta^2 \eta^2\expect\norm{\nabla f(\yy_{i,k-1}^{t})}^2\\
    &\leq (1-\tfrac{23a}{24})\expect\norm{\ee^{t-1}}^2 + \frac{3a^2 G^2}{S} + \frac{1}{KS} \sum_{i,k}36 K^2 \delta^2 \eta^2\expect\norm{\nabla f(\yy_{i,k-1}^{t})}^2\,.
  \end{align*}
  The last step used our bound on the momentum parameter that $1 \geq a \geq 2592\eta^2\delta^2K^2$. Note that $\eta \leq \frac{1}{51\delta K}$ ensures that this set is non-empty.
\end{proof}

\paragraph{Progress in one round.} Finally, we can compute the progress made in a round. Note that we need a technical condition that $f$ is $\delta$-weakly convex. However, this is only needed because we insist on running the algorithm on $S$ clients in parallel and then averaging their weights---the averaging requires weak convexity to ensure that the loss doesn't blow up. It has been experimentally observed in \cite{mcmahan2017communication} that with the right initialization, averaging of the parameters does not increase the loss value and so weak convexity within this region might be vaalid. Finally note that if we instead simply run the local updates on a single chosen client with all the rest only being used to compute $\cc^{t-1}$, we will retain all convergence rates without needing weak-convexity.

\begin{lemma}\label{lem:mimemvr-round-progress}
  For any round of MimeMVR with step size $\eta \leq \min\rbr*{\frac{1}{L}, \frac{1}{864 \delta K}}$ and momentum parameter $a \geq 912 \eta^2 \delta^2 K^2$. Then, given that \eqref{asm:heterogeneity}--\eqref{asm:hessian-similarity} hold and $f$ is $\delta$-weakly convex, we have
  \begin{align*}
    \frac{\eta}{24 KS} \sum_{k \in [K], j\in\cS^t}\expect \norm{ \nabla  f(\yy_{i,k-1}^t)}^2 \leq \Phi^{t-1} - \Phi^{t} + \frac{17 \eta a \delta^2 K^2 G^2}{S}\,,
  \end{align*}
where we define the sequence 
\[ \Phi^t := \tfrac{1}{K}\expect [f(\xx^{t}) - f^\star]+ \frac{96\eta}{23 a}\expect\norm{\ee^{t}}^2 + \frac{8 \delta }{K}\Delta^t\,.
\]
\end{lemma} 
\begin{proof}
  We start by summing over the progress in single client updates as in Lemma~\ref{lem:mimemvr-step-progress}
  \begin{align*}
    \sum_{k\in[K]}\frac{\eta}{4}\expect\norm{\nabla f(\yy_{i,0}^t)}^2 &\leq \E f(\yy_{i,0}^t) + \delta\rbr*{1 + \frac{2}{K}}^{K}\Delta_{i,0}^t \\
    &\hspace{1cm} - \E f(\yy_{i,K}^t) - \delta\Delta_{i,K}^t \\
    &\hspace{1cm}   + 3\eta K\E\norm{\ee^{t-1} }^2  + \frac{3\eta K a^2 G^2}{S}\\
    &\leq \E f(\yy_{i,0}^t) + 8\delta \Delta_{i,0}^t - \E f(\yy_{i,K}^t) - \delta\Delta_{i,K}^t\\
    &\hspace{1cm}   + 3\eta K\E\norm{\ee^{t-1} }^2  + \frac{3\eta K a^2 G^2}{S}\\
    &\leq \E f(\xx^{t-1}) + 8\delta\Delta^{t-1} - \E f(\yy_{i,K}^t) - \delta\Delta_{i,K}^t\\
    &\hspace{1cm}   + 3\eta K\E\norm{\ee^{t-1} }^2  + \frac{3\eta K a^2 G^2}{S}\,.
  \end{align*}

  Recall that $\Delta_{i,k}^t = \expect\norm{\yy_{i,k}^t - \xx^{t-2}}^2$ and $\yy^t_{i,0} = \xx^{t-1}$. This gives the last step above, making $\Delta_{i,0}^t = \Delta^{t-1}$.
  Then by the averaging Lemma~\ref{lem:averaging}, we have
  \begin{align*}
    \frac{1}{S}\sum_{j \in \cS^t}\expect [f(\yy_{j,K}^t)] + \delta \Delta_{j,K}^t &= \frac{1}{S}\sum_{j \in \cS}\expect [f(\yy_{j,K}^t)] + \delta\expect\norm{\xx^{t-2} - \yy_{j,K}^t}^2 \\
    &\geq \expect [f(\xx^t)] + \delta\expect\norm{\xx^{t-2} - \xx^t}^2\,.
  \end{align*}
  So by averaging our inequality over the sampled clients, and diving our summation over the updates by $K$, we get

  \begin{align*}
    \frac{\eta}{4 KS} \sum_{k \in [K], j\in\cS^t} \expect \norm{\nabla f(\yy_{i,k-1}^t)}^2&\\
    &\hspace{-2cm}\leq \tfrac{1}{K}\expect [f(\xx^{t-1})]+ 3\eta \expect\norm{\ee^{t-1}}^2 + \frac{8 \delta}{K}\Delta^{t-1}  - \tfrac{1}{K}\expect [f(\xx^{t})] + \frac{3 \eta a^2 G^2}{S}\,.
  \end{align*}
  We can use the bound on $\Delta_t$ from Lemma~\ref{lem:mimemvr-distance-round-bound} to proceed as
  \begin{align*}
    \frac{\eta}{4 KS} \sum_{k \in [K], j\in\cS^t} &\expect \norm{\nabla f(\yy_{i,k-1}^t)}^2 \\
    &\leq \tfrac{1}{K}\expect [f(\xx^{t-1})] - \tfrac{1}{K}\expect [f(\xx^{t})]  + 3\eta \expect\norm{\ee^{t-1}}^2  + \frac{3 \eta a^2 G^2}{S}\\
    &\hspace{0.5cm} + \frac{8 \delta}{K}\Delta^{t-1} - \frac{8 \delta}{K}\Delta^{t}\\
  &\hspace{0.5cm}+ 432 K \delta \eta^2  \expect\norm{\ee^{t-1}}^2 +  \frac{432 K \delta \eta^2 a^2 G^2}{S} + \frac{1}{KS} \sum_{i,k}144 K \delta \eta^2\expect\norm{\nabla f(\yy_{i,k-1}^{t})}^2\\
  &\leq \tfrac{1}{K}\expect [f(\xx^{t-1})] - \tfrac{1}{K}\expect [f(\xx^{t})]  + 4\eta \expect\norm{\ee^{t-1}}^2  + \frac{4 \eta a^2 G^2}{S}\\
    &\hspace{0.5cm} + \frac{8 \delta}{K}\Delta^{t-1} - \frac{8 \delta}{K}\Delta^{t} + \frac{\eta}{6 KS} \sum_{i,k}\expect\norm{\nabla f(\yy_{i,k-1}^{t})}^2
  \end{align*}
  The last step used the bound on the step size that $\eta \leq \frac{1}{864 \delta K}$. Now, multiplying the error bound Lemma~\ref{lem:mimemvr-momentum-bound} by $\frac{96\eta}{23a}$ gives
  \[
    \frac{96\eta}{23a}\expect\norm{\ee^{t}}^2 \leq \frac{4*24\eta}{23a}(1-\tfrac{23a}{24})\expect\norm{\ee^{t-1}}^2 + \frac{13 \eta a G^2}{S} + \frac{1}{KS} \sum_{i,k}\frac{38 K^2 \delta^2 \eta^3}{a}\expect\norm{\nabla f(\yy_{i,k-1}^{t})}^2\,.
  \]
  Adding this to the previously obtained bound yields
  \begin{align*}
    \frac{\eta}{4 KS} \sum_{k \in [K], j\in\cS^t} \expect \norm{\nabla f(\yy_{i,k-1}^t)}^2 &\leq \rbr*{\frac{1}{6} + \frac{38K^2 
    \delta^2 \eta^2}{a}}\frac{\eta}{KS} \sum_{k \in [K], j\in\cS^t}\expect \norm{\nabla f(\yy_{i,k-1}^t)}^2
    \\&\hspace{1cm} + \tfrac{1}{K}\expect [f(\xx^{t-1})] - \tfrac{1}{K}\expect [f(\xx^{t})]
    \\&\hspace{1cm} + \frac{96 \eta }{23 a}\E\norm{\ee^{t-1}}^2 - \frac{96 \eta }{23 a}\E\norm{\ee^{t}}^2
    \\&\hspace{1cm} + \frac{8 \delta}{K}\Delta^{t-1} - \frac{8 \delta}{K}\Delta^{t}
    \\&\hspace{1cm} - \tfrac{1}{K}\expect [f(\xx^{t})] - \frac{4\eta}{a}\expect\norm{\ee^{t}}^2
    \\&\hspace{1cm} + \rbr*{13 \eta a + 3\eta a^2} \frac{G^2}{S}\,.
  \end{align*}
  Since $a \geq 912 \eta^2 K^2 \delta^2$, we have $ \frac{1}{4} - \rbr*{\frac{1}{6} - \frac{38K^2 
    \delta^2 \eta^2}{a}} \geq \frac{1}{24}$. Using this proves the lemma.
\end{proof}

\subsection{Final convergence rates}
\begin{theorem}[Convergence of MimeMVR]
  Let us run MimeMVR with step size $\eta = \min\rbr*{ \frac{1}{L}, \frac{1}{864 \delta K}, \rbr*{\frac{S (f(\xx^{0}) - f^\star)}{6936 K^3 T \delta^2 G^2}}^{1/3}}$ and momentum parameter $a = \max\rbr*{1536 \eta^2 \delta^2 K^2, \frac{1}{T}}$. Then, given that \eqref{asm:heterogeneity} and \eqref{asm:hessian-similarity} hold, we have
  \begin{align*}
    \frac{1}{KST} \sum_{t\in[T]}\sum_{k \in [K]}\sum_{j\in\cS^t}\expect \norm{\nabla f(\yy_{i,k-1}^t)}^2 \leq \cO\rbr[\bigg]{ \rbr[\Big]{\frac{\delta^2 G^2 F}{S T^2}}^{1/3} + \frac{G^2}{S T} + \frac{(L + \delta K)F}{K T}}\,,
  \end{align*}
where we define $F := f(\xx^0) - f^\star$.
\end{theorem}
\begin{proof}
  Unroll the one round progress Lemma~\ref{lem:mimemvr-round-progress} and average over $T$ rounds to get
  \begin{align*}
    \frac{1}{KST} \sum_{t\in[T]}\sum_{k \in [K]}\sum_{j\in\cS^t}\expect \norm{f(\yy_{i,k-1}^t)}^2 &\leq \frac{24(\Phi^0 - \Phi^T)}{\eta T} + \frac{408 a G^2}{S}\,.
  \end{align*}
  Recall that we defined 
  \[ \Phi^t := \tfrac{1}{K}\expect [f(\xx^{t}) - f^\star]+ \frac{96\eta}{23 a}\expect\norm{\ee^{t}}^2 + \frac{8 \delta }{K}\Delta^t\,.
\]
Hence, $\Phi^T \geq 0$. Further, note that by definition $\Delta^0 =0$ and $\E\norm{\ee_0}^2 := \E\norm{\mm^0 - \nabla f(\xx^0)}^2$. \cite{cutkosky2019momentum} show that by using time-varying step sizes, it is possible to directly control the error $\ee_0$. Alternatively, \cite{tran2019hybrid} use a large initial accumulation for the momentum term. For the sake of simplicity, we will follow the latter approach. It is straightforward to extend our techniques to the time-varying step-size case as well but with additional proof complexity. Note that either way, the total complexity only changes by a factor of 2.
Suppose that we run the algorithm for $2T$ rounds wherein for the first $T$ rounds, we simply compute 
$
    \mm^0 = \frac{1}{T_0 S}\sum_{t=1}^{T_0} \sum_{j\in \cS^t}\nabla f_j(\xx^0)\,.
$
With this, we have 
$
\ee_0 = \E\norm{\mm^0 - \nabla f(\xx^0)}^2 \leq \frac{G^2}{ST}\,.
$
Thus, we have for the first round $t=1$
\[ \Phi^0 = \tfrac{1}{K}\expect [f(\xx^{0}) - f^\star]+ \frac{96\eta}{23a}\expect\norm{\ee^{0}}^2 \leq \tfrac{1}{K}\expect [f(\xx^{0}) - f^\star] + \frac{96 \eta G^2}{23 a T S} \,.
\]
Together, this gives
\begin{align*}
    \frac{1}{KST} \sum_{t\in[T]}\sum_{k \in [K]}\sum_{i\in\cS^t}\expect \norm{f(\yy_{i,k-1}^t)}^2 &\leq \frac{24(f(\xx^{0}) - f^\star)}{\eta K T}  + \frac{96 G^2}{a T^2 S} + \frac{408 a G^2}{S}\,.
  \end{align*}
  The above equation holds for any choice of $\eta \leq \min\rbr*{\frac{1}{L}, \frac{1}{864 \delta K}}$ and momentum parameter $a \geq 912 \eta^2 \delta^2 K^2$. Set the momentum parameter as 
  \[
    a = \max\rbr*{912 \eta^2 \delta^2 K^2 , \frac{1}{T}}
    \]
    With this choice, we can simplify the rate of convergence as 
      \[
        \frac{24 (f(\xx^{0}) - f^\star)}{ \eta K T}  + \frac{96 G^2}{T S} + \frac{166464 \eta^2 \delta^2 K^2 G^2}{S} + \frac{408 G^2}{ ST}\,.
      \]
      Now let us pick 
      \[
      \eta = \min\rbr*{ \frac{1}{L}, \frac{1}{864 \delta K}, \rbr*{\frac{S (f(\xx^{0}) - f^\star)}{6936 K^3 T \delta^2 G^2}}^{1/3}}\,.
      \]
      For this combination of step size $\eta$ and $a$, the rate simplifies to
      \[
      \frac{504 G^2}{T S} + 916 \rbr*{\frac{(f(\xx^{0}) - f^\star) \delta^2 G^2}{S T^2}}^{1/3} +  \frac{24 (L + 864 \delta K) (f(\xx^{0}) - f^\star)}{K T}\,.
      \]
  This finishes the proof of the theorem.
\end{proof}

\begin{theorem}[Convergence of MimeLiteMVR]
  Let us run MimeLiteMVR with step size $\eta = \min\rbr*{ \frac{1}{L}, \frac{1}{864 \delta K}, \rbr*{\frac{(f(\xx^{0}) - f^\star)}{6936 K^3 T \delta^2 (G^2 + \sigma^2)}}^{1/3}}$ and momentum parameter $a = \max\rbr*{1536 \eta^2 \delta^2 K^2, \frac{1}{T}}$. Then, given that \eqref{asm:heterogeneity} and \eqref{asm:smoothness} hold, we have
  \begin{align*}
    \frac{1}{KST} \sum_{t\in[T]}\sum_{k \in [K]}\sum_{j\in\cS^t}\expect \norm{\nabla f(\yy_{i,k-1}^t)}^2 \leq \cO\rbr[\bigg]{ \rbr[\Big]{\frac{\delta^2 (G^2 + \sigma^2) F}{ T^2}}^{1/3} + \frac{G^2 + \sigma^2}{T} + \frac{(L + \delta K)F}{K T}}\,,
  \end{align*}
where we define $F := f(\xx^0) - f^\star$.
\end{theorem}
\begin{proof}
  The proof for MimeLiteMVR is identical to that of MimeMVR, except that as noted in Lemma~\ref{lem:mimelitemvr-bias-control}, the $\frac{G^2}{S}$ term in Mime gets replaced by $(G^2 + \sigma^2)$ everywhere. Note that MimeLiteMVR (Lemma~\ref{lem:mimelitemvr-bias-control}) requires a weaker Hessian variance condition of $\norm{\nabla^2 f_i(\xx) - \nabla^2 f(\xx)}\leq \delta$ as opposed to MimeMVR which needs  $\norm{\nabla^2 f_i(\xx; \zeta) - \nabla^2 f(\xx)}\leq \delta$.
\end{proof}


\end{document}